\documentclass[11pt,letterpaper]{article}

\usepackage[margin=1in]{geometry}
\usepackage[T1]{fontenc}
\usepackage{lmodern}
\usepackage{microtype}

\usepackage[round,authoryear]{natbib}
\usepackage{amsmath,amssymb,amsthm}
\usepackage{mathtools}
\usepackage{bm}
\usepackage{bbm}
\usepackage{enumitem}

\usepackage{graphicx}
\usepackage{subcaption}
\usepackage{booktabs}
\usepackage{multirow}
\usepackage{tabularx}
\usepackage{longtable}

\usepackage{algorithm}
\usepackage{algpseudocode}
\usepackage{float}

\usepackage{authblk}

\usepackage{xspace}
\usepackage{xcolor}

\definecolor{commentgreen}{rgb}{0.0,0.45,0.0}
\definecolor{darkpurple}{RGB}{105,45,145}
\definecolor{darkorange}{RGB}{220,120,20}

\definecolor{mypurple}{HTML}{563276}
\definecolor{myyellow}{HTML}{DFA72A}
\definecolor{mygray}{HTML}{9B91A8}

\newcommand{\purpleblock}{%
    \textcolor{mypurple}{\rule{1.2ex}{1.2ex}}%
}
\newcommand{\lightpurpleblock}{%
    \textcolor{mygray}{\rule{1.2ex}{1.2ex}}%
}
\newcommand{\yellowblock}{%
    \textcolor{myyellow}{\rule{1.2ex}{1.2ex}}%
}

\definecolor{lucidL}{HTML}{4B2868}
\definecolor{lucidU}{HTML}{70485A}
\definecolor{lucidC}{HTML}{946640}
\definecolor{lucidI}{HTML}{B78428}
\definecolor{lucidD}{HTML}{D39A1F}
\definecolor{lucidEnd}{HTML}{EDC967}

\newcommand{\anchorletter}[2]{%
    \textcolor{#1}{%
        \scalebox{1.3}[1.3]{\textbf{#2}}%
    }%
}

\newcommand{\LUCID}{%
    \anchorletter{lucidL}{L}%
    \anchorletter{lucidU}{U}%
    \anchorletter{lucidC}{C}%
    \anchorletter{lucidI}{I}%
    \anchorletter{lucidD}{D}%
}

\usepackage{url}
\usepackage{hyperref}

\definecolor{mydarkblue}{rgb}{0,0.08,0.45}

\hypersetup{
    colorlinks=true,
    citecolor=mydarkblue,
    linkcolor=mydarkblue,
    urlcolor=mydarkblue
}

\usepackage[capitalize,noabbrev,sort]{cleveref}

\allowdisplaybreaks

\newtheorem{theorem}{Theorem}

\newtheorem{proposition}{Proposition}

\newtheorem{remark}{Remark}



\newcommand{\LCVI}{LUCID\xspace}

\title{
    \LUCID:
    \anchorletter{black}{L}ocal
\anchorletter{black}{U}nderstanding of
\anchorletter{black}{C}ovariate
\anchorletter{black}{I}mportance with Conditional
    \anchorletter{black}{D}ensity Estimation
}

\author[1]{Yixiao Wang}
\author[1]{Srikar Katta}
\author[1,2]{Jon Donnelly}
\author[3]{Emanuele Borgonovo}
\author[1]{Cynthia Rudin}

\affil[1]{Duke University, Durham, USA}
\affil[2]{Johns Hopkins University, Baltimore, USA}
\affil[3]{Bocconi University, Milan, Italy}

\affil[ ]{\texttt{
    yixiao.wang@duke.edu,
    srikar.katta@duke.edu,
    jdonnelly@jhu.edu,
    emanuele.borgonovo@unibocconi.it,
    cynthia@cs.duke.edu
}}

\date{}

\begin{document}

\twocolumn[
\begin{@twocolumnfalse}

\maketitle

\begin{abstract}
    Computing local-level variable importance (VI) is critical for understanding machine learning models. Local permutation-based VI metrics offer a simple yet powerful approach, measuring how much a model's output changes when perturbing its inputs. However, existing perturbation strategies often ignore feature dependence or account for it only through restricted approximations. Perturbing a feature without fully respecting this dependence can generate unrealistic counterfactuals, thereby evaluating the model in regions unsupported by the observed data distribution. We address this gap by introducing a new model-agnostic local VI metric and estimator, \textbf{Local Understanding of Covariate Importance with conditional Density estimation} (\LCVI). \LCVI perturbs features of interest using plausible values drawn from the full conditional distribution given the remaining observed features. We establish theoretical connections between \LCVI and existing VI measures, and provide an estimator with convergence guarantees and uncertainty quantification. Finally, we demonstrate \LCVI{}'s stable performance and practical utility across simulated and real-world examples.

\end{abstract}

\vspace{1em}
\end{@twocolumnfalse}
]

\section{Introduction}
\label{sec:introduction}

\begin{figure}[t]
    \centering
    \includegraphics[width=\columnwidth]{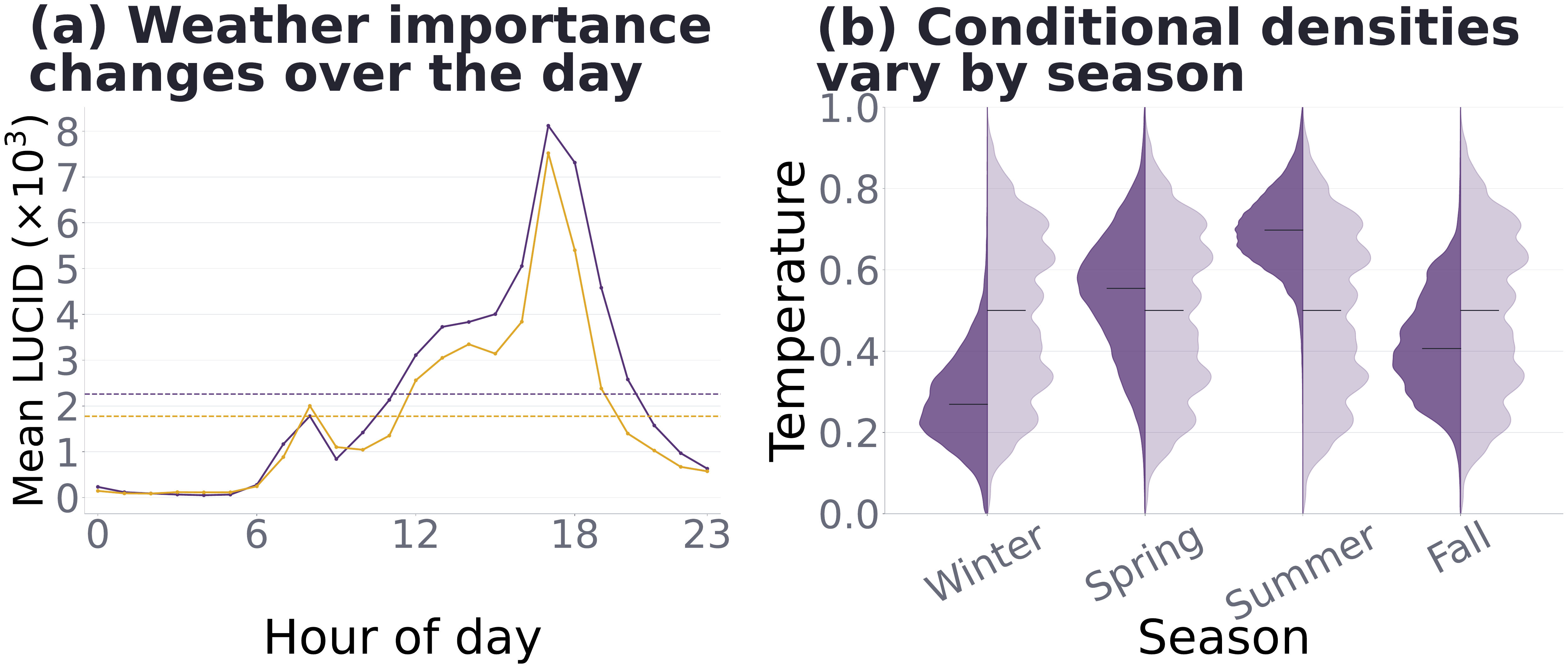}
    \caption{ 
    \textbf{Two motivations for \LCVI on the Bike Sharing dataset \citep{fanaee2014event}.}
   \textbf{(a)} \LCVI values for \purpleblock~\textcolor{mypurple}{temperature} and
\yellowblock~\textcolor{myyellow}{humidity} \LCVI (solid lines) vary substantially across hours, while their corresponding global VI values (dashed lines) obscure this heterogeneity. \textbf{(b)} For \textcolor{mypurple}{temperature}, each split violin compares the estimated \purpleblock~ conditional replacement distribution from DR-BART \citep[left;][]{orlandi2021density} with the \lightpurpleblock~marginal replacement distribution (right). Conditional densities adapt across seasons, whereas the marginal distribution does not.
}
    \label{fig:bike_teaser}
\end{figure}

Variable importance (VI) is a fundamental tool for interpreting machine learning models, enabling the characterization of features' usage and their effects on the model's predictions. VI is particularly useful for probing the otherwise opaque behavior of black-box machine learning systems \citep{guidotti2018survey,casalicchio2018visualizing,fisher2019all}. Most established VI methods summarize each feature with a single population-level score or ranking \citep{breiman2001random,gregorutti2015grouped}. However, a feature may strongly influence the model in one region of the covariate space while being irrelevant in another, so averaging importance across the population can obscure when, where, and for which observations a feature matters. For example, in the Bike Sharing dataset \citep{fanaee2014event}, the importance of temperature and humidity changes substantially over the course of a day, while their global VI scores, shown as dashed lines, collapse this variation into single values and obscure the underlying heterogeneity (Figure \ref{fig:bike_teaser}(a)). More fundamentally, distinct local importance patterns can share the same population average, meaning that global importance alone cannot identify how importance is distributed across observations. This motivates local VI and its accurate estimation, which makes observation-level feature usage accessible and characterizes heterogeneous feature usage across individual predictions.

A central challenge for local VI is determining how to perturb the feature. 
Existing local perturbation-based approaches typically sample counterfactuals from the marginal distribution of the feature of interest \citep{casalicchio2018visualizing,bladen2024conditional}, which can break feature dependence and generate implausible counterfactuals. For example, in the Bike Sharing data, plausible temperature values vary across seasons, whereas marginal permutation draws replacements from the same pool regardless of the remaining covariates (Fig.~\ref{fig:bike_teaser}(b)). Although global conditional importance methods provide dependence-aware alternatives, their perturbation schemes often rely on restricted approximations whose validity requires additional assumptions \citep{fisher2019all,chamma2023statistically,borgonovo2026trustworthy,strobl2008conditional,debeer2020conditional,molnar2024model,blesch2025conditional}. Such schemes can recover population-average importance while producing systematically biased local scores, because errors across observations may cancel after averaging (\cref{subsec:exp_synthetic_main} DGP 3). Accurate local estimation therefore requires capturing the observation-specific conditional structure. This motivates flexible estimation of the full conditional distribution, which existing perturbation schemes may struggle to capture, particularly under heteroskedasticity, multimodality, high dimensionality, or complex interactions.

In this work, we introduce \textbf{Local Understanding of Covariate Importance with conditional Density estimation} (\LCVI), a model-agnostic framework for measuring feature importance at the observation level. For a fixed predictive model, \LCVI replaces a feature of interest with plausible values drawn from its conditional density given the remaining observed features, and measures the resulting change in model behavior. Our main contributions are:
\begin{itemize}[labelindent=0pt, labelsep=0.5em, leftmargin=*, nosep]
    \item We establish an interpretable decomposition of \LCVI that captures local leave-one-covariate-out (LOCO) importance without refitting the predictive model and additionally quantifies conditional prediction variability omitted by local LOCO.
    
   \item We develop a flexible estimation framework for \LCVI that accommodates modern conditional density estimators (CDEs), establish convergence guarantees, and provide a Bayesian implementation with  posterior uncertainty quantification.

     \item We extend \LCVI to explain feature groups, enabling observation-level assessment of their joint importance conditional on the observed covariates outside the group.

     \item Through extensive synthetic and real-data experiments, we demonstrate \LCVI's ability to capture heterogeneous local importance under complex covariate dependence, together with the high accuracy, robustness, and computational scalability of our implementation.
\end{itemize}

\section{Related Work}
\label{sec:related_work}

Permutation-based variable importance measures feature relevance through changes in model behavior under feature perturbations, beginning with the Random Forest permutation importance of \citet{breiman2001random}. Subsequent work has developed a range of marginal and conditional variable importance methods at both global and local levels. We provide a detailed discussion of these methods in Section~\ref{app:related}. In contrast, \LCVI defines observation-level importance with respect to the full conditional feature distribution. Accurate estimation of this distribution yields replacements that preserve the dependence between the feature of interest and the remaining covariates

Other notions of feature importance and local explanation target different estimands. We discuss connections to LOCO \citep{lei2018distribution}, SHAP \citep{lundberg2017unified, aas2021explaining}, LIME \citep{ribeiro2016should}, grouped variable importance \citep{gregorutti2015grouped, au2022grouped,chamma2024variable, little2025iloco}, and other explanation paradigms in Section~\ref{app:related}.

\section{Our Method: \LCVI}
\label{sec:LCVI}
\subsection{Notation and Background}
\label{subsec:notation}

Let $Z$ denote the feature of interest and $\bm X=(X_1,\ldots,X_k)$ denote the remaining covariates. We write $\bm W=(Z, \bm X)$ for the complete feature vector and let $g(\bm W)=g(Z, \bm X)$ denote a given prediction model to be explained, with $\ell$ denoting a prediction loss. We assume the analyst observes a dataset of $n$ independent and identically distributed observations $\{\bm W_i \}_{i = 1}^n$ from some distribution $P$. For an observation $\bm w_i=(z_i,\bm x_i)$. We write $P_Z$ and $P(\cdot\mid\bm x)$ to denote the marginal and conditional (on all other features) distributions of feature $z$, respectively. For distributions with densities $p$ and $q$, we denote the total variation distance by $\operatorname{TV}(p,q):=\frac12\int |p(z)-q(z)|\,dz$, and, for a measurable function $h$, the $L^2$ norm by $\|h\|_{L^2(P)}:=\left(\mathbb E_P[h^2]\right)^{1/2}$.

\textbf{Permutation feature importance} (PFI) is the population-average discrepancy between the original prediction and the prediction obtained by replacing $Z$ with an independent draw from its marginal distribution:
\begin{equation}
\mathrm{PFI} := \mathbb E_{Z,\bm X}\left\{\mathbb E_{Z'\sim \color{red}{P_Z}}\!\left[\ell\!\left(g(Z',\bm X),g(Z,\bm X)\right)\right]\right\}.
\label{eq:pfi}
\end{equation}

\textbf{Conditional permutation feature importance} (CPFI) instead draws the replacement from the conditional distribution of $Z$ given the remaining features:
\begin{equation}
    \mathrm{CPFI} := \mathbb E_{Z,\bm X}\left\{\mathbb E_{Z'\sim \textcolor{red}{P(\cdot\mid\bm X)}}\!\left[\ell\!\left(g(Z',\bm X),g(Z,\bm X)\right)\right]\right\}.
     \label{eq:cpfi}
\end{equation}
\textbf{Leave-one-covariate-out importance} (LOCO) measures the expected prediction loss when the full model is replaced by a reduced model using only the remaining features \citep{lei2018distribution}: $\mathrm{LOCO} := \mathbb E_{Z,\bm X}\!\left[\ell\!\left(g_{-Z}(\bm X),g(Z,\bm X)\right)\right]$, where $g_{-Z}$ denotes a reduced model, typically obtained by refitting without feature $Z$.

\subsection{\LCVI Definition}
\label{subsec:lcvi_target}

We now introduce the instance-level counterpart of conditional permutation feature importance. For an observation $\bm w$, \textbf{Local Understanding of Covariate Importance with conditional Density estimation} (\LCVI) of feature $Z$ is defined as
\begin{equation}
\mathrm{\LCVI}(\bm w) := \mathbb E_{Z'\sim \textcolor{red}{P(\cdot\mid\bm x)}}\!\left[\ell\!\left(g(Z',\bm x),g(z,\bm x)\right)\right].
\label{eq:lcvi}
\end{equation}
\LCVI measures the expected prediction discrepancy when $z$ is replaced by a draw from $P(\cdot\mid \bm x)$ with the remaining covariates fixed at $\bm x$. By the law of total expectation, integrating \LCVI in \cref{eq:lcvi} over the population distribution of $\bm W$ recovers the CPFI in \cref{eq:cpfi}.

\subsection{Understanding \LCVI}
\label{subsec:lcvi_geometry}

Under squared loss, \LCVI admits a simple decomposition into two terms, each asking a separate question: (i) how varied are the predictions from $g$ across different counterfactual values of $Z$; (ii) how extreme was the prediction for the \textit{observed} value of $Z$ compared to the predictions for possible counterfactual values of $Z$? We first establish this decomposition for squared-error loss, illustrating the theoretical connection to local LOCO. We then extend the decomposition to Bregman divergences generated by differentiable, strictly convex functions.

Below, expectations and variances are over $Z'\sim P(\cdot\mid\bm x)$, and we assume the black box's predictions admit a finite variance over different perturbations.

\begin{theorem}[Decomposition of \LCVI]
\label{thm:local_lcvi_decomposition}

\begin{figure}[t]
    \centering
    \includegraphics[width=\columnwidth]{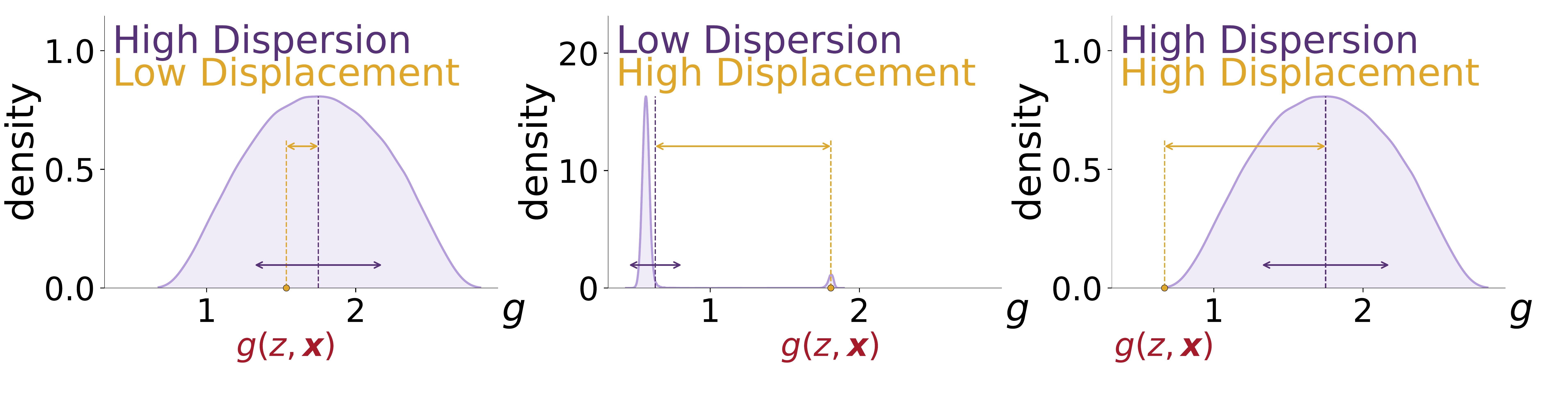}
    \caption{\textbf{Three examples of conditional distributions of perturbed outputs $g(Z',\bm x)$.} The \textcolor{mypurple}{purple dashed line} marks the mean \textcolor{myyellow}{$\mathbb E\!\left[g(Z',\bm x)\right]$}, the \textcolor{myyellow}{yellow arrow} shows the square root of the \textcolor{myyellow}{displacement},  while the \textcolor{mypurple}{purple arrow} illustrates the conditional standard deviation, the square root of the \textcolor{mypurple}{dispersion} term.}
    \label{fig:main_paper_decomposition_2d_3}
\end{figure}

Under the prediction reference and squared loss, $\mathrm{\LCVI}(\bm w)$ admits the decomposition
\begin{align}
\textcolor{mypurple}{\underbrace{\operatorname{Var}\!\left[g(Z',\bm x)\right]}_{\text{dispersion}}}
+
\textcolor{myyellow}{\underbrace{\left(g(z,\bm x)-\mathbb E\!\left[g(Z',\bm x)\right]\right)^2}_{\text{displacement}}},
\label{eq:squared_lcvi_decomposition}
\end{align}
\end{theorem}
We refer to the two components in~\eqref{eq:squared_lcvi_decomposition} as the \textcolor{mypurple}{\textbf{dispersion}} and \textcolor{myyellow}{\textbf{displacement}} terms, respectively. The proof for Theorem \ref{thm:local_lcvi_decomposition} is in \cref{{app:squared-loco}}. Globally, dispersion and displacement contribute equally on average over the data distribution (Corollary \ref{app:cor:squared-averages}). At the local level, however, they capture distinct mechanisms through which a feature can be important for a particular observation. The \textcolor{mypurple}{dispersion} term measures how much the model output varies across conditionally plausible replacements of $Z$. 

In contrast, the \textcolor{myyellow}{displacement} term measures how far the original output lies from the average prediction across perturbations. So, if we observe a value of $Z$ that leads to an extreme prediction far from the central tendency, displacement will be high. Both quantities measure how important the feature is for making the prediction, but in different ways. Dispersion asks if predictions vary wildly across possible values of $Z \mid \bm{x}$ while displacement asks if the prediction for the \emph{observed} data point is itself an extreme/outlier compared to the average prediction. As illustrated in Figure~\ref{fig:main_paper_decomposition_2d_3}, \LCVI can therefore be large because the perturbed outputs vary substantially (high dispersion), because the observed output lies far from their mean (high displacement), or both. Conversely, since both terms are nonnegative, by Corollary \ref{app:cor:zero},
\begin{equation}
\mathrm{\LCVI}(z,\bm x)=0
\quad\Longleftrightarrow\quad
g(Z',\bm x)=g(z,\bm x)\ \text{a.s.}
\label{eq:squared_lcvi_zero_condition}
\end{equation}

This can occur (i) because the model does not use feature $Z$ in predicting $\bm x$'s outcome, (ii) because $Z$ is determined by the remaining features, or (iii) because its conditionally plausible values all produce the same model output.

\paragraph{Relationship with LOCO.} \label{subsec:lcvi_loco}

The decomposition also clarifies the relationship between \LCVI and LOCO, both of which assess a feature's contribution beyond the remaining features. Under squared loss with Bayes-optimal full and reduced predictors, \textbf{CPFI equals twice LOCO} (Corollary \ref{app:cor:squared-averages}) \citep{reyero2025conditional}. However, this proportionality does not generally hold at the local level.

For a comparison  \LCVI, if the reduced model is the optimal reduced predictor of $g(Z, \bm X)$ under squared loss, then $g_{-Z}(\bm x) = \mathbb E[g(Z,\bm x)]$. The corresponding local LOCO score is $\mathrm{LOCO}(\bm w) := \left(g(z, \bm x)-g_{-Z}(\bm x)\right)^2$.
\begin{proposition}[Relationship between \LCVI and local LOCO]
\label{prop:lcvi_loco}
Under squared loss and the prediction reference, 
\begin{equation}
\mathrm{\LCVI}(\bm w)
= \textcolor{mypurple}{\underbrace{\operatorname{Var}\!\left[g(Z',\bm x)\right]}_{\text{dispersion}}} +
\textcolor{myyellow}{\underbrace{\mathrm{LOCO}(\bm w)}_{\text{displacement}}}
.\nonumber
\label{eq:lcvi_loco}
\end{equation}
\end{proposition}

The result follows directly from \cref{thm:local_lcvi_decomposition} and the proof for Proposition \ref{prop:lcvi_loco} is in \cref{{app:squared-loco}}. Under this comparison, local LOCO measures the discrepancy between the original output and the optimal reduced prediction, while \LCVI additionally captures variation among conditionally plausible outputs. This additional dispersion term can change local feature rankings even when the corresponding global importance measures differ only by a factor of $2$ (see \cref{app:subsec:extrem_tail}).

\paragraph{Beyond Squared Loss.}
The \textcolor{mypurple}{dispersion}--\textcolor{myyellow}{displacement} interpretation of \LCVI extends to
Bregman divergences \citep{BREGMAN1967200,adlam2022understanding}. For a differentiable, strictly convex function
$\phi$, define
$D_\phi(a,b)=\phi(a)-\phi(b)-\langle\nabla\phi(b),a-b\rangle$. 
\begin{proposition}[Bregman Decomposition of \LCVI]
\label{prop:lcvi_bregman_decomposition}
Under the prediction reference and Bregman loss $\ell(a,b)=D_\phi(a,b)$, assuming the relevant expectations are finite, $\mathrm{\LCVI}(\bm w)$ admits the decomposition
    \begin{align}
    \mathrm{\LCVI}(\bm w) &= \textcolor{mypurple}{\underbrace{ \mathbb E\left[D_\phi\left(g(Z',\bm x),\mathbb E\left[g(Z',\bm x)\right]\right)\right]}_{\text{dispersion}}}\nonumber\\
    &+\textcolor{myyellow}{\underbrace{D_\phi\!\left(\mathbb E\left[g(Z',\bm x)\right],g(z,\bm x)\right)}_{\text{displacement}}}.
    \label{eq:lcvi_bregman_decomposition}
    \end{align}
\end{proposition}

For asymmetric Bregman divergences, the divergence direction matters. When averaging, the squared-loss factor-of-two identity between global CPFI and LOCO need not hold. We provide detailed proofs for Proposition \ref{prop:lcvi_bregman_decomposition} and more illustrative examples in \cref{app:bregman}. 

\subsection{\LCVI for Feature Groups}
\label{subsec:group_lcvi}

Individual features' \LCVI may fail to capture joint effects when features carry overlapping information. 
For example, suppose a prediction model depends on both temperature and season, while these features are strongly predictable from one another. Conditioning on temperature can substantially restrict plausible replacements of season, and vice versa, yielding small individual \LCVI scores. Such scores reflect limited unique information, however, and do not imply that the two features are jointly unimportant. We therefore extend \LCVI to feature groups \citep{gregorutti2015grouped,au2022grouped,chamma2024variable}. For a group of features $\bm Z$, let $\bm X$ denote the remaining conditioning covariates. We define the group \LCVI as
\begin{equation}
\mathrm{\LCVI}(\bm z,\bm x)
:=
\mathbb E_{\bm Z'\sim P(\cdot\mid\bm x)}
\left[
\ell\!\left(g(\bm Z',\bm x),g(\bm z,\bm x)\right)
\right].
\label{eq:group_lcvi}
\end{equation}
The definition extends directly to a group of features by replacing the scalar $Z$ with a vector $\bm Z$. Let $G$ denote a feature group with $d=|G|$. In general, group \LCVI need not equal the sum of the corresponding individual \LCVI scores, because individual replacements condition on the other group members, whereas group replacement allows all members to vary jointly. Examining individual and group \LCVI together provides a fuller picture, as illustrated in Section \ref{subsec: real_dataset_main}.

\section{Estimating \LCVI}
\label{sec:lcvi_estimation}

Estimating \LCVI requires sampling from an unknown conditional feature distribution. We first introduce a plug-in Monte Carlo estimator applicable to any conditional density estimator (CDE), and then introduce an implementation with DR-BART.

\subsection{Conditional-Sampling Estimation}
\label{subsec:lcvi_algorithm}

Let $\mathcal D_n$ denote the sample used to estimate the conditional distribution of $Z$ given $\bm X$, and let $\bm w_i=(z_i,\bm x_i)$, $i=1,\ldots,N$, denote the observations to be explained. We continue to treat $g$ as fixed and use prediction-reference \LCVI.

Given an estimate $\widehat P_{Z\mid\bm X}(\cdot\mid\bm x)$, draw $Z_m'\overset{\mathrm{i.i.d.}}{\sim}\widehat P_{Z\mid\bm X}(\cdot\mid\bm x)$, $m=1,\ldots,M$, and define
\begin{equation}
\widehat{\mathrm{\LCVI}}(\bm w)
:=
\frac{1}{M}\sum_{m=1}^{M}
\ell\!\left(g(Z_m',\bm x),g(z,\bm x)\right).
\label{eq:lcvi_estimator}
\end{equation}
Estimation therefore requires only conditionally plausible draws and evaluations of the fixed model, making the procedure agnostic to the choice of CDE.

\subsection{Stability and Convergence}
\label{subsec:lcvi_convergence}

To study sensitivity to conditional-distribution estimation, write $\mathrm{\LCVI}[p](\bm w)$  for prediction-reference \LCVI induced by conditional density $p(\cdot\mid\bm x)$, and let $p^*$ denote the truth. Let $d$ be a distance between conditional distributions satisfying, for some
constant $A$,
\begin{equation}
\left\|\operatorname{TV}\!\left\{p(\cdot\mid \bm X),q(\cdot\mid \bm X)\right\}\right\|_{L^2(P_{\bm X})}
\leq A\,d(p,q).
\label{eq:conditional_distance_control}
\end{equation}

\begin{theorem}[Stability of \LCVI]
\label{thm:lcvi_contraction_transfer}
Suppose that the loss \(\ell(g(z',\bm x),g(z,\bm x))\) is uniformly bounded by some \(C_\ell<\infty\), for \(P_{Z,\bm X}\)-almost every \((z,\bm x)\) and all replacement values \(z'\) in the supports of the conditional distributions being compared. If the estimated conditional density satisfies \(d(\widehat p,p^*)=O_P(\epsilon_n)\), then the induced \LCVI function converges at the same rate in \(L^2(P_{Z,\bm X})\):
\[
\left\|\mathrm{\LCVI}[\widehat p]-\mathrm{\LCVI}[p^*]\right\|_{L^2(P_{Z,\bm X})}=O_P(\epsilon_n).
\]
\end{theorem}
The theorem shows that conditional distribution estimation error passes to \LCVI without a loss in convergence rate under the stated bound. For squared loss, bounded model outputs are sufficient. Proof is given in \cref{app:stability}.

\subsection{Implementing \LCVI with DR-BART}
\label{subsec:lcvi_drbart}

Our main implementation estimates each conditional distribution $P(\cdot\mid\bm X)$ using Density Regression Bayesian Additive Regression Trees (DR-BART) \citep{orlandi2021density}. Posterior draws of this conditional distribution induce posterior draws of \LCVI, allowing point estimation and uncertainty quantification.

\paragraph{Posterior point estimates and credible intervals.}
For each retained posterior draw $b=1,\ldots,B$, generate $M$ replacements and compute
\begin{equation}
\widehat{\mathrm{\LCVI}}^{(b)}(\bm w)
:=
\frac{1}{M}\sum_{m=1}^{M}
\ell\!\left(g(Z_{m}^{\prime(b)},\bm x),g(z,\bm x)\right).
\label{eq:lcvi_drbart_posterior_draw}
\end{equation}
We use the average of $\{\widehat{\mathrm{\LCVI}}(\bm w)^{(b)}\}_{b=1}^{B}$ as the posterior point estimate and their empirical quantiles to construct posterior credible intervals.

\paragraph{Convergence under DR-BART.}
Under the bounded-loss assumption in \cref{thm:lcvi_contraction_transfer} and the regularity and prior conditions of \citet[Theorem~2]{orlandi2021density}, the DR-BART density posterior contracts in integrated Hellinger distance, which satisfies the required distance-control condition in \cref{thm:lcvi_contraction_transfer}. The induced \LCVI posterior therefore contracts at the same rate in $L^2(P_{Z,\bm X})$. Under these conditions, Sections~\ref{app:posterior-theory}--\ref{app:posterior-averaging} further establish $O_P(\epsilon_n^2)$ posterior second-moment control for \LCVI (Corollary~\ref{app:cor:drbart}) and an $O_P(\epsilon_n)$ convergence rate for the posterior-mean conditional density and its induced \LCVI function (Proposition~\ref{app:prop:posterior-averaging}).

\cref{app:implementation_details} also provides implementation details, including direct conditional sampling from the DR-BART latent-variable representation, binary and categorical features, and joint conditional sampling for feature groups.

\paragraph{Computational scalability with DR-BART.}
Although fitting DR-BART is computationally intensive, fitted conditional samplers can be reused across observations.  \Cref{sec:drbart_sampling_batching} provides theoretical time and space complexity analyses, while \cref{app:openxai_scaling} reports empirical runtime and memory scaling on
higher-dimensional OpenXAI \citep{NEURIPS2022_65398a0e} benchmark datasets, including datasets with more than $10^5$ observations and more than $50$ features.

\section{Experiments}
\label{sec:experiments}

We evaluate \LCVI in controlled synthetic settings with known models and in a real-data bike-demand case study with learned prediction models. For each dataset, we use five-fold cross-fitting: we train a conditional density estimator on four folds and estimate \LCVI on the held-out fold \citep{chernozhukov2018double}. For simplicity, throughout the experiments we use $I_j$ to denote the \LCVI of feature $j$. We list baseline methods in Table~\ref{tab:baselines} and provide complete experimental details and additional results in \cref{app:bike,app:dgp}.
\begin{table}[t]
\centering
\caption{Methods considered in our experiments.}
\label{tab:baselines}
\scriptsize
\setlength{\tabcolsep}{3pt}
\renewcommand{\arraystretch}{0.92}
\begin{tabular}{@{}ll@{}}
\toprule
\textbf{Category} & \textbf{Method} \\
\midrule
Attribution
& SHAP \citep{lundberg2017unified} \\
& LIME \citep{ribeiro2016should} \\
& Cond. SHAP \citep{aas2021explaining,jullum2025shapr} \\
\midrule
Removal
& local version of LOCO \citep{lei2018distribution} \\
\midrule
Permutation
& ICI \citep{casalicchio2018visualizing} \\
& CMR \citep{fisher2019all} \\
& GCMR \citep{borgonovo2026trustworthy} \\
& CLIQUE \citep{bladen2024conditional} \\
& local version of cs-PFI \citep{molnar2024model} \\
& local version of cARFi \citep{blesch2025conditional} \\
\bottomrule
\end{tabular}
\vspace{2pt}
\parbox{\linewidth}{\scriptsize\emph{Note:} We refer to ICI as ``Permutation'' in this paper. ARF denotes the conditional replacement sampler used for our local adaptation of cARFi.}
\end{table}

\subsection{Synthetic Experiments}
\label{subsec:exp_synthetic_main}
Because the synthetic experiments admit analytical \LCVI, we first compare estimated and analytical importance using mean squared error (MSE) for feature of interest $X_1$ in Table \ref{tab:dgp-x1-lcvi-mse}. Across these controlled settings, our DR-BART implementation is the most consistently accurate conditional estimator overall, attaining the best average MSE rank while remaining robust across discontinuity (DGP 1), correlation (DGP 2), heteroskedasticity (DGP 3), interaction structure (DGP 4), and extreme tail (DGP 5).  Competing methods can perform well when their structural assumptions align with the data-generating process, but their errors can increase sharply when those assumptions are violated. We further repeat the synthetic experiments over 50 random seeds in \cref{app:dgp_stress_test}, confirming that these qualitative conclusions are not driven by a particular realization.

\begin{table}[t]
\centering
\small
\setlength{\tabcolsep}{4pt}
\caption{MSE between estimated and analytical \LCVI, averaged over all evaluated instances of $X_1$. The lowest MSE (from unrounded MSE values) in each DGP is highlighted in \textcolor{darkpurple}{\underline{\textbf{purple bold}}}. Mean rank averages within-DGP rankings; lower is better.}
\label{tab:dgp-x1-lcvi-mse}
\resizebox{0.85\columnwidth}{!}{%
\begin{tabular}{lrrrrrrr}
\toprule
 & DR-BART & ARF & cs-PFI & CMR & GCMR & CLIQUE & Permutation \\
\midrule
DGP 1 & 0.022 & 0.144 & 0.008 & 0.590 & 0.543 & \textcolor{darkpurple}{\underline{\textbf{0.001}}} & 0.008 \\
DGP 2 & \textcolor{darkpurple}{\underline{\textbf{0.004}}} & 0.005 & 0.015 & 0.011 & 0.011 & 4.460 & 4.431 \\
DGP 3 & \textcolor{darkpurple}{\underline{\textbf{0.449}}} & 0.454 & 1.173 & 2.685 & 2.485 & 2.060 & 2.065 \\
DGP 4 & \textcolor{darkpurple}{\underline{\textbf{0.401}}} & 1.298 & 83.329 & 1.809 & 4.315 & 83.673 & 83.135 \\
DGP 5 & 0.066 & 0.185 & \textcolor{darkpurple}{\underline{\textbf{0.036}}} & 0.849 & 0.851 & 0.251 & 0.269 \\
\midrule
Mean rank & \textcolor{darkpurple}{\underline{\textbf{1.800}}} & 2.800 & 3.400 & 5.200 & 5.400 & 4.600 & 4.800 \\
\bottomrule
\end{tabular}
}
\end{table}

We organize subsequent contents around four questions motivating our framework: ``why local importance?'' (DGP 1); ``why conditional perturbations?'' (DGP 2); ``why density estimation?'' (DGP 3); and ``why estimate the conditional density directly?'' (irrelevant-feature robustness test). We also consider DGP 4 to show why subgroup-based methods can struggle to recover interaction structure, and DGP 5 to further explore the relationship between \LCVI and local LOCO; we defer these two experiments, along with analytical \LCVI calculations and additional experimental details and results, to \cref{app:dgp}.

\paragraph{DGP 1: Discontinuity.}
\textbf{\LCVI handles feature heterogeneity, unlike global VI measures.}
This case study uses three independent standard Gaussian features, $X_1,X_2,X_3\stackrel{\mathrm{iid}}{\sim}\mathcal N(0,1)$, and a simple tree predictor $g(\bm x)=x_1\mathbf 1_{x_3>0}+x_2\mathbf 1_{x_3\leq0}$, so that $X_1$ is active only when $X_3>0$, whereas $X_2$ is active only when $X_3\leq0$ \citep[a setup similar to the one considered in][]{bladen2024conditional}. As shown in the DR-BART panel in Figure \ref{fig:main_experiments_dgp1_v2}, because $X_1$ and $X_2$ play symmetric roles across the discontinuity, their global importance (dashed line) is identical, masking the fact that only one of them matters at any given observation. In contrast, \LCVI captures this heterogeneity exactly, with the \textcolor{myyellow}{yellow} and \textcolor{mypurple}{purple} points representing observations on the two sides of the discontinuity.

\textbf{\LCVI correctly identifies locally irrelevant features.} 
When a feature is locally inactive, perturbing it leaves the prediction unchanged, so \LCVI is zero by \cref{eq:squared_lcvi_zero_condition}. In contrast, LIME, SHAP, conditional SHAP, and local LOCO assign nonzero scores to \(X_1\) at every inactive instance (\(x_3 \leq 0\)); a representative point is highlighted by the red box in Figure \ref{fig:main_experiments_dgp1_v2}. This results in importance leakage for different reasons: LIME fits a locally linear model that does not recognize boundaries, while SHAP and conditional SHAP average over subsets that omit $X_3$, mixing contributions across the discontinuity. Moreover, finite-sample errors in model refitting can lead to local LOCO producing nonzero scores. \cref{app:subsec:discontinuity} contains detailed comparisons and experiment results.

\begin{figure}[t]
    \centering
    \includegraphics[width=\columnwidth]{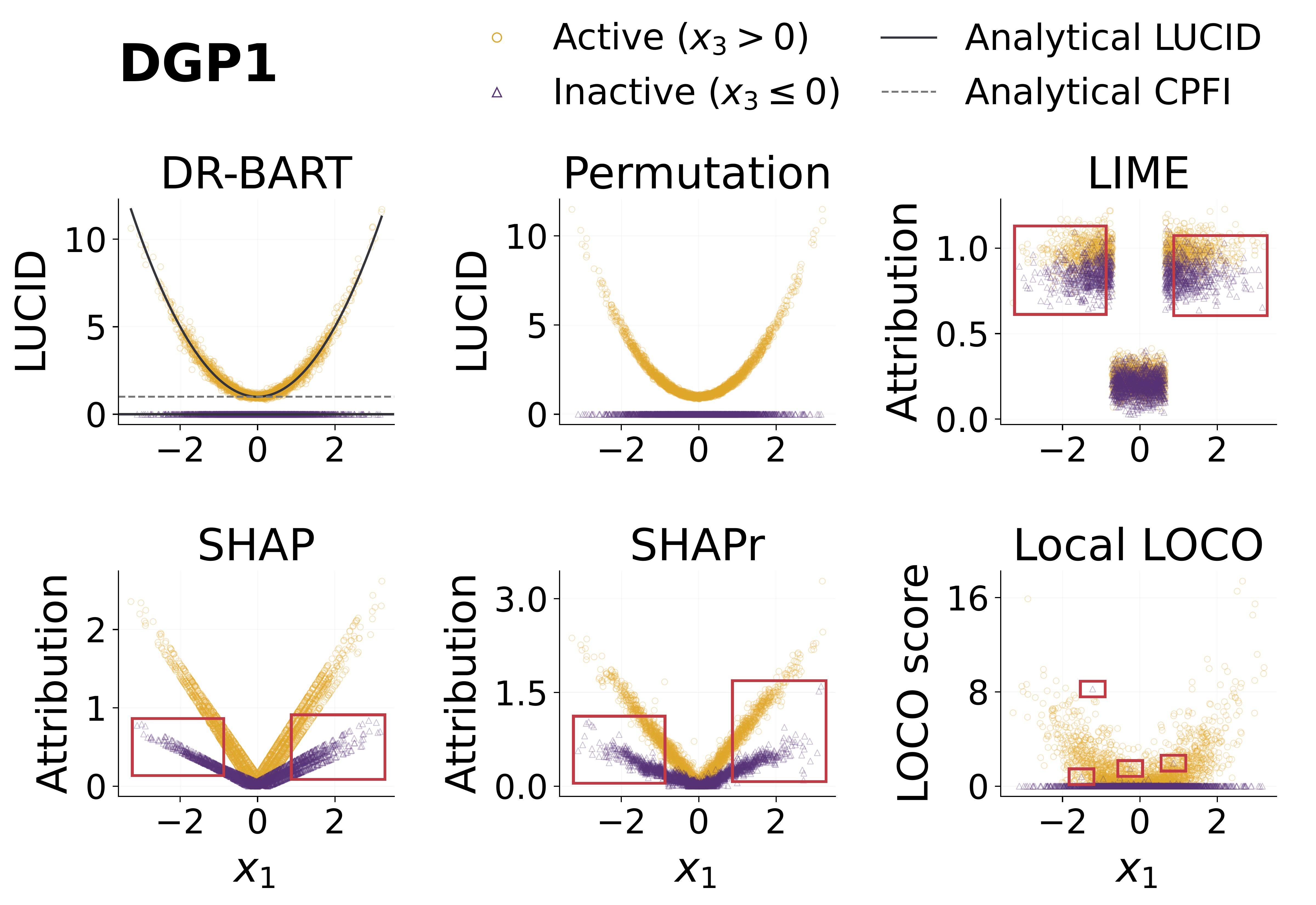}
    \caption{\textbf{Discontinuity.} Local importance for \(X_1\) across methods. Red boxes mark representative nonzero importance in the inactive region. Solid and dashed curves denote analytical \LCVI and CPFI, respectively.}
    \label{fig:main_experiments_dgp1_v2}
\end{figure}

\paragraph{DGP 2: Correlated.}
\textbf{Conditioning in \LCVI accounts for correlated features.} 
When features are dependent, marginal permutation breaks the observed dependence structure and may evaluate the fitted model on unrealistic feature combinations \citep{strobl2008conditional, hooker2021unrestricted}. We isolate this effect using two jointly Gaussian features, $(X_1,X_2)^\top\sim\mathcal N_2(0,\Sigma_\rho)$ with unit variances and correlation $\rho$, and the predictor $g(\bm x)=x_1$. In contrast to marginal permutation, conditional perturbations preserve features' dependency and therefore measure the additional information in $X_1$ beyond $X_2$. Accordingly, in this DGP, the analytical-average \LCVI equals $2(1-\rho^2)$ (where $\rho$ denotes the correlation between features), whereas global marginal permutation importance remains equal to $2$. Thus, as $\rho\to1$ (i.e., as features become increasingly correlated), conditional importance reflects the increasing redundancy, whereas permutation-based methods fail to respect the feature dependence structure, leading to poor performance in the MSE comparison in Table \ref{tab:dgp-x1-lcvi-mse};  see \cref{app:subsec:correlated} for experiment results.

\begin{figure}[t]
    \centering
    \includegraphics[width=\columnwidth]{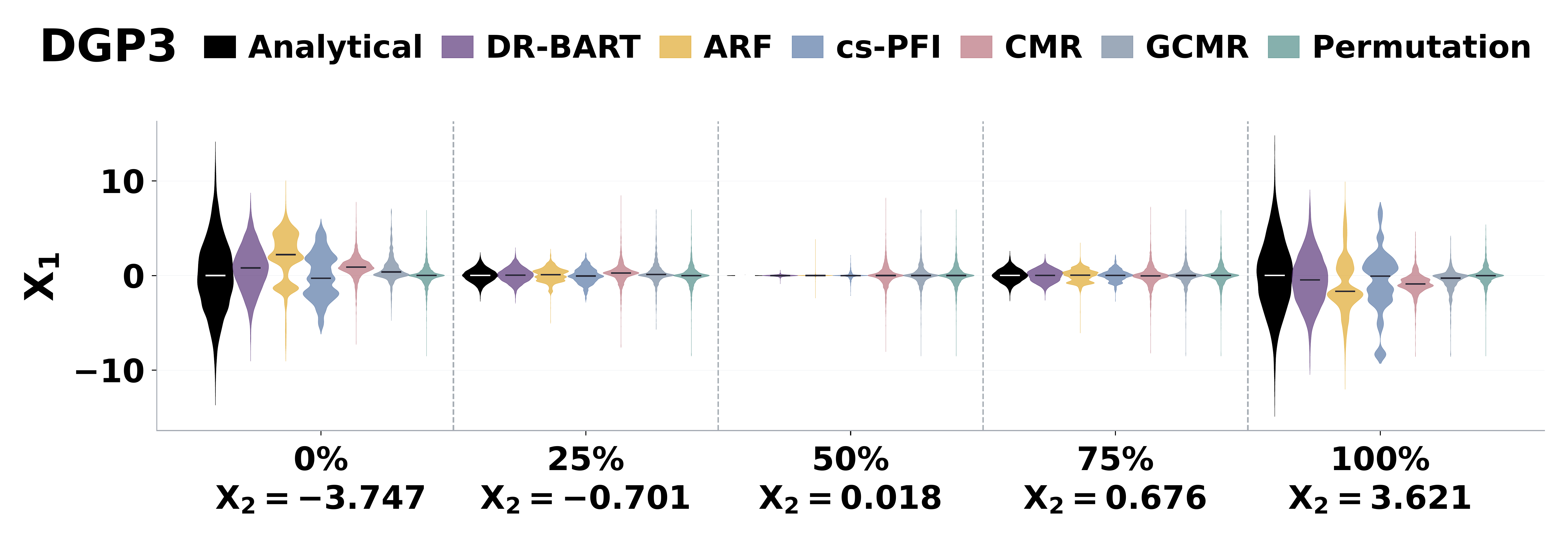}
    \caption{\textbf{Heteroskedastic.} Resampling distributions of $X_1$ at five $X_2$ quantile points across samplers.}\label{fig:app_experiments_dgp_cde_heteroskedastic_conditional_dependence_main}
\end{figure}
\paragraph{DGP 3: Heteroskedastic.}
\textbf{Residual-based conditional permutation importance fail to capture heteroskedasticity.}
We isolate heteroskedastic dependence by letting $X_2\sim\mathcal N(0,1)$ and $X_1\mid X_2=x_2\sim\mathcal N(0,x_2^2)$, with predictor $g(\bm x)=x_1$. Thus, the conditional mean $\mathbb E[X_1\mid X_2]=0$ is constant, while the conditional variance $\operatorname{Var}(X_1\mid X_2=x_2)=x_2^2$ varies with $x_2$. Even with the correct conditional mean, methods based on globally permuted residuals \citep[e.g.,][]{fisher2019all, borgonovo2026trustworthy} draw perturbations with the pooled residual variance $\mathbb{E}[X_2^2]=1$, failing to preserve this heteroskedasticity. 
As a result, while the analytical \LCVI is $x_1^2+x_2^2$, these methods asymptotically yield $x_1^2+1$. Importantly, this local misspecification is invisible after global averaging because $\mathbb E[X_2^2]=1$: the incorrect local scores recover the correct global average importance despite systematically misrepresenting where the feature matters. Moreover, \cref{app:subsec:Heteroskedastic} shows that any common replacement distribution incurs a population local MSE of at least $2$ in this DGP, further demonstrating why accurate conditional density estimation is essential for local variable importance. In Figure~\ref{fig:app_experiments_dgp_cde_heteroskedastic_conditional_dependence_main}, we show the resampling distributions induced by each method. Only CDE-based approaches capture the changing conditional spread, with DR-BART most closely matching the analytical CDE.

\begin{figure}[t]
    \centering
    \includegraphics[width=\columnwidth]{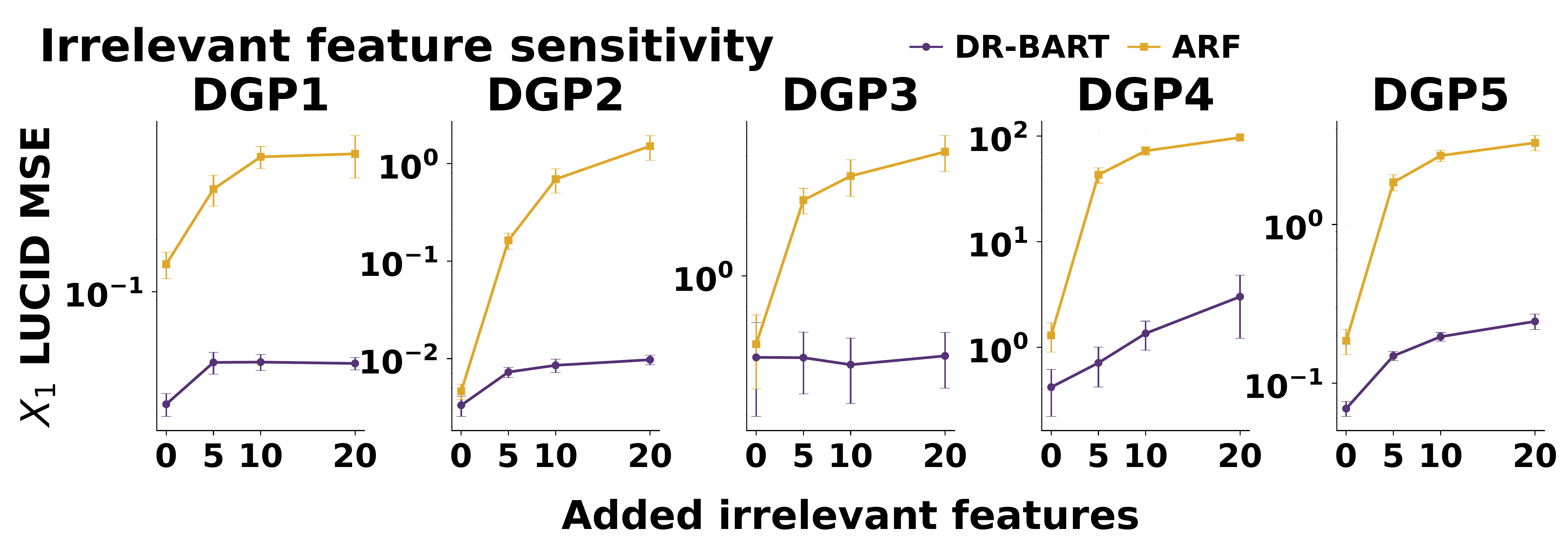}
    \caption{\LCVI estimation MSE as independent irrelevant features are added across five DGPs.}
    \label{fig:main_experiments_dgp_irrelevant}
\end{figure}

\textbf{Conditional density estimation is more robust to high-dimensional covariate spaces than joint density estimation.}
Similar to our DR-BART estimator of \LCVI, cARFi also generates perturbations by sampling from a learned density  \citep{blesch2025conditional}. cARFi obtains conditional replacements from the joint distribution using Adversarial Random Forest \citep[ARF;][]{watson2023adversarial}, whereas we instead fit the conditional density directly. Because \textit{joint} distribution estimation generally becomes more difficult as the ambient dimension grows \citep{izbicki2016nonparametric},
we consider a naive stress test: we augment each of the five DGPs with $0$, $5$, $10$, or $20$ independent standard Gaussian noise features and compare ARF with DR-BART.
These additions increase the ambient dimension without changing the true conditional distribution of $X_1$ or its analytical \LCVI. As shown in ~\cref{fig:main_experiments_dgp_irrelevant}, the \LCVI estimation MSE increases sharply for ARF as irrelevant features are added, whereas DR-BART's estimated conditional density is far more stable. In \cref{app:add_results_cspfi}, we further consider more challenging high-dimensional settings with additional relevant features, where the degradation of joint-density-based estimation is even more pronounced.

\paragraph{Posterior credible intervals.}
A major advantage of \LCVI is that its estimation admits posterior credible intervals. Table~\ref{tab:dgp-x1-active-coverage-width-levels-main} evaluates the uncertainty quantification provided by the DR-BART posterior for \LCVI over observations where $X_1$ is active in the model. Across the five DGPs, empirical coverage is generally close to the nominal $90\%$ and $95\%$ levels, with DGP~1 somewhat conservative, consistent with its particularly simple structure, and DGP~3 slightly below nominal at $95\%$, mainly because of the sparse, highly dispersed observations in the tails. As expected, interval widths increase with the credible level. We exclude inactive observations, whose \LCVI{}s are identically zero and recovered exactly.
\begin{table}[t]
\centering
\small
\renewcommand{\arraystretch}{0.92}
\setlength{\tabcolsep}{5pt}
\caption{Empirical coverage ($C$) and width ($W$) of \LCVI posterior credible intervals for active $X_1$ observations across five DGPs.}
\label{tab:dgp-x1-active-coverage-width-levels-main}
\begin{tabular}{@{}lccccc@{}}
\toprule
& DGP 1 & DGP 2 & DGP 3 & DGP 4 & DGP 5 \\
\midrule
90\% $C$ & 98.43\% & 91.70\% & 91.34\% & 93.20\% & 97.12\% \\
90\% $W$ & 0.766 & 0.155 & 0.944 & 0.969 & 0.956 \\
95\% $C$ & 99.37\% & 95.76\% & 92.94\% & 95.64\% & 98.52\% \\
95\% $W$ & 0.910 & 0.184 & 1.123 & 1.152 & 1.138 \\
\bottomrule
\end{tabular}
\end{table}

\begin{figure*}[!t]
    \centering
     \includegraphics[width=0.95\textwidth]{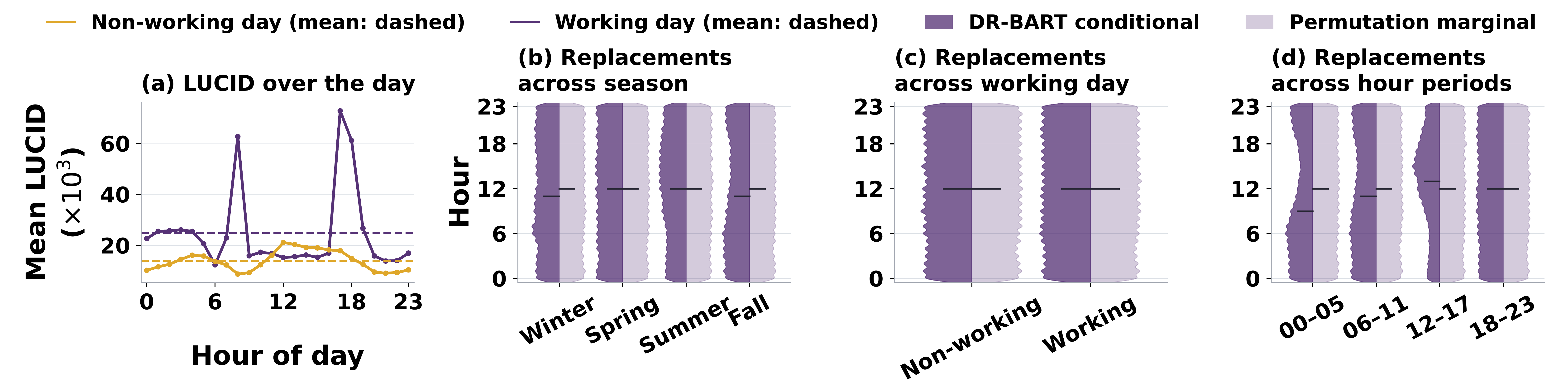}
    \caption{
\textbf{Local importance and conditional replacements for hour.}
(a) \LCVI for \texttt{hour} varies substantially over the day and between \purpleblock~\textcolor{mypurple}{working} and \yellowblock~\textcolor{myyellow}{non-working days} (solid lines). (b--d) Split violins compare the estimated \purpleblock~\textcolor{mypurple}{conditional} distribution of replacement \texttt{hour} values using DR-BART (left) with the
\lightpurpleblock~marginal permutation distribution (right).
}
    \label{fig:main_experiments_bike_sharing_hr}
\end{figure*}
\subsection{Real-Dataset Experiments}
\label{subsec: real_dataset_main}

We next examine individual and group \LCVI for a learned prediction model using the 2011 subset of the UCI Bike Sharing dataset (\(n=8{,}645\) hourly records), which combines Capital Bikeshare rental counts with calendar and weather measurements~\citep{fanaee2014event}. Following \citet{blesch2025conditional}, we retain hour, working-day status, normalized temperature, normalized humidity, and season. For additional experimental details and results, see \cref{app:bike}.

\begin{figure}[t]
    \centering
    \includegraphics[width=\columnwidth]{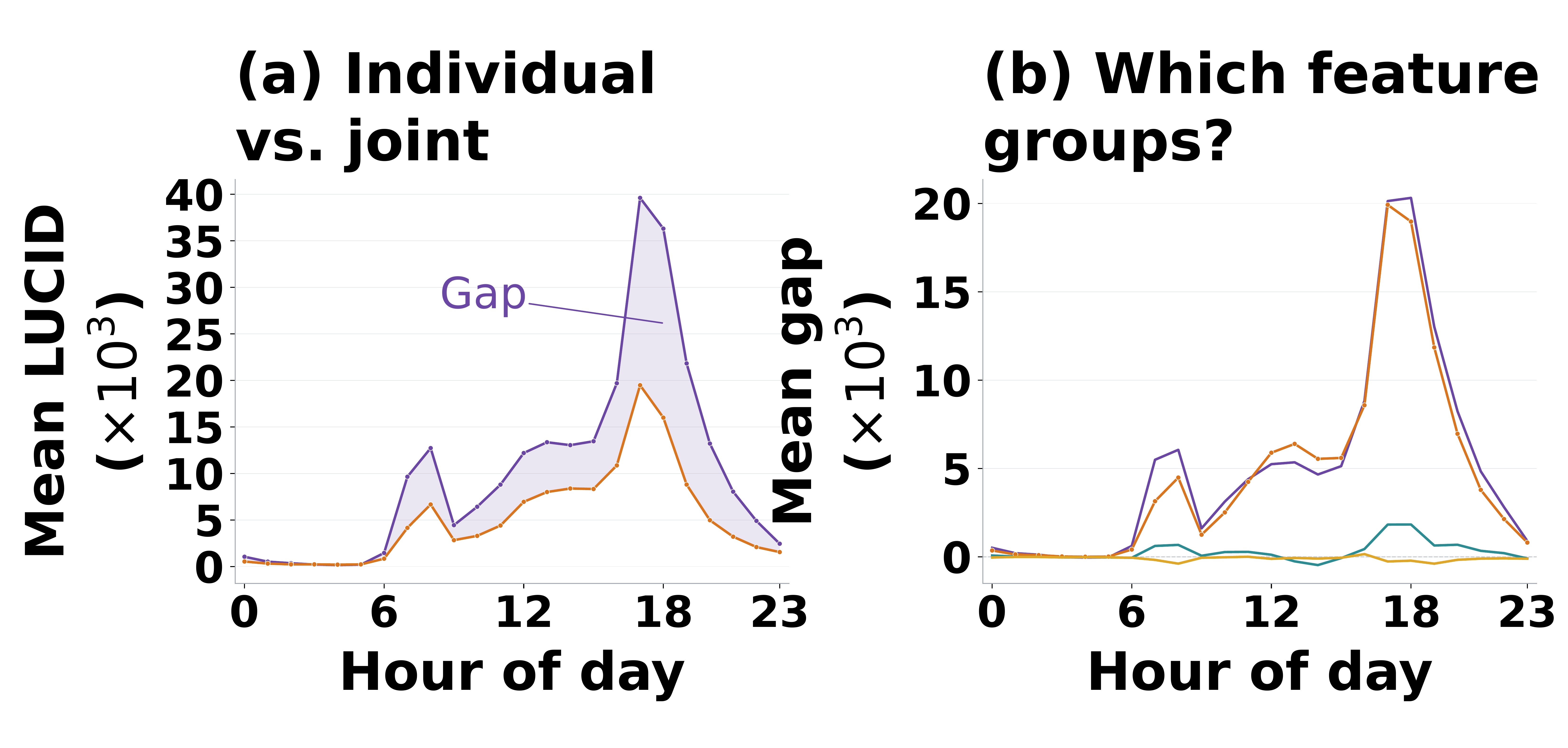}
    \caption{
\textbf{weather-group \LCVI}
\textbf{(a)} The \textcolor{darkpurple}{purple line} shows the joint \LCVI of temperature, humidity, and season, while the \textcolor{darkorange}{orange dashed line} shows the sum of their individual \LCVI{}s; the shaded region highlights their gap.
\textbf{(b)} Group gaps are shown for the \textcolor{darkpurple}{$(T,H,S)$}, \textcolor{darkorange}{$(T,S)$}, \textcolor{teal}{$(T,H)$}, and \textcolor{myyellow}{$(H,S)$} groups.
}
    \label{fig:bike_sharing_group_combination_main_2}
\end{figure}

\paragraph{\LCVI reveals local heterogeneity across real covariate contexts.}
Figure~\ref{fig:bike_teaser}(a) shows that temperature \LCVI varies substantially over the day, so its global importance masks when temperature matters most. As shown in Figure~\ref{fig:bike_teaser}(b), its conditional replacement distribution also shifts across seasons, whereas marginal permutation fails to capture this heterogeneity by drawing replacements from the same marginal distribution throughout. Figure~\ref{fig:main_experiments_bike_sharing_hr} shows a complementary pattern for hour: on working days, importance peaks sharply during the morning and evening commuting periods, whereas this structure is largely absent on non-working days. Unlike temperature, however, hour is only weakly predictable from the remaining features, so its local importance is less constrained by information already captured by those covariates. The observed heterogeneity therefore mainly reflects changes in the model's sensitivity to hour. Together, these results show that \LCVI identifies when a learned model relies on a feature and how its plausible perturbations vary given the remaining covariates. In the next paragraph, we show how group \LCVI complements this individual-level view by capturing shared information among related features, as motivated in \cref{subsec:group_lcvi}.

\paragraph{Group \LCVI complements individual local importance.}
Figure~\ref{fig:bike_sharing_group_combination_main_2}(a) compares the joint \LCVI of the weather group (temperature, humidity, and season) with the sum of their individual \LCVI{}s. Both are small overnight, whereas the group \LCVI substantially exceeds the individual sum during the day, particularly during evening rush hours.  Thus, assessing the weather features individually does not fully capture their joint local predictive importance in these periods. 

Figure~\ref{fig:bike_sharing_group_combination_main_2}(b) shows that the temperature–season pair exhibits a gap similar to that of the full weather group. Its pattern closely tracks that of the full weather group, while the temperature--humidity and humidity--season pairs contribute much smaller gaps. This might be because, for the other two pairs, season or temperature remains in the conditioning set, respectively, and can constrain the joint replacements through its dependence on the selected features. Thus, evaluating these features separately can understate their collective local importance when they share information.

\section{Conclusion}
\label{sec:conclusion}

We present \LCVI, a model-agnostic framework for local variable importance based on the full conditional feature distribution. \LCVI characterizes where and how strongly a fixed model responds to conditionally plausible feature changes, with an interpretable decomposition, convergence guarantees, and Bayesian uncertainty quantification. Its flexible estimation framework accommodates modern conditional density estimators and feature groups. Future work could extend \LCVI to outcome-reference importance and model class reliance \citep{fisher2019all}, moving from explaining a fixed model toward assessing importance for the prediction problem. Also, since global conditional permutation importance is the expectation of \LCVI, accurate local estimation also supports accurate global conditional importance.

\bibliography{references}

\clearpage
\onecolumn
\appendix

\paragraph{Road map.}
\Cref{app:related} summarizes the related work.
\Cref{app:notation} summarizes the notation used throughout the paper.
\Cref{app:theory} provides the complete theoretical results, proofs, and convergence guarantees.
\Cref{app:implementation} presents implementation details, computational complexity, experimental setup, software, and data information.
\Cref{app:dgp} provides complete synthetic-experiment definitions, calculations, and additional results.
\Cref{app:bike} provides additional Bike Sharing analyses and results.
\Cref{app:openxai_scaling} reports runtime and memory scalability results on OpenXAI benchmark datasets.

\providecommand{\LCVI}{LUCID}
\providecommand{\LUCIDapp}{\operatorname{LUCID}}
\providecommand{\LUCIDMainSquaredLabel}{thm:local_lcvi_decomposition}
\providecommand{\LUCIDMainLOCOLabel}{prop:lcvi_loco}
\providecommand{\LUCIDMainBregmanLabel}{prop:lcvi_bregman_decomposition}
\providecommand{\LUCIDMainStabilityLabel}{thm:lcvi_contraction_transfer}
\newtheorem{lucidappprop}{Proposition}[section]
\newtheorem{lucidappcor}[lucidappprop]{Corollary}
\newtheorem{lucidapplemma}[lucidappprop]{Lemma}
\numberwithin{equation}{section}

\section{Related Work}
\label{app:related}
\paragraph{Perturbation-based variable importance.}
\citet{breiman2001random} introduced permutation-based VI in his seminal work on Random Forests, recognizing that if predictions change when shuffling the inputs to the forest (i.e., resampling them from their respective marginal distributions), then the forest likely depends on the feature. This core principle has inspired subsequent approaches to global perturbation-based variable importance that go beyond sampling from marginal distributions, but the sampling distribution in these methods is often restricted and its validity requires additional assumptions. Residual-based approaches \citep{fisher2019all,chamma2023statistically,borgonovo2026trustworthy} estimate the conditional mean of the feature of interest and use the residual distribution to generate counterfactuals, which can yield biased importance estimates when the residual distribution varies with the conditioning covariates, such as under heteroskedasticity or other changes in the conditional residual distribution. Tree- or subgroup-based approaches \citep{strobl2008conditional,debeer2020conditional,molnar2024model,chamma2023statistically} permute or resample feature values within regions defined by tree-based partitions, which may overlook local variation in the conditional distribution, while overly fine partitions leave too few observations for resampling. Approaches based on joint-distribution estimation \citep{blesch2025conditional} instead require modeling the full joint distribution, which can become increasingly challenging as the ambient dimension grows \citep{izbicki2016nonparametric}. 
There are several recent methods that extend permutation-based importance to the local level, including Individual Conditional Importance \citep[ICI;][]{casalicchio2018visualizing} and Conditional Local Importance by Quantile Expectations \citep[CLIQUE;][]{bladen2024conditional}. However, these existing local perturbation-based methods ignore dependence between the perturbed feature and the remaining covariates.
In contrast, \LCVI defines observation-level importance with respect to the full conditional feature distribution. Accurate estimation of this distribution yields replacements that preserve dependence between the feature of interest and the remaining covariates.

\paragraph{Other explanation and importance paradigms.} 
Several related methods target different notions of feature importance. (i) Local explanation methods also explain individual predictions by assigning contributions to features. SHapley Additive exPlanations (SHAP)-based methods \citep{lundberg2017unified,aas2021explaining} evaluate changes in model outputs as features are added or marginalized over, with conditional SHAP using conditional feature densities to account for dependencies between features \citep{aas2021explaining,jullum2025shapr}. However, their estimand differs from \LCVI, as SHAP allocates the deviation of an individual prediction from a baseline across feature coalitions, whereas \LCVI measures the change in model behavior induced by conditionally plausible replacements of a specified feature while holding the remaining context fixed. LIME \citep{ribeiro2016should}, by contrast, constructs a locally linear surrogate approximation of the prediction function, using the local linear regression's parameters as explanations for the black box. Both of these methods can unfaithfully capture nonzero variable attribution to features that are inactive at the observation being explained, especially when the data exhibits discontinuities. As we show in \cref{fig:main_experiments_dgp1_v2}, LIME, SHAP, and SHAPr (a SHAP variant) all yield non-zero feature attribution even when the feature is completely unused by some specific regions of the known outcome generating function.
(ii) Feature-removal methods such as leave-one-covariate-out \citep[LOCO;][]{lei2018distribution} importance measure how much predictive information is lost by retraining the predictive model with all-but-one of the original features. In practice, this change can arise from both information loss due to feature removal and errors introduced by model refitting, making it difficult to disentangle these effects. Moreover, because LOCO compares two different models, it does not directly quantify how the original fitted model relies on the omitted feature. Instead, it describes how much the algorithm that discovered the model relies on the feature \citep{fisher2019all}. At the population level and under squared loss, however, LOCO with optimal reduced predictor is proportional to and closely related to conditional permutation feature importance~\citep{reyero2025conditional}. We further explore this relationship at the observation level in Section~\ref{subsec:lcvi_geometry}. (iii) Grouped VI methods measure the importance of sets of related features jointly to account for shared or redundant information~\citep{gregorutti2015grouped, au2022grouped,chamma2024variable, little2025iloco}. Our framework naturally extends this idea but to the local level, quantifying conditional importance for feature groups at the observation level. 

\section{Notation}
\label{app:notation}

Table~\ref{app:tab:notation} summarizes the notation used in the main theory and this appendix. Throughout, $g$ is a fixed prediction model. The scalar feature of interest is denoted by $Z$, and the vector $\bm{X}$ contains the remaining covariates. For group importance, the feature group and its replacement are written $\bm{Z}$ and $\bm{Z}'$. All population expectations refer to the true covariate distribution. 

Vector symbols use $\bm{X},\bm{x},\bm{W},\bm{w}$, whereas individual coordinates, such as $X_j,x_j$, and the scalar replacements $Z',z'$ remain unbold. The model $g$ is scalar-valued in the squared-loss subsection and may have vector-valued outputs in the Bregman subsection. We use \textcolor{mypurple}{dispersion} and \textcolor{myyellow}{displacement} with the same colors as the main paper.

\begingroup
\small
\renewcommand{\arraystretch}{1.13}
\setlength{\tabcolsep}{6pt}
\begin{longtable}{@{}>{\raggedright\arraybackslash}p{0.25\textwidth}>{\raggedright\arraybackslash}p{0.71\textwidth}@{}}
\caption{Notation for \LCVI\ and its theoretical guarantees.}\label{app:tab:notation}\\
\toprule
Notation & Definition \\
\midrule
\endfirsthead
\multicolumn{2}{l}{\tablename~\thetable\ (continued)}\\
\toprule
Notation & Definition \\
\midrule
\endhead
\midrule
\multicolumn{2}{r}{Continued on next page}\\
\endfoot
\bottomrule
\endlastfoot
$Z,\ z$ & feature of interest and its observed value.\\
$\bm{X}=(X_1,\ldots,X_k),\ \bm{x}$ & Remaining covariates and their observed values; $k$ is the number of conditioning covariates.\\
$\bm{W}=(Z,\bm{X}),\ \bm{w}=(z,\bm{x})$ & Complete feature vector and one observation.\\
$Z',\ z'$ & conditional replacement and its integration variable.\\
$\bm{Z},\ \bm{z},\ \bm{Z}',\ \bm{z}'$ & Feature group, its observed vector, its joint conditional replacement, and the vector integration variable.\\
$P_{Z,\bm{X}},\ P_{\bm{X}},\ P_Z$ & True joint covariate distribution and its marginals.\\
$P(\cdot\mid \bm{x})$ & True conditional distribution of $Z$ given $\bm{X}=\bm{x}$.\\
$p^*(z\mid \bm{x})$ & True conditional density, with respect to the stated dominating measure.\\
$p,\ q,\ \widehat p$ & Candidate conditional densities and a fitted conditional density.\\
$g(z,\bm{x})$ & Fixed prediction model being explained.\\
$\ell(a,b)$ & Scalar-valued loss comparing alternative and reference predictions; for vector arguments, write $\ell(\bm{a},\bm{b})$.\\
$\LUCIDapp(\bm{w})$ & Local importance under the true conditional distribution.\\
$\LUCIDapp[p](\bm{w})$ & Local importance induced by a candidate conditional density $p$.\\
$g_{-Z}(\bm{x})$ & Conditional mean of the fixed model output given $\bm{X}=\bm{x}$; the optimal squared-loss reduced predictor.\\
$\mathrm{LOCO}(\bm{w})$ & Squared discrepancy between $g(z,\bm{x})$ and $g_{-Z}(\bm{x})$.\\
$\mathrm{LOCO}_{\phi}(\bm{w})$ & $D_\phi(g_{-Z}(\bm{x}),g(z,\bm{x}))$, using the same conditional-mean reduced predictor.\\
$I_j(\bm{w})$ & Experimental shorthand for \LCVI\ when the $j$th original feature is selected as $Z$.\\
$\phi,\ D_\phi$ & Scalar-valued strictly convex differentiable generator and its Bregman divergence; vector arguments are written $\bm{a},\bm{b}$, with scalar coordinates $a_c,b_c$.\\
$\mathrm{TV}(p,q)$ & Total variation distance, with the convention $\frac12\int|p-q|$.\\
$d(p,q)$ & Conditional-distribution distance satisfying the main paper distance-control condition.\\
$h(p,q)$ & Integrated Hellinger distance defined in Equation~\eqref{app:eq:hellinger}; $h(p(\cdot\mid \bm{x}),q(\cdot\mid \bm{x}))$ denotes its fixed-context version.\\
$\|\cdot\|_2$ & $L^2(P_{Z,\bm{X}})$ norm for importance functions; $L^2(P_{\bm{X}})$ is written explicitly when used.\\
$C_\ell,\ A$ & Uniform loss bound and constant in the distance-control condition.\\
$\epsilon_n$ & Conditional-density contraction rate, transferred to \LCVI.\\
$\mathcal D_n,\ n$ & I.i.d. covariate sample used to fit the conditional density, and its size.\\
$\bm{w}_i,\ N$ & Observations to be explained, and their number.\\
$\Pi,\ \Pi(\cdot\mid\mathcal D_n)$ & Prior and posterior on conditional densities.\\
$p^{(b)},\ B$ & Retained posterior density draw and number of retained draws.\\
$\overline p$ & Complete posterior mean density.\\
$b_m,\ M$ & Random posterior index for replacement $m$, and number of replacements in an estimate.\\
$Z'_m,\ Z_m'^{(b)}$ & Replacement from a specified sampler, and replacement generated with posterior draw $b$ held fixed.\\
$\widehat{\LUCIDapp}_M[p]$ & Monte Carlo approximation of $\LUCIDapp[p]$ using $M$ replacements.\\
$\mathrm{KL}(p^*\|p),\ V(p^*,p)$ & Integrated log-density ratio and its second moment under the truth.\\
$\mathcal G_n,\ \mathcal N(\epsilon,\mathcal G_n,h)$ & Density sieve and its Hellinger covering number.\\
$L_n(p),\ \psi_n$ & Likelihood ratio and binary test used only in the posterior-tail proof.\\
$\alpha,\ \beta,\ d_0$ & DR-BART smoothness, $\beta=\min\{\alpha,1\}$, and effective dimension, including the response coordinate.\\
$\gamma,\ \widehat Q_{\bm{w}}$ & Tail probability for a $(1-\gamma)$ credible interval, and empirical posterior quantile function at $\bm{w}$.\\
\end{longtable}
\endgroup

\section{Theoretical Results and Proofs}
\label{app:theory}

We first prove the structural results for the true conditional distribution. We then establish stability under conditional-distribution error and distinguish posterior contraction from the moment control needed for posterior averaging.

\subsection{Setup and Global--Local Identity}
\label{app:setup}

Let $\bm{W}=(Z,\bm{X})\sim P_{Z,\bm{X}}$. A conditional replacement is constructed so that
\begin{equation}
Z'\mid \bm{X}\sim P(\cdot\mid \bm{X}),\qquad Z'\perp Z\mid \bm{X}.
\label{app:eq:coupling}
\end{equation}
For a candidate conditional density $p$, write
\begin{equation}
\LUCIDapp[p](\bm{w}):=\int \ell(g(z',\bm{x}),g(z,\bm{x}))\,p(z'\mid \bm{x})\,dz',
\qquad \LUCIDapp(\bm{w}):=\LUCIDapp[p^*](\bm{w}).
\label{app:eq:functional}
\end{equation}
At fixed $\bm{w}=(z,\bm{x})$, expectations over $Z'$ use $P(\cdot\mid \bm{x})$ unless another distribution is specified. All conditional densities are defined with respect to a common dominating measure \(\mu\); for discrete features, the corresponding integrals are sums.

\begin{lucidappprop}[Global--local identity]
\label{app:prop:global-local}
If $\mathbb E|\ell(g(Z',\bm{X}),g(Z,\bm{X}))|<\infty$, then
\begin{equation}
\mathrm{CPFI}=\mathbb E[\LUCIDapp(\bm{W})].
\label{app:eq:global-local}
\end{equation}
\end{lucidappprop}

\begin{proof}
By Equation~\eqref{app:eq:coupling}, the conditional distribution of $Z'$ given $(Z,\bm{X})$ is $P(\cdot\mid \bm{X})$. Thus
\[
\mathbb E[\ell(g(Z',\bm{X}),g(Z,\bm{X}))\mid Z,\bm{X}]=\LUCIDapp(Z,\bm{X}).
\]
Taking expectations yields the stated identity and the prediction-based definition of global CPFI in the main paper, Equation \ref{eq:cpfi}.
\end{proof}

\subsection{Squared-Loss Decomposition and the Relationship with LOCO}
\label{app:squared-loco}

Assume first that $g$ is real-valued and the relevant second moments are finite. Define the conditional-mean reduced predictor by
\begin{equation}
g_{-Z}(\bm{x}):=\mathbb E[g(Z',\bm{x})]=\mathbb E[g(Z,\bm{X})\mid \bm{X}=\bm{x}].
\label{app:eq:reduced-predictor}
\end{equation}
This definition refers to the given model $g$. 

\begin{proof}[Proof of Theorem~\ref{\LUCIDMainSquaredLabel}]
Under squared loss, adding and subtracting the conditional mean gives
\begin{align}
\LUCIDapp(\bm{w})
&=\mathbb E[(g(Z',\bm{x})-g(z,\bm{x}))^2]\notag\\
&=\mathbb E[(g(Z',\bm{x})-\mathbb E[g(Z',\bm{x})])^2]+(\mathbb E[g(Z',\bm{x})]-g(z,\bm{x}))^2\notag\\
&\quad+2(\mathbb E[g(Z',\bm{x})]-g(z,\bm{x}))\mathbb E[g(Z',\bm{x})-\mathbb E[g(Z',\bm{x})]]\notag\\
&=\textcolor{mypurple}{\underbrace{\operatorname{Var}(g(Z',\bm{x}))}_{\text{dispersion}}}
+\textcolor{myyellow}{\underbrace{(g(z,\bm{x})-\mathbb E[g(Z',\bm{x})])^2}_{\text{displacement}}}.
\label{app:eq:squared-decomposition}
\end{align}
The cross term vanishes by Equation~\eqref{app:eq:reduced-predictor}.
\end{proof}

The conditional mean is the optimal squared-loss reduced predictor: for every real number $a$,
\begin{equation}
\mathbb E[(g(Z',\bm{x})-a)^2]=\operatorname{Var}(g(Z',\bm{x}))+(g_{-Z}(\bm{x})-a)^2.
\label{app:eq:reduced-optimality}
\end{equation}
Consequently, $g_{-Z}(\bm{x})$ uniquely minimizes this conditional risk. Recall the local LOCO comparison in the main paper,
\begin{equation}
\mathrm{LOCO}(\bm{w}):=\textcolor{myyellow}{(g(z,\bm{x})-g_{-Z}(\bm{x}))^2}.
\label{app:eq:local-loco}
\end{equation}
A reduced model obtained by finite-sample refitting need not equal this projection, so the following identity does not assert equality for a refitted model.

\begin{proof}[Proof of Proposition~\ref{\LUCIDMainLOCOLabel}]
Substituting Equation~\eqref{app:eq:local-loco} into Equation~\eqref{app:eq:squared-decomposition} gives
\begin{equation}
\LUCIDapp(\bm{w})=\textcolor{mypurple}{\operatorname{Var}(g(Z',\bm{x}))}+\textcolor{myyellow}{\mathrm{LOCO}(\bm{w})}.
\label{app:eq:local-loco-identity}
\end{equation}
\end{proof}

\begin{lucidappcor}[Conditional and global averages]
\label{app:cor:squared-averages}
Following \citet{reyero2025conditional}, if $\mathbb E[g(Z,\bm{X})^2]<\infty$, then
\begin{align}
\mathbb E[\mathrm{LOCO}(\bm{W})\mid \bm{X}]&=\operatorname{Var}(g(Z,\bm{X})\mid \bm{X}),\label{app:eq:conditional-loco}\\
\mathbb E[\LUCIDapp(\bm{W})\mid \bm{X}]&=2\operatorname{Var}(g(Z,\bm{X})\mid \bm{X}).\label{app:eq:conditional-lucid}
\end{align}
In particular, \textcolor{mypurple}{dispersion} and \textcolor{myyellow}{displacement} have the same population mean, and
\begin{equation}
\mathrm{CPFI}=2\mathbb E[\mathrm{LOCO}(\bm{W})]
=2\mathbb E[\operatorname{Var}(g(Z,\bm{X})\mid \bm{X})].
\label{app:eq:global-factor-two}
\end{equation}
\end{lucidappcor}

\begin{proof}
Equation~\eqref{app:eq:conditional-loco} is the definition of conditional variance. At $\bm{W}=(Z,\bm{X})$, the \textcolor{mypurple}{dispersion} in Equation~\eqref{app:eq:local-loco-identity} equals $\operatorname{Var}(g(Z,\bm{X})\mid \bm{X})$ because $Z'$ and $Z$ have the same conditional distribution. Taking conditional expectations proves Equation~\eqref{app:eq:conditional-lucid}. Averaging again and applying Proposition~\ref{app:prop:global-local} proves Equation~\eqref{app:eq:global-factor-two}.
\end{proof}

\subsection{Bregman Decomposition and Global Identities}
\label{app:bregman}

Let the model outputs lie in a convex domain in a finite-dimensional Euclidean space, and let $\phi$ be strictly convex and continuously differentiable on its relative interior. Its Bregman divergence \citep{BREGMAN1967200} is
\begin{equation}
D_\phi(\bm{a},\bm{b}):=\phi(\bm{a})-\phi(\bm{b})-\langle\nabla\phi(\bm{b}),\bm{a}-\bm{b}\rangle.
\label{app:eq:bregman-definition}
\end{equation}
For each fixed $\bm{x}$, assume that $g(Z',\bm{x})$ has a finite mean and that both $g(z,\bm{x})$ and this mean, $g_{-Z}(\bm{x})$, lie in the relative interior of the domain of $\phi$. Assume also that the expected Bregman divergences appearing below are finite. For vector-valued outputs, $g_{-Z}(\bm{x})$ is obtained by taking the conditional expectation of each output component in Equation~\eqref{app:eq:reduced-predictor}. If the output domain lies in a lower-dimensional affine space, gradients are taken within that space; for example, for probability vectors, derivatives are taken along directions that preserve the constraint that the components sum to one.

\begin{proof}[Proof of Proposition~\ref{\LUCIDMainBregmanLabel}]
Expanding Equation~\eqref{app:eq:bregman-definition} gives the three-point identity
\begin{align*}
D_\phi(g(Z',\bm{x}),g(z,\bm{x}))
&=\phi(g(Z',\bm{x}))-\phi(g(z,\bm{x}))
-\left\langle\nabla\phi(g(z,\bm{x})),g(Z',\bm{x})-g(z,\bm{x})\right\rangle\\
&=\phi(g(Z',\bm{x}))-\phi(\mathbb E[g(Z',\bm{x})])
+\phi(\mathbb E[g(Z',\bm{x})])-\phi(g(z,\bm{x}))\\
&\quad-\left\langle\nabla\phi(g(z,\bm{x})),g(Z',\bm{x})-\mathbb E[g(Z',\bm{x})]\right\rangle\\
&\quad-\left\langle\nabla\phi(g(z,\bm{x})),\mathbb E[g(Z',\bm{x})]-g(z,\bm{x})\right\rangle\\
&=\Bigl[\phi(g(Z',\bm{x}))-\phi(\mathbb E[g(Z',\bm{x})])\\
&\qquad-\left\langle\nabla\phi(\mathbb E[g(Z',\bm{x})]),g(Z',\bm{x})-\mathbb E[g(Z',\bm{x})]\right\rangle\Bigr]\\
&\quad+\Bigl[\phi(\mathbb E[g(Z',\bm{x})])-\phi(g(z,\bm{x}))\\
&\qquad-\left\langle\nabla\phi(g(z,\bm{x})),\mathbb E[g(Z',\bm{x})]-g(z,\bm{x})\right\rangle\Bigr]\\
&\quad+\left\langle\nabla\phi(\mathbb E[g(Z',\bm{x})])-\nabla\phi(g(z,\bm{x})),
g(Z',\bm{x})-\mathbb E[g(Z',\bm{x})]\right\rangle\\
&=D_\phi(g(Z',\bm{x}),\mathbb E[g(Z',\bm{x})])
+D_\phi(\mathbb E[g(Z',\bm{x})],g(z,\bm{x}))\\
&\quad+\left\langle\nabla\phi(\mathbb E[g(Z',\bm{x})])-\nabla\phi(g(z,\bm{x})),
g(Z',\bm{x})-\mathbb E[g(Z',\bm{x})]\right\rangle.
\end{align*}
The last term has mean zero. Hence, for $\ell=D_\phi$,
\begin{equation}
\LUCIDapp(\bm{w})=\textcolor{mypurple}{\underbrace{\mathbb E[D_\phi(g(Z',\bm{x}),\mathbb E[g(Z',\bm{x})])]}_{\text{dispersion}}}
+\textcolor{myyellow}{\underbrace{D_\phi(\mathbb E[g(Z',\bm{x})],g(z,\bm{x}))}_{\text{displacement}}}.
\label{app:eq:bregman-decomposition}
\end{equation}
\end{proof}
For $\phi(\bm{a})=\|\bm{a}\|^2$, the Bregman divergence reduces to squared Euclidean distance, so Equation~\eqref{app:eq:bregman-decomposition} becomes
\begin{equation}
\LUCIDapp(\bm{w})
=\textcolor{mypurple}{\underbrace{\mathbb E[\|g(Z',\bm{x})-g_{-Z}(\bm{x})\|^2]}_{\text{dispersion}}}
+\textcolor{myyellow}{\underbrace{\|g_{-Z}(\bm{x})-g(z,\bm{x})\|^2}_{\text{displacement}}}.
\label{app:eq:local-bregman-squared}
\end{equation}
For scalar outputs, the dispersion term is $\operatorname{Var}(g(Z',\bm{x}))$, recovering the squared-loss decomposition.

For strictly positive probability-vector outputs and negative entropy $\phi(\bm{a})=\sum_c a_c\log a_c$, the Bregman divergence is $\mathrm{KL}(\bm{a}\|\bm{b})$. Thus, whenever the expectations are finite,
\begin{equation}
\LUCIDapp(\bm{w})
=\textcolor{mypurple}{\underbrace{\mathbb E[\mathrm{KL}(g(Z',\bm{x})\|g_{-Z}(\bm{x}))]}_{\text{dispersion}}}
+\textcolor{myyellow}{\underbrace{\mathrm{KL}(g_{-Z}(\bm{x})\|g(z,\bm{x}))}_{\text{displacement}}}.
\label{app:eq:local-kl}
\end{equation}

\paragraph{The direction of the LOCO comparison.}
The main paper places the alternative prediction first and the reference prediction second in $\ell$ to show that we measure the alternative prediction using the original model output as the reference. Keeping this order, the Bregman comparison with the conditional-mean reduced predictor is
\begin{equation}
\mathrm{LOCO}_\phi(\bm{w}):=\textcolor{myyellow}{D_\phi(g_{-Z}(\bm{x}),g(z,\bm{x}))}.
\label{app:eq:bregman-loco}
\end{equation}
It is exactly the \textcolor{myyellow}{displacement} term in Equation~\eqref{app:eq:bregman-decomposition}. Here, we retain the conditional-mean reduced predictor, which need not minimize the Bregman loss when placed in its first argument. Reversing the comparison leads to an analogous formulation in dual coordinates, with a corresponding dual-mean predictor. As this extension is beyond the main scope of our paper, we omit its details and refer readers to \citet{adlam2022understanding} for further discussion about Bregman divergence.

\begin{lucidappcor}[Zero importance]
\label{app:cor:zero}
For squared loss or a well-defined Bregman divergence with a strictly convex generator,
\begin{equation}
\LUCIDapp(\bm{w})=0\quad\Longleftrightarrow\quad g(Z',\bm{x})=g(z,\bm{x})\quad P(\cdot\mid \bm{x})\text{-almost surely}.
\label{app:eq:zero}
\end{equation}
\end{lucidappcor}

\begin{proof}
By the definition of \LCVI,
\[
\LUCIDapp(\bm{w})
=\mathbb E\!\left[\ell(g(Z',\bm{x}),g(z,\bm{x}))\right],
\qquad Z'\sim P(\cdot\mid\bm{x}).
\]
For squared loss, $\ell(\bm{a},\bm{b})=\|\bm{a}-\bm{b}\|^2$ is nonnegative and equals zero if and only if $\bm{a}=\bm{b}$. For Bregman divergence, strict convexity of $\phi$ implies
\[
\phi(\bm{a})>\phi(\bm{b})+\langle\nabla\phi(\bm{b}),\bm{a}-\bm{b}\rangle
\qquad\text{whenever }\bm{a}\neq\bm{b}.
\]
Together with Equation~\eqref{app:eq:bregman-definition}, this gives $D_\phi(\bm{a},\bm{b})\geq0$, with equality if and only if $\bm{a}=\bm{b}$. Thus, if $\LUCIDapp(\bm{w})=0$, the nonnegative loss inside the expectation must vanish $P(\cdot\mid\bm{x})$-almost surely, implying $g(Z',\bm{x})=g(z,\bm{x})$ almost surely. Conversely, this almost-sure equality makes the loss zero almost surely and hence gives $\LUCIDapp(\bm{w})=0$.
\end{proof}

\begin{lucidappprop}[Global Bregman identity]
\label{app:prop:global-bregman}
Under the preceding conditions and population integrability of the displayed terms,
\begin{equation}
\mathrm{CPFI}=\mathbb E[D_\phi(g(Z,\bm{X}),g_{-Z}(\bm{X}))]
+\mathbb E[D_\phi(g_{-Z}(\bm{X}),g(Z,\bm{X}))].
\label{app:eq:global-bregman}
\end{equation}
\end{lucidappprop}

\begin{proof}
Average Equation~\eqref{app:eq:bregman-decomposition} over $(Z,\bm{X})$. Conditional on $\bm{X}$, the outputs $g(Z',\bm{X})$ and $g(Z,\bm{X})$ have the same distribution and conditional mean $g_{-Z}(\bm{X})$. Therefore, the average \textcolor{mypurple}{dispersion} equals $\mathbb E[D_\phi(g(Z,\bm{X}),g_{-Z}(\bm{X}))]$. The average \textcolor{myyellow}{displacement} is the second term in Equation~\eqref{app:eq:global-bregman}. 
\end{proof}

When the relevant generator expectations are finite, the first term also satisfies
\begin{equation}
\mathbb E[D_\phi(g(Z,\bm{X}),g_{-Z}(\bm{X}))]
=\mathbb E[\phi(g(Z,\bm{X}))]-\mathbb E[\phi(g_{-Z}(\bm{X}))],
\label{app:eq:jensen-gap}
\end{equation}
since the conditional mean of $g(Z,\bm{X})-g_{-Z}(\bm{X})$ given $\bm{X}$ is zero. For $\phi(\bm{a})=\|\bm{a}\|^2$, the two divergences in Equation~\eqref{app:eq:global-bregman} coincide, recovering the factor-of-two identity. For a strictly positive probability-vector output and negative entropy $\phi(\bm{a})=\sum_c a_c\log a_c$, direct substitution gives
\begin{equation}
\mathrm{CPFI}=\mathbb E[\mathrm{KL}(g(Z,\bm{X})\|g_{-Z}(\bm{X}))]
+\mathbb E[\mathrm{KL}(g_{-Z}(\bm{X})\|g(Z,\bm{X}))].
\label{app:eq:global-kl}
\end{equation}
These two expectations need not agree. Consequently, even though local \textcolor{myyellow}{displacement} is the comparison in Equation~\eqref{app:eq:bregman-loco}, global CPFI need not be twice its population mean.

\paragraph{Feature groups.}
All results above also hold for a feature group $\bm{Z}$, with $\bm{Z}'$ drawn jointly conditional on the remaining covariates $\bm{X}$. In the corresponding formulas, replace $Z,z,Z',z'$ by $\bm{Z},\bm{z},\bm{Z}',\bm{z}'$, respectively. The proofs use the conditional distribution of the model output and do not require scalar feature replacements. 

\subsection{Stability Under Conditional-Distribution Error}
\label{app:stability}

Throughout this subsection, suppose
\begin{equation}
|\ell(g(z',\bm{x}),g(z,\bm{x}))|\le C_\ell<\infty
\label{app:eq:bounded-loss}
\end{equation}
for $P_{Z,\bm{X}}$-almost every $(z,\bm{x})$ and replacement values in the supports of the candidate distributions. Bounded outputs of a fixed model suffice for squared loss. Write $\|\cdot\|_2=\|\cdot\|_{L^2(P_{Z,\bm{X}})}$.

\begin{proof}[Proof of Theorem~\ref{\LUCIDMainStabilityLabel}]
For every pair $p,q$ satisfying Equation~\eqref{app:eq:bounded-loss},
\begin{align}
|\LUCIDapp[p](\bm{w})-\LUCIDapp[q](\bm{w})|
&\le C_\ell\int|p(z'\mid \bm{x})-q(z'\mid \bm{x})|\,dz'\notag\\
&=2C_\ell\,\mathrm{TV}\{p(\cdot\mid \bm{x}),q(\cdot\mid \bm{x})\}.
\label{app:eq:pointwise-stability}
\end{align}
The right-hand side depends only on $\bm{x}$. Squaring and integrating with respect to $P_{Z,\bm{X}}$, and applying the main paper distance-control assumption, yields
\begin{align}
\|\LUCIDapp[p]-\LUCIDapp[q]\|_2
&\le2C_\ell\|\mathrm{TV}\{p(\cdot\mid \bm{X}),q(\cdot\mid \bm{X})\}\|_{L^2(P_{\bm{X}})}\notag\\
&\le2C_\ell A\,d(p,q).
\label{app:eq:stability}
\end{align}
Taking $p=\widehat p$ and $q=p^*$ proves the stated $O_P(\epsilon_n)$ rate whenever $d(\widehat p,p^*)=O_P(\epsilon_n)$.
\end{proof}

Local estimation accuracy also controls the error in population-average importance. Since $\mathrm{CPFI}=\mathbb E[\LUCIDapp[p^*](\bm{W})]$, Cauchy--Schwarz gives
\begin{equation}
\begin{aligned}
\left|\mathbb E[\LUCIDapp[p](\bm{W})]-\mathrm{CPFI}\right|
&=\left|\mathbb E\!\left[\LUCIDapp[p](\bm{W})-\LUCIDapp[p^*](\bm{W})\right]\right|\\
&\leq\left(\mathbb E\!\left[\left|\LUCIDapp[p](\bm{W})-\LUCIDapp[p^*](\bm{W})\right|^2\right]\right)^{1/2}\\
&=\|\LUCIDapp[p]-\LUCIDapp[p^*]\|_2.
\end{aligned}
\label{app:eq:global-error-transfer}
\end{equation}
Here, expectations are taken over the true covariate distribution of $\bm{W}$. Thus, the root mean squared error of the local importance function bounds the absolute error of its population average.

When the population average is approximated using $N$ evaluation observations, the total error decomposes as
\begin{equation}
\begin{aligned}
\frac1N\sum_{i=1}^N\LUCIDapp[p](\bm{w}_i)-\mathrm{CPFI}
&=\underbrace{\frac1N\sum_{i=1}^N\LUCIDapp[p](\bm{w}_i)
-\mathbb E[\LUCIDapp[p](\bm{W})]}_{\text{evaluation-sample error}}\\
&\quad+\underbrace{\mathbb E[\LUCIDapp[p](\bm{W})]-\mathrm{CPFI}}_{\text{conditional-density estimation error}}.
\end{aligned}
\label{app:eq:global-error-decomposition}
\end{equation}
Equation~\eqref{app:eq:global-error-transfer} controls the second term. The first is a separate sampling error arising from the finite evaluation sample.

\paragraph{Hellinger distance.}
Define
\begin{equation}
h^2(p,q):=\int\!\int\{\sqrt{p(z'\mid \bm{x})}-\sqrt{q(z'\mid \bm{x})}\}^2\,dz'\,P_{\bm{X}}(d\bm{x}).
\label{app:eq:hellinger}
\end{equation}
The corresponding fixed-context distance omits the integral over $P_{\bm{X}}$. Factoring $p-q=(\sqrt p-\sqrt q)(\sqrt p+\sqrt q)$ and applying Cauchy--Schwarz gives
\[
\int|p(z'\mid \bm{x})-q(z'\mid \bm{x})|\,dz'
\le 2h(p(\cdot\mid \bm{x}),q(\cdot\mid \bm{x})),
\]
because $\int(\sqrt p+\sqrt q)^2\le4$. Thus the distance-control condition holds with $d=h$ and $A=1$, and
\begin{equation}
\|\LUCIDapp[p]-\LUCIDapp[q]\|_2\le2C_\ell h(p,q).
\label{app:eq:hellinger-stability}
\end{equation}
The normalization in Equation~\eqref{app:eq:hellinger} also gives $h^2(p,q)\le2$.

\subsection{Posterior Contraction and Posterior Moment Control}
\label{app:posterior-theory}
We now establish convergence guarantees for DR-BART-based \citep{orlandi2021density} \LCVI under the regularity, prior, and bounded-loss conditions stated below. \textbf{Readers interested only in posterior contraction for \LCVI may proceed directly to Corollary \ref{app:cor:posterior-transfer} and then Section~\ref{app:drbart}, where we verify the conditions needed to apply Equation~\eqref{app:eq:posterior-transfer}. }The remainder of this subsection addresses a further question relevant to our implementation: does the posterior mean conditional density
\[
\bar p(z\mid\bm x)
:=\mathbb E_{\Pi}[p(z\mid\bm x)\mid\mathcal D_n]
\]
inherit the same convergence rate? Posterior contraction alone does not automatically guarantee this rate, because the posterior mean also includes contributions from the posterior tails. We therefore establish quantitative tail bounds and posterior second-moment control, which allow us to retain the rate when passing from posterior draws to the posterior mean density and its induced \LCVI.

Let $\mathcal D_n=\{(Z_i,\bm{X}_i)\}_{i=1}^n$ be i.i.d. from $P_{Z,\bm{X}}$, and let $\Pi(\cdot\mid\mathcal D_n)$ be the posterior on conditional densities. All $O_P$ statements are with respect to the randomness of the data under the true sampling distribution, together with any posterior or Monte Carlo sampling used to compute the estimator.

\begin{lucidappcor}[Posterior contraction transfer]
\label{app:cor:posterior-transfer}
Suppose Equation~\eqref{app:eq:bounded-loss} holds over the posterior support. If, for some $K>0$,
\[
\Pi\{d(p,p^*)>K\epsilon_n\mid\mathcal D_n\}\longrightarrow0
\quad\text{in probability},
\]
and $d$ satisfies the distance-control condition with constant $A$, then
\begin{equation}
\Pi\{\|\LUCIDapp[p]-\LUCIDapp[p^*]\|_2>2C_\ell A K\epsilon_n\mid\mathcal D_n\}
\longrightarrow0
\label{app:eq:posterior-transfer}
\end{equation}
in probability.
\end{lucidappcor}

\begin{proof}
By Equation~\eqref{app:eq:stability} and the distance-control condition,
\[
\|\LUCIDapp[p]-\LUCIDapp[p^*]\|_2
\le 2C_\ell A\,d(p,p^*)
\]
over the posterior support. Hence, an importance error exceeding $2C_\ell A K\epsilon_n$ requires $d(p,p^*)>K\epsilon_n$, so
\[
\Pi\{\|\LUCIDapp[p]-\LUCIDapp[p^*]\|_2>2C_\ell A K\epsilon_n\mid\mathcal D_n\}
\le
\Pi\{d(p,p^*)>K\epsilon_n\mid\mathcal D_n\}.
\]
The right-hand side converges to zero in probability by assumption, proving the claim.
\end{proof}

This conclusion describes concentration of the posterior distribution. To establish the same rate for posterior averages, we next control a posterior second moment. The following lemma makes explicit the additional consequence of the standard sieve-and-testing conditions used in density contraction proofs. Its argument follows the testing framework of \citet{ghosal2000convergence,orlandi2021density}
.

Define the integrated log-likelihood moments by
\begin{equation}
\mathrm{KL}(p^*\|p):=\mathbb E_{P_{Z,\bm X}}\!\left[\log\frac{p^*(Z\mid \bm{X})}{p(Z\mid \bm{X})}\right],
\qquad
V(p^*,p):=\mathbb E_{P_{Z,\bm X}}\!\left[\left\{\log\frac{p^*(Z\mid \bm{X})}{p(Z\mid \bm{X})}\right\}^2\right].
\label{app:eq:kl-moments}
\end{equation}

\begin{lucidapplemma}[Posterior tail and second-moment control]
\label{app:lem:posterior-moments}
Let $\epsilon_n\to0$ and $n\epsilon_n^2\to\infty$. Suppose there are constants $C_1,\ldots,C_5>0$ and deterministic density sieves $\mathcal G_n$ such that, for all sufficiently large $n$,
\begin{align}
&G1: \Pi\{p:\mathrm{KL}(p^*\|p)\le C_1^2\epsilon_n^2,\ V(p^*,p)\le C_1^2\epsilon_n^2\}
\ge e^{-C_2n\epsilon_n^2},\label{app:eq:prior-thickness}\\
&G2: \Pi(\mathcal G_n^c)
\le C_3e^{-(C_2+C_1^2+4)n\epsilon_n^2},\label{app:eq:sieve-tail}\\
&G3:\log\mathcal N(C_5\epsilon_n,\mathcal G_n,h)
\le C_4n\epsilon_n^2.\label{app:eq:sieve-entropy}
\end{align}
Then, for a sufficiently large fixed $K$,
\begin{equation}
P_{Z,\bm{X}}^{\otimes n}\!\left[
\Pi\{h(p,p^*)>K\epsilon_n\mid\mathcal D_n\}>2e^{-n\epsilon_n^2}
\right]
\le\frac{C_1^2}{n\epsilon_n^2}+3e^{-n\epsilon_n^2}
\label{app:eq:posterior-tail}
\end{equation}
for all sufficiently large $n$. If also $\log(1/\epsilon_n)=o(n\epsilon_n^2)$, then
\begin{equation}
\mathbb E_\Pi[h^2(p,p^*)\mid\mathcal D_n]=O_P(\epsilon_n^2).
\label{app:eq:density-second-moment}
\end{equation}
Under Equation~\eqref{app:eq:bounded-loss}, it follows that
\begin{equation}
\mathbb E_\Pi[\|\LUCIDapp[p]-\LUCIDapp[p^*]\|_2^2\mid\mathcal D_n]
=O_P(\epsilon_n^2).
\label{app:eq:importance-second-moment}
\end{equation}
\end{lucidapplemma}

\begin{proof}

Write
\[
L_n(p):=\prod_{i=1}^n\frac{p(Z_i\mid\bm{X}_i)}{p^*(Z_i\mid\bm{X}_i)}.
\]

\paragraph{Proof of Equation \ref{app:eq:posterior-tail}}
By Bayes' theorem, the posterior tail probability  $\Pi\{h(p,p^*)>K\epsilon_n\mid\mathcal D_n\}$ can be written as

\begin{align*}
 \Pi\{h(p,p^*)>K\epsilon_n\mid\mathcal D_n\} &= \frac{\int_{h(p,p^*)>K\epsilon_n} \prod_{i=1}^{n} p(Z_i\mid \bm X_i)\,d\Pi(p)}{\int \prod_{i=1}^{n} p(Z_i\mid \bm X_i)\,d\Pi(p)} \\
 &=\frac{\int_{h(p,p^*)>K\epsilon_n} \frac{\prod_{i=1}^{n} p(Z_i\mid \bm X_i)}{\prod_{i=1}^{n} p^*(Z_i\mid \bm X_i)}\,d\Pi(p)}{\int  \frac{\prod_{i=1}^{n} p(Z_i\mid \bm X_i)}{\prod_{i=1}^{n} p^*(Z_i\mid \bm X_i)}\,d\Pi(p)}\\
 &=\frac{\int_{h(p,p^*)>K\epsilon_n} L_n(p)\,d\Pi(p)}{\int  L_n(p)\,d\Pi(p)}
\end{align*}

Our goal is to show that this posterior tail probability $\Pi\{h(p,p^*)>K\epsilon_n\mid\mathcal D_n\}$ is large only on an event of small sampling probability.

\textbf{(i) The denominator $\int L_n(p)\,d\Pi(p)$ is large with high probability}

We consider the denominator $\int L_n(p)\,d\Pi(p)$. Let $\mathcal B_n$ be the KL neighborhood in Equation~\eqref{app:eq:prior-thickness}, and let $\Pi_{\mathcal B_n}$ denote the prior normalized on $\mathcal B_n$, defined by
$$d\,\Pi_{\mathcal{B}_n}(p) := \frac{\mathbf 1_{\{p\in\mathcal{B}_n\}}d\,\Pi(p)}{\Pi(\mathcal B_n)}.$$ 
By G1 (Equation \ref{app:eq:prior-thickness}), every \(p\in\mathcal B_n\) satisfies $\mathrm{KL}(p^*\|p)\le C_1^2\epsilon_n^2,\ V(p^*,p)\le C_1^2\epsilon_n^2$ and $\mathcal B_n$ satisfies $\Pi(\mathcal B_n)\geq e^{-C_2 n \epsilon^2_n}$.

Since the prior mass $\Pi(\mathcal B_n)$ is bounded below, we restrict the integral to $\mathcal B_n$ and express it in terms of the normalized prior $\Pi_{\mathcal B_n}$:
\[
\int L_n(p)\,d\Pi(p)
\ge \int_{\mathcal B_n}L_n(p)\,d\Pi(p)
= \Pi(\mathcal B_n)\int L_n(p)\,d\Pi_{\mathcal B_n}(p).
\]

Taking logarithms and applying Jensen's inequality gives 
\[
\log\int L_n(p)\,d\Pi(p)
\ge \log\Pi(\mathcal B_n)
+\int\log L_n(p)\,d\Pi_{\mathcal B_n}(p).
\]

\textbf{For the first term:} $ \log\Pi(\mathcal B_n)\geq -C_2n\epsilon_n^2$.

\textbf{For the second term:} We first compute the mean and variance of the second term and then use Chebyshev's inequality to show that the \textbf{denominator is not too small with high probability}.
\begin{align*}
    \int\log L_n(p)\,d\Pi_{\mathcal B_n}(p) &= \int\log \prod_{i=1}^n\frac{p(Z_i\mid\bm{X}_i)}{p^*(Z_i\mid\bm{X}_i)}\,d\Pi_{\mathcal B_n}(p)\\
    &=\int\sum_{i=1}^n\log \frac{p(Z_i\mid\bm{X}_i)}{p^*(Z_i\mid\bm{X}_i)}\,d\Pi_{\mathcal B_n}(p)\\
    &=\sum_{i=1}^n\int\log \frac{p(Z_i\mid\bm{X}_i)}{p^*(Z_i\mid\bm{X}_i)}\,d\Pi_{\mathcal B_n}(p)
\end{align*}

Let's denote 
\begin{align*}
    A_i=\int\log \frac{p(Z_i\mid\bm{X}_i)}{p^*(Z_i\mid\bm{X}_i)}\,d\Pi_{\mathcal B_n}(p)
\end{align*}

Since each $A_i$ is a measurable function of $(Z_i,\bm X_i)$, the variables $A_1,\ldots,A_n$ are independent under the true sampling distribution $P_{Z,\bm X}^{\otimes n}$. We have 
\begin{align*}
    \mathbb E_{P_{Z,\bm X}^{\otimes n}}\!\left[\int\log L_n(p)\,d\Pi_{\mathcal B_n}(p)\right] &=\mathbb E_{P_{Z,\bm X}^{\otimes n}}\!\left[\sum_{i=1}^n A_i\right]\\
    &=n\mathbb E_{P_{Z,\bm X}}\!\left[A_1\right]\\
    &=n\mathbb E_{P_{Z,\bm X}}\!\left[\int\log \frac{p(Z_1\mid\bm{X}_1)}{p^*(Z_1\mid\bm{X}_1)}\,d\Pi_{\mathcal B_n}(p)\right]
    \\
    &=n\int\mathbb E_{P_{Z,\bm X}}\!\left[\log \frac{p(Z_1\mid\bm{X}_1)}{p^*(Z_1\mid\bm{X}_1)}\right]\,d\Pi_{\mathcal B_n}(p)\\
    &=-n\int \mathrm{KL}(p^*\|p)\,d\Pi_{\mathcal B_n}(p)\\
    &\geq -n C_1^2\epsilon_n^2\int \,d\Pi_{\mathcal B_n}(p)\\
    &=-n C_1^2\epsilon_n^2
\end{align*}

\begin{align*}
    \operatorname{Var}_{P_{Z,\bm X}^{\otimes n}}\!\left[\int\log L_n(p)\,d\Pi_{\mathcal B_n}(p)\right] &=\operatorname{Var}_{P_{Z,\bm X}^{\otimes n}}\!\left[\sum_{i=1}^n A_i\right]\\
    &=n\operatorname{Var}_{P_{Z,\bm X}}\!\left[A_1\right]\\
    &=n\operatorname{Var}_{P_{Z,\bm X}}\!\left[\int\log \frac{p(Z_1\mid\bm{X}_1)}{p^*(Z_1\mid\bm{X}_1)}\,d\Pi_{\mathcal B_n}(p)\right]
    \\
   &=n \mathbb E_{P_{Z,\bm X}}\!\left[\left(\int\log \frac{p(Z_1\mid\bm{X}_1)}{p^*(Z_1\mid\bm{X}_1)}\,d\Pi_{\mathcal B_n}(p)-\mathbb E_{P_{Z,\bm X}}\!\left[\int\log \frac{p(Z_1\mid\bm{X}_1)}{p^*(Z_1\mid\bm{X}_1)}\,d\Pi_{\mathcal B_n}(p)\right]\right)^2\right]\\
    &=n \mathbb E_{P_{Z,\bm X}}\!\left[\left(\int\log \frac{p(Z_1\mid\bm{X}_1)}{p^*(Z_1\mid\bm{X}_1)}\,d\Pi_{\mathcal B_n}(p)-\int\mathbb E_{P_{Z,\bm X}}\!\left[\log \frac{p(Z_1\mid\bm{X}_1)}{p^*(Z_1\mid\bm{X}_1)}\right]\,d\Pi_{\mathcal B_n}(p)\right)^2\right]\\
    &=n \mathbb E_{P_{Z,\bm X}}\!\left[\left(\int\log \frac{p(Z_1\mid\bm{X}_1)}{p^*(Z_1\mid\bm{X}_1)}-\mathbb E_{P_{Z,\bm X}}\!\left[\log \frac{p(Z_1\mid\bm{X}_1)}{p^*(Z_1\mid\bm{X}_1)}\right]\,d\Pi_{\mathcal B_n}(p)\right)^2\right]\\
    &\leq n \mathbb E_{P_{Z,\bm X}}\!\left[\int\left(\log \frac{p(Z_1\mid\bm{X}_1)}{p^*(Z_1\mid\bm{X}_1)}-\mathbb E_{P_{Z,\bm X}}\!\left[\log \frac{p(Z_1\mid\bm{X}_1)}{p^*(Z_1\mid\bm{X}_1)}\right]\right)^2\,d\Pi_{\mathcal B_n}(p)\right]\\
    &= n \int \mathbb E_{P_{Z,\bm X}}\!\left[\left(\log \frac{p(Z_1\mid\bm{X}_1)}{p^*(Z_1\mid\bm{X}_1)}-\mathbb E_{P_{Z,\bm X}}\!\left[\log \frac{p(Z_1\mid\bm{X}_1)}{p^*(Z_1\mid\bm{X}_1)}\right]\right)^2\right]\,d\Pi_{\mathcal B_n}(p)\\
    &\leq n\int\mathbb E_{P_{Z,\bm X}}\!\left[\left(\log \frac{p(Z_1\mid\bm{X}_1)}{p^*(Z_1\mid\bm{X}_1)}\right)^2\right]\,d\Pi_{\mathcal B_n}(p)\\
    &=n\int V(p^*,p)\,d\Pi_{\mathcal B_n}(p)\\
    &\leq n C_1^2\epsilon_n^2\int \,d\Pi_{\mathcal B_n}(p)\\
    &=n C_1^2\epsilon_n^2
\end{align*}

Thus, the expectation is at least $-C_1^2n\epsilon_n^2$ and the variance is at most $C_1^2n\epsilon_n^2$.

On the event that $\int\log L_n(p)\,d\Pi_{\mathcal B_n}(p)\geq -(C_1^2+1)n\epsilon_n^2$, the preceding bounds imply $\log\int L_n(p)\,d\Pi(p)\geq -(C_1^2+1+C_2)n\epsilon_n^2$. Chebyshev’s inequality therefore yields
 \begin{align*}
     P_{Z,\bm{X}}^{\otimes n}\!\left\{
\int L_n(p)\,d\Pi(p)
<e^{-(C_2+C_1^2+1)n\epsilon_n^2}
\right\}&=P_{Z,\bm{X}}^{\otimes n}\!\left\{
\log \int L_n(p)\,d\Pi(p)
<{-(C_2+C_1^2+1)n\epsilon_n^2}
\right\}\\
&\leq P_{Z,\bm{X}}^{\otimes n}\!\left\{
\int\log L_n(p)\,d\Pi_{\mathcal B_n}(p)< -(C_1^2+1)n\epsilon_n^2
\right\}\\
&= P_{Z,\bm{X}}^{\otimes n}\!\left\{
\int\log L_n(p)\,d\Pi_{\mathcal B_n}(p)+C_1^2n\epsilon_n^2< -n\epsilon_n^2
\right\}\\
&\leq P_{Z,\bm{X}}^{\otimes n}\!\left\{
\int\log L_n(p)\,d\Pi_{\mathcal B_n}(p)-\mathbb E_{P_{Z,\bm X}^{\otimes n}}\!\left[\int\log L_n(p)\,d\Pi_{\mathcal B_n}(p)\right]< -n\epsilon_n^2
\right\}\\
&\leq P_{Z,\bm{X}}^{\otimes n}\!\left\{
|\int\log L_n(p)\,d\Pi_{\mathcal B_n}(p)-\mathbb E_{P_{Z,\bm X}^{\otimes n}}\!\left[\int\log L_n(p)\,d\Pi_{\mathcal B_n}(p)\right]| > n\epsilon_n^2
\right\}\\
&\le\frac{\operatorname{Var}_{P_{Z,\bm X}^{\otimes n}}\!\left[\int\log L_n(p)\,d\Pi_{\mathcal B_n}(p)\right] }{(n\epsilon_n^2)^2}
\\&\le\frac{C_1^2}{n\epsilon_n^2}.
 \end{align*}
 
\color{red}
Combining the above bounds, with probability at least
\[
1-\frac{C_1^2}{n\epsilon_n^2},
\]
the denominator is bounded below by
\[
e^{-(C_2+C_1^2+1)n\epsilon_n^2}.
\]
\color{black}

\textbf{(ii) The numerator $\int_{h(p,p^*)>K\epsilon_n} L_n(p)\,d\Pi(p)$ is large only with low probability}

We use deterministic density sieves \(\mathcal G_n\) satisfying G2 (Equation \ref{app:eq:sieve-tail}) to control $\Pi(\mathcal G_n^c)
\le C_3e^{-(C_2+C_1^2+4)n\epsilon_n^2}$. The main difficulty is to control the contribution from $\mathcal A_n :=\{h(p,p^*)>K\epsilon_n\}\cap\mathcal G_n$. More precisely,
\begin{align*}
    \int_{h(p,p^*)>K\epsilon_n} L_n(p)\,d\Pi(p)\leq\int_{\mathcal A_n} L_n(p)\,d\Pi(p)+\int_{\mathcal G_n^c} L_n(p)\,d\Pi(p)
\end{align*}

\textbf{For the second term,} the argument is simpler. Since we will ultimately apply Markov inequality to bound the probability, we first bound $\mathbb E_{P_{Z,\bm X}^{\otimes n}}\!\left[\int_{\mathcal G_n^c} L_n(p)\,d\Pi(p)\right]$. For later use, we establish the following more general bound. Let \(\mathcal S_n :=(\mathcal Z\times\mathcal X)^n\) denote the sample space for \(\mathcal D_n\).
\[
S_n^*:=\left\{\mathcal D_n\in\mathcal S_n:\prod_{i=1}^n p^*(z_i\mid\bm x_i)>0\right\}.
\] 
For any nonnegative measurable function \(F\) of \(\mathcal D_n\),
\begin{align}
  \mathbb E_{P_{Z,\bm X}^{\otimes n}}\!\left[\int_{\mathcal G_n^c}F  (\mathcal D_n)L_n(p)\,d\Pi(p)\right]&=\int_{\mathcal G_n^c}\mathbb E_{P_{Z,\bm X}^{\otimes n}}\!\left[F (\mathcal D_n) L_n(p)\right]\,d\Pi(p)\\
  &=\int_{\mathcal G_n^c}\left[\int_{S_n^*} F(\mathcal D_n)\prod_{i=1}^n\frac{p(z_i\mid\bm x_i)}{p^*(z_i\mid\bm x_i)}\,dP_{Z,\bm X}^{\otimes n}(\mathcal D_n)\right]d\Pi(p)\\
   &=\int_{\mathcal G_n^c}\left[
\int_{S_n^*} F(\mathcal D_n)\prod_{i=1}^n p(z_i\mid\bm x_i)
\,d(\mu\otimes P_{\bm X})^{\otimes n}(\mathcal D_n)
\right]d\Pi(p)\\
 &\le\int_{\mathcal G_n^c}\left[
\int_{\mathcal S_n} F(\mathcal D_n)\prod_{i=1}^n p(z_i\mid\bm x_i)
\,d(\mu\otimes P_{\bm X})^{\otimes n}(\mathcal D_n)
\right]d\Pi(p)\\
&=\int_{\mathcal G_n^c}\mathbb E_{P_p^{\otimes n}}[F(\mathcal D_n)]\,d\Pi(p),
\label{app:eq:integrated-change-of-measure}
\end{align}
where $P_p^{\otimes n}$ is the joint law of $n$ independent observations from $P_p$, defined by
\[
dP_p(z,\bm x):=p(z\mid\bm x)\,d\mu(z)\,dP_{\bm X}(\bm x),
\]
with $\mu$ the dominating measure for $Z$ and $P_{\bm X}$ the fixed true marginal distribution of $\bm X$.

Taking \(F\equiv1\), we obtain
\begin{align*}
  \mathbb E_{P_{Z,\bm X}^{\otimes n}}\!\left[\int_{\mathcal G_n^c}L_n(p)\,d\Pi(p)\right]\leq\int_{\mathcal G_n^c}\mathbb E_{P_p^{\otimes n}}[1]\,d\Pi(p)=\Pi(\mathcal{G}_n^c)
\end{align*}
\color{red}
Thus, by Markov's inequality, we have
\begin{align*}
    P_{Z,\bm X}^{\otimes n}\!\left[\int_{\mathcal G_n^c}L_n(p)\,d\Pi(p)>e^{-(C_2+C_1^2+2)n\epsilon_n^2}\right]\leq \frac{C_3e^{-(C_2+C_1^2+4)n\epsilon_n^2}}{e^{-(C_2+C_1^2+2)n\epsilon_n^2}}=C_3e^{-2n\epsilon_n^2}\le e^{-n\epsilon_n^2}
\end{align*}
\color{black}
The last inequality holds for all sufficiently large \(n\), since \(n\epsilon_n^2\to\infty\).

\textbf{For the first term,} we use the standard Hellinger testing construction of \citet{ghosal2000convergence}. We construct a test that distinguishes the true distribution from all candidate distributions indexed by $\mathcal A_n$. More specifically, we test
\[
H_0:\mathcal D_n\sim P_{Z,\bm X}^{\otimes n}
\qquad\text{against}\qquad
H_1:\mathcal D_n\sim P_p^{\otimes n}
\text{ for some }p\in\mathcal A_n.
\] 
We seek a test that equals \(1\) when \(H_0\) is rejected and \(0\) otherwise.

By Equation~\eqref{app:eq:sieve-entropy}, there exists a cover of $\mathcal G_n$ by $N_n$ Hellinger balls of radius $C_5\epsilon_n$, with centers $q_1,\ldots,q_{N_n}$, such that
\[
N_n\le e^{C_4n\epsilon_n^2}.
\]
Retain only the balls that intersect $\mathcal A_n$, and denote their index set by $J_n$. For each $j\in J_n$, choose $p_j\in\mathcal A_n$ in the corresponding ball. The triangle inequality gives
\[
h(q_j,p^*)
\ge h(p_j,p^*)-h(p_j,q_j)
>(K-C_5)\epsilon_n.
\]
In particular, if $K>3C_5$, then
\[
C_5\epsilon_n<\frac12h(q_j,p^*).
\]

The robust Hellinger testing result in
\citet[Section~7, Equations~(7.1)--(7.2)]{ghosal2000convergence}
therefore provides, for each $j\in J_n$, a binary test $\phi_{n,j}$ and a universal constant $c>0$ such that
\begin{align*}
\mathbb E_{P_{Z,\bm X}^{\otimes n}}[\phi_{n,j}]
&\le e^{-cn h^2(q_j,p^*)}
\le e^{-c(K-C_5)^2n\epsilon_n^2},\\
\sup_{p:\,h(p,q_j)\le C_5\epsilon_n}
\mathbb E_{P_p^{\otimes n}}[1-\phi_{n,j}]
&\le e^{-cn h^2(q_j,p^*)}
\le e^{-c(K-C_5)^2n\epsilon_n^2}.
\end{align*}

Define $\psi_n:=\max_{j\in J_n}\phi_{n,j}$, with $\psi_n:=0$ if $J_n$ is empty. By the union bound,
\begin{align*}
\mathbb E_{P_{Z,\bm X}^{\otimes n}}[\psi_n]
&\le\sum_{j\in J_n}\mathbb E_{P_{Z,\bm X}^{\otimes n}}[\phi_{n,j}]\\
&\le N_n e^{-c(K-C_5)^2n\epsilon_n^2}\\
&\le e^{-\{c(K-C_5)^2-C_4\}n\epsilon_n^2}.
\end{align*}
Moreover, every $p\in\mathcal A_n$ belongs to at least one retained ball, say the ball indexed by $j$. Since $\psi_n\ge\phi_{n,j}$,
\[
\mathbb E_{P_p^{\otimes n}}[1-\psi_n]
\le\mathbb E_{P_p^{\otimes n}}[1-\phi_{n,j}]
\le e^{-c(K-C_5)^2n\epsilon_n^2}.
\]
Choose a fixed $K>3C_5$ sufficiently large that
\[
c(K-C_5)^2
\ge\max\{C_4+1,\ C_2+C_1^2+4\}.
\]
Then
\begin{equation}
\mathbb E_{P_{Z,\bm X}^{\otimes n}}[\psi_n]\le e^{-n\epsilon_n^2},
\qquad
\sup_{p\in\mathcal A_n}
\mathbb E_{P_p^{\otimes n}}[1-\psi_n]
\le e^{-(C_2+C_1^2+4)n\epsilon_n^2}.
\label{app:eq:tests}
\end{equation}

We now apply the same change-of-measure argument in Equation~\eqref{app:eq:integrated-change-of-measure}, replacing the integration set $\mathcal G_n^c$ by $\mathcal A_n$ and taking $F(\mathcal D_n)=1-\psi_n(\mathcal D_n)\ge0$. Together with Equation~\eqref{app:eq:tests}, this yields
\begin{align*}
\mathbb E_{P_{Z,\bm X}^{\otimes n}}\!\left[
(1-\psi_n)\int_{\mathcal A_n}L_n(p)\,d\Pi(p)
\right]
&\le\int_{\mathcal A_n}
\mathbb E_{P_p^{\otimes n}}[1-\psi_n]\,d\Pi(p)\\
&\le e^{-(C_2+C_1^2+4)n\epsilon_n^2}\Pi(\mathcal A_n)\\
&\le e^{-(C_2+C_1^2+4)n\epsilon_n^2}.
\end{align*}
Consequently, Markov's inequality gives
\begin{equation}
P_{Z,\bm X}^{\otimes n}\!\left\{
(1-\psi_n)\int_{\mathcal A_n}L_n(p)\,d\Pi(p)
>e^{-(C_2+C_1^2+2)n\epsilon_n^2}
\right\}
\le e^{-2n\epsilon_n^2}.
\label{app:eq:first-term-bound}
\end{equation}
Since $\psi_n$ is binary, Equation~\eqref{app:eq:tests} also gives
$P_{Z,\bm X}^{\otimes n}(\psi_n=1)\le e^{-n\epsilon_n^2}$.
\color{red}
Thus, with probability at least
$1-e^{-n\epsilon_n^2}-e^{-2n\epsilon_n^2}$, we have $\psi_n=0$ and
\[
\int_{\mathcal A_n}L_n(p)\,d\Pi(p)
\le e^{-(C_2+C_1^2+2)n\epsilon_n^2}.
\]
\color{black}

\textbf{(iii) Combining the preceding bounds.}
Recall that
\[
\int_{\{p:h(p,p^*)>K\epsilon_n\}}L_n(p)\,d\Pi(p)
\le \int_{\mathcal A_n}L_n(p)\,d\Pi(p)
+\int_{\mathcal G_n^c}L_n(p)\,d\Pi(p).
\]
By the union bound, the denominator lower bound and both numerator upper bounds hold simultaneously with probability at least 
\[
1-\left\{
\frac{C_1^2}{n\epsilon_n^2}
+2e^{-n\epsilon_n^2}
+e^{-2n\epsilon_n^2}
\right\}
\]
under $P_{Z,\bm X}^{\otimes n}$, for all sufficiently large $n$. On this event,
\begin{align*}
\Pi\{h(p,p^*)>K\epsilon_n\mid\mathcal D_n\}
&=\frac{\int_{\{p:h(p,p^*)>K\epsilon_n\}}L_n(p)\,d\Pi(p)}
{\int L_n(p)\,d\Pi(p)}\\
&\le\frac{\int_{\mathcal A_n}L_n(p)\,d\Pi(p)
+\int_{\mathcal G_n^c}L_n(p)\,d\Pi(p)}
{\int L_n(p)\,d\Pi(p)}\\
&\le\frac{2e^{-(C_2+C_1^2+2)n\epsilon_n^2}}
{e^{-(C_2+C_1^2+1)n\epsilon_n^2}}\\
&=2e^{-n\epsilon_n^2}.
\end{align*}
Since $e^{-2n\epsilon_n^2}\le e^{-n\epsilon_n^2}$, we conclude that
\[
P_{Z,\bm X}^{\otimes n}\!\left\{
\Pi\{h(p,p^*)>K\epsilon_n\mid\mathcal D_n\}
>2e^{-n\epsilon_n^2}
\right\}
\le \frac{C_1^2}{n\epsilon_n^2}+3e^{-n\epsilon_n^2}.
\]

\paragraph{Proof of Equation \ref{app:eq:density-second-moment}}

Since \(h^2\le2\), we have the deterministic posterior bound
\[
\mathbb E_\Pi[h^2(p,p^*)\mid\mathcal D_n]
\le K^2\epsilon_n^2
+2\Pi\{h(p,p^*)>K\epsilon_n\mid\mathcal D_n\}.
\]
On the high-probability event established above,
\[
\mathbb E_\Pi[h^2(p,p^*)\mid\mathcal D_n]
\le K^2\epsilon_n^2+4e^{-n\epsilon_n^2}.
\]
The additional condition $\log(1/\epsilon_n)=o(n\epsilon_n^2)$ implies
\[
\frac{2\log(1/\epsilon_n)}{n\epsilon_n^2}\longrightarrow0.
\]
Thus, 
\[
-n\epsilon_n^2+2\log(1/\epsilon_n)
=-n\epsilon_n^2
\left(1-\frac{2\log(1/\epsilon_n)}{n\epsilon_n^2}\right)
\longrightarrow-\infty,
\]
Therefore, we get
\[
\frac{e^{-n\epsilon_n^2}}{\epsilon_n^2}
=\exp\{-n\epsilon_n^2+2\log(1/\epsilon_n)\}\longrightarrow0.
\]
Thus, for all sufficiently large $n$, we have
\[
4e^{-n\epsilon_n^2}\le\epsilon_n^2.
\]
Hence, we can build
\[
P_{Z,\bm X}^{\otimes n}\!\left\{
\mathbb E_\Pi[h^2(p,p^*)\mid\mathcal D_n]
>(K^2+1)\epsilon_n^2
\right\}
\le \frac{C_1^2}{n\epsilon_n^2}+3e^{-n\epsilon_n^2}
\longrightarrow0.
\]
Consequently,
\[
\mathbb E_\Pi[h^2(p,p^*)\mid\mathcal D_n]
=O_P(\epsilon_n^2),
\]
where $O_P$ is with respect to the true sampling law
$P_{Z,\bm X}^{\otimes n}$.

\paragraph{Proof of Equation \ref{app:eq:importance-second-moment}}

\[
\mathbb E_\Pi[\|\mathrm{LUCID}[p]-\mathrm{LUCID}[p^*]\|_2^2\mid\mathcal D_n]
\leq4C_\ell^2\mathbb E_\Pi[h^2(p,p^*)\mid\mathcal D_n]
=O_P(\epsilon_n^2).
\]

\end{proof}

\textbf{Note:} mere convergence of a posterior tail probability to zero, without a quantitative bound such as Equation~\eqref{app:eq:posterior-tail}, is insufficient to infer the same rate for a posterior mean.

\subsection{DR-BART Specialization}
\label{app:drbart}

We apply Conditions F and P1--P6 of \citet[Section~4]{orlandi2021density} to response $Z$ and conditioning covariates $\bm{X}$. These conditions are adopted in full. They impose smoothness and sparsity of the log conditional density, restrictions on the true conditional quantile function, and specified support-selection, tree, leaf, scale, and split-point priors, including an approximation condition. They are stronger than choosing the default tuning parameters of a density estimator.

\begin{lucidappcor}[DR-BART contraction and posterior moments]
\label{app:cor:drbart}
Under these conditions, with $k$ conditioning coordinates, let
\begin{equation}
\epsilon_n=
\left(\frac{n}{\log n}\right)^{-\frac{\alpha}{\alpha+1}\frac{\beta}{2\beta+d_0}}
+\sqrt{\frac{d_0\log(k+1)}{n}},
\qquad \beta=\min\{\alpha,1\}.
\label{app:eq:drbart-rate}
\end{equation}
Here $0<\alpha\le2$ is the smoothness index and $d_0$ counts active coordinates including the response. The density posterior contracts at this rate in integrated Hellinger distance. Its proof conditions also yield Equation~\eqref{app:eq:density-second-moment}. If Equation~\eqref{app:eq:bounded-loss} holds, the \LCVI\ posterior satisfies both Equations~\eqref{app:eq:posterior-transfer} and \eqref{app:eq:importance-second-moment}.
\end{lucidappcor}

\begin{proof}
The density contraction statement follows from
\citet[Theorem~2]{orlandi2021density}. We verify that the bounds in its proof satisfy the assumptions of Lemma~\ref{app:lem:posterior-moments}.

First, condition G1, verified in Section~8.3 of the cited paper, gives constants $C_1,C_2>0$ such that
\[
\Pi\{p:\mathrm{KL}(p^*\|p)\le C_1^2\epsilon_n^2,\ 
V(p^*,p)\le C_1^2\epsilon_n^2\}
\ge e^{-C_2n\epsilon_n^2}.
\]
This is Equation~\eqref{app:eq:prior-thickness}.

Second, the proof of G2 constructs sieves $\mathcal G_{n,\kappa}$ indexed by a sufficiently large fixed constant $\kappa$ and establishes
\[
\Pi(\mathcal G_{n,\kappa}^c)
\le \exp(-C_\mathcal{G}\kappa n\epsilon_n^2),
\]
where $C_\mathcal{G}>0$ depends only on the prior. Choose
\[
\kappa>\frac{C_2+C_1^2+4}{C_\mathcal{G}}
\]
and fix $\mathcal G_n=\mathcal G_{n,\kappa}$. Then
\[
\Pi(\mathcal G_n^c)
\le e^{-(C_2+C_1^2+4)n\epsilon_n^2},
\]
which verifies Equation~\eqref{app:eq:sieve-tail} with $C_3=1$. This choice changes only the sieve used in the proof, not the prior.

Third, the covering construction establishing G3 remains valid for this fixed $\kappa$. It gives finite constants $C_4,C_5>0$, which may depend on $\kappa$, such that
\[
\log\mathcal N(C_5\epsilon_n,\mathcal G_n,h)
\le C_4n\epsilon_n^2.
\]
This is Equation~\eqref{app:eq:sieve-entropy}. 

It remains to check the growth conditions. Assume $\epsilon_n\to0$, as stipulated in the corollary, and write
\[
a=\frac{\alpha}{\alpha+1}\frac{\beta}{2\beta+d_0},
\qquad \beta=\min\{\alpha,1\}.
\]
Since $0<\alpha\le2$ and $d_0\ge1$,
\[
0<a\le\frac23\cdot\frac13=\frac29.
\]
For all sufficiently large $n$,
\[
\epsilon_n\ge(n/\log n)^{-a}
\ge(n/\log n)^{-2/9},
\]
and hence
\[
n\epsilon_n^2
\ge n^{5/9}(\log n)^{4/9}\longrightarrow\infty.
\]
Moreover, once $\epsilon_n<1$,
\[
0\le\log(1/\epsilon_n)
\le\frac29\log(n/\log n)
\le\frac29\log n
=o(n\epsilon_n^2).
\]
Lemma~\ref{app:lem:posterior-moments} therefore yields
Equation~\eqref{app:eq:density-second-moment}. Under
Equation~\eqref{app:eq:bounded-loss}, it also gives
Equation~\eqref{app:eq:importance-second-moment}.
Finally, the density contraction result, together with
Corollary~\ref{app:cor:posterior-transfer} and
Equation~\eqref{app:eq:hellinger-stability}, establishes the stated
contraction of the \LCVI\ posterior.
\end{proof}

\subsection{Convergence Rate of Posterior Averaging}
\label{app:posterior-averaging}

The linear dependence of \LCVI\ on the replacement distribution allows posterior averaging without explicitly constructing an average density. For analysis, define
\begin{equation}
\overline p(z'\mid \bm{x}):=\mathbb E_\Pi[p(z'\mid \bm{x})\mid\mathcal D_n].
\label{app:eq:mixture-densities}
\end{equation}
This is a conceptual replacement distribution used for the analysis; it need not be fitted or evaluated explicitly.

\begin{lucidappprop}[Posterior averaging and its rate]
\label{app:prop:posterior-averaging}
Let $\overline p(z\mid\bm{x}):=\mathbb E_\Pi[p(z\mid\bm{x})\mid\mathcal D_n]$ denote the posterior mean conditional density.

If Equation~\eqref{app:eq:density-second-moment} holds, then
\begin{equation}
h(\overline p,p^*)=O_P(\epsilon_n).
\label{app:eq:posterior-mean-density-rate}
\end{equation}

Under Equation~\eqref{app:eq:bounded-loss},
\begin{equation}
\LUCIDapp[\overline p](\bm{w})
=\mathbb E_\Pi[\LUCIDapp[p](\bm{w})\mid\mathcal D_n].
\label{app:eq:posterior-mean}
\end{equation}
If Equation~\eqref{app:eq:importance-second-moment} also holds, then
\begin{equation}
\|\LUCIDapp[\overline p]-\LUCIDapp[p^*]\|_2=O_P(\epsilon_n).
\label{app:eq:posterior-mean-rate}
\end{equation}
\end{lucidappprop}

\begin{proof}
Tonelli's theorem ensures that $\overline p(\cdot\mid\bm{x})$ integrates to one. Convexity of squared Hellinger distance in its first argument gives
\[
h^2(\overline p,p^*)
\le\mathbb E_\Pi[h^2(p,p^*)\mid\mathcal D_n]
=O_P(\epsilon_n^2),
\]
where the last equality follows from Equation~\eqref{app:eq:density-second-moment}. Taking square roots proves Equation~\eqref{app:eq:posterior-mean-density-rate}.

Under Equation~\eqref{app:eq:bounded-loss}, Fubini's theorem and Equation~\eqref{app:eq:functional} yield
\begin{align*}
\LUCIDapp[\overline p](\bm{w})
&=\int \ell(g(z',\bm{x}),g(z,\bm{x}))\,\overline p(z'\mid\bm{x})\,dz'\\
&=\mathbb E_\Pi\left[\int \ell(g(z',\bm{x}),g(z,\bm{x}))\,p(z'\mid\bm{x})\,dz'\,\middle|\,\mathcal D_n\right]\\
&=\mathbb E_\Pi[\LUCIDapp[p](\bm{w})\mid\mathcal D_n],
\end{align*}
proving Equation~\eqref{app:eq:posterior-mean}. Jensen's inequality in $L^2(P_{Z,\bm{X}})$ then gives
\[
\|\LUCIDapp[\overline p]-\LUCIDapp[p^*]\|_2^2
\le\mathbb E_\Pi[\|\LUCIDapp[p]-\LUCIDapp[p^*]\|_2^2\mid\mathcal D_n]
=O_P(\epsilon_n^2),
\]
where the last equality follows from Equation~\eqref{app:eq:importance-second-moment}. Taking square roots proves Equation~\eqref{app:eq:posterior-mean-rate}.
\end{proof}

These results establish the convergence rate of the exact posterior mean. Approximating it with finitely many retained posterior draws and conditional replacements introduces additional Monte Carlo error, which must be controlled separately.

\newpage
\section{Implementation Details and Experimental Setup}
\label{app:implementation}

\subsection{Implementation Details}
\label{app:implementation_details}
This section describes conditional replacement sampling, posterior summaries, and the experimental settings used to evaluate \LCVI.

\paragraph{Pseudocode.}
We first present the pseudocode for \LCVI estimation without batching in \cref{alg:lcvi}.

\begin{algorithm}[h]
\caption{Estimating \LCVI by conditional sampling}
\label{alg:lcvi}
\begin{algorithmic}[1]
\Require Fixed model $g$; evaluation observations
$\{\bm w_i=(z_i,\bm x_i)\}_{i=1}^{N}$;
conditional sampler $\widehat P_{Z\mid \bm X=\bm x_{i}}(\cdot\mid \bm x_i)$;
loss $\ell$; number of replacements $M$

\For{$i=1,\ldots,N$}
    \State Draw $z_{i,1}',\ldots,z_{i,M}'$ independently from
    $\widehat P_{Z\mid \bm X=\bm x_{i}}(\cdot\mid \bm x_i)$
    \State Form perturbed observations
    $\bm w_{i,m}'=(z_{i,m}',\bm x_i)$, $m=1,\ldots,M$
    \State Compute $\widehat{\mathrm{\LCVI}}_i$ using
    \cref{eq:lcvi_estimator}
\EndFor

\State \Return $\{\widehat{\mathrm{\LCVI}}_i\}_{i=1}^{N}$
\end{algorithmic}
\end{algorithm}

\paragraph{Direct conditional sampling.}
For a continuous feature, a posterior draw $b$ of DR-BART has the
latent-variable representation
\begin{equation}
Z'
=f^{(b)}(\bm x,U)
+\sigma^{(b)}(\bm x,U)\varepsilon,
\qquad
U\sim\operatorname{Unif}(0,1),\quad
\varepsilon\sim\mathcal{N}(0,1),
\label{eq:drbart_latent_sampler}
\end{equation}
where $U$ and $\varepsilon$ are independent, and $f^{(b)}$ and
$\sigma^{(b)}$ are the location and standard-deviation functions
from that posterior draw. Here $U$ is a continuous latent variable;
$b$ indexes a draw from the posterior over the model parameters.

We use this representation to generate conditional replacements directly, without first evaluating the fitted conditional density on a grid. Conditional on a posterior draw, this procedure samples exactly from the corresponding fitted conditional distribution, avoiding grid discretization and numerical inversion of the cumulative distribution function. It also eliminates the need to construct and store a density-evaluation matrix for sampling. More specifically, first draw $U$, evaluate the location and standard deviation at $(\bm x,U)$, and then draw $Z'$ from the resulting normal distribution. \Cref{alg:drbart_conditional_sampling} gives the sampling procedure.

\begin{algorithm}[h]
\caption{Direct posterior-predictive sampling with DR-BART}
\label{alg:drbart_conditional_sampling}
\begin{algorithmic}[1]
\Require Conditioning vector $\bm x$; retained posterior functions
$\{(f^{(b)},\sigma^{(b)})\}_{b=1}^{B}$; number of replacements $M$
\For{$m=1,\ldots,M$}
    \State Draw $b_m\sim\operatorname{Unif}\{1,\ldots,B\}$
    \State Draw $u_m\sim\operatorname{Unif}(0,1)$
    \State Set $\mu_m\gets f^{(b_m)}(\bm x,u_m)$ and
    $s_m\gets\sigma^{(b_m)}(\bm x,u_m)$
    \State Draw $Z_m'\sim \mathcal{N}(\mu_m,s_m^2)$
\EndFor
\State \Return $Z_1',\ldots,Z_M'$
\end{algorithmic}
\end{algorithm}

\paragraph{Binary and categorical features.}
For binary and categorical features, the conditional distribution is fully determined by the conditional category probabilities. We therefore
use classification BART \citep{chipman2010bart,quiroga2022bart} to estimate these probabilities directly. For a binary feature, we estimate $P(Z=1\mid \bm X=\bm x)$ and draw replacements from the corresponding Bernoulli distribution. For a feature with multiple categories, we estimate a vector of conditional category probabilities that sum to one and draw replacements from the resulting categorical distribution. Each posterior draw of these probabilities specifies a complete conditional distribution, which can be used in the same \LCVI estimation and posterior summarization procedures described below. 

\paragraph{Posterior point estimates and credible intervals.}
For each retained posterior draw $b=1,\ldots,B$, generate $M$ replacements using \cref{alg:drbart_conditional_sampling}, however, keeping $b$ fixed across all replacements, and compute
\begin{equation}
\widehat{\mathrm{\LCVI}}^{(b)}(\bm w)
:=
\frac{1}{M}\sum_{m=1}^{M}
\ell\!\left(
g(Z_m^{\prime(b)},\bm x),
g(z,\bm x)
\right).
\label{eq:lcvi_drbart_posterior_draw_app}
\end{equation}
The posterior mean is approximated by
\begin{equation}
\widehat{\mathrm{\LCVI}}^{\mathrm{post}}(\bm w)
:=
\frac{1}{B}\sum_{b=1}^{B}
\widehat{\mathrm{\LCVI}}^{(b)}(\bm w),
\label{eq:lcvi_drbart_posterior_mean}
\end{equation}

When only a posterior point estimate is required, we can alternatively sample from the posterior predictive distribution. For each replacement, we select an index $b$ uniformly from $\{1,\ldots,B\}$ and generate a conditional replacement using \cref{eq:drbart_latent_sampler}. Because \LCVI is linear in the conditional replacement distribution, averaging the resulting loss contrasts estimates the same posterior mean, conditional on the retained draws, without requiring a separate \LCVI estimate for every draw.

To construct posterior credible intervals, we instead compute a separate \LCVI estimate for each posterior draw using \cref{eq:lcvi_drbart_posterior_draw_app}. The empirical quantiles of these estimates $\{\widehat{\mathrm{\LCVI}}^{(b)}(\bm w)\}_{b=1}^{B}$ describe posterior uncertainty in \LCVI induced by the conditional density estimator. Writing $\widehat Q_{\bm w}(\gamma)$ for the empirical $\gamma$-quantile of
$\{\widehat{\mathrm{\LCVI}}^{(b)}(\bm w)\}_{b=1}^{B}$, an approximate
$(1-\gamma)$ posterior credible interval is
\begin{equation}
\left[
\widehat Q_{\bm w}(\gamma/2),
\widehat Q_{\bm w}(1-\gamma/2)
\right].
\label{eq:lcvi_credible_interval}
\end{equation}
These empirical quantiles approximate posterior uncertainty in LUCID when the Monte Carlo error within each posterior draw is negligible; otherwise, they also reflect finite-\(M\) sampling variability.

\paragraph{Sampling feature groups.}
Let $\bm Z=(Z_1,\ldots,Z_d)$ denote a feature group with a fixed ordering, and let $\bm X$ denote the remaining covariates. We construct a joint conditional sampler through the factorization
\begin{equation}
p(\bm z\mid\bm x)
=
\prod_{r=1}^{d}
p\!\left(
z_r\mid \bm x,z_1,\ldots,z_{r-1}
\right).
\label{eq:group_conditional_factorization}
\end{equation}
We fit each factor using the remaining covariates and the preceding group members as predictors. For continuous group members, we fit these factors with DR-BART.

To generate one replacement vector, first sample $Z_1'$ given $\bm x$, then sample $Z_2'$ given $(\bm x,Z_1')$, and continue through the group. Each step conditions on the values already sampled, allowing the sampler to represent dependence within the group. Substitute the resulting joint draw $\bm Z'$ into \cref{eq:group_lcvi}, and average the resulting losses as in \cref{eq:lcvi_estimator}.

To quantify posterior uncertainty in group \LCVI, we additionally draw the factor models from their joint posterior. With independent priors for the factor models, a posterior draw of the joint sampler is formed by selecting a posterior draw from each factor independently. We then hold these selected functions fixed across the $M$ joint replacements used to compute one group \LCVI draw. Repeating this procedure yields posterior point estimates and credible intervals for $\mathrm{\LCVI}(\bm w)$ for the feature group, in the same way as for individual features.

\paragraph{Feature bounds.}
The experimental implementation optionally restricts continuous features to bounds $[L_j,H_j]$. These bounds may be specified by the user; when they are not provided, we set them to the observed dataset range padded by 10\% on each side. Candidates outside this interval are rejected while holding the posterior index fixed and redrawing only $U_{im}$ and $\varepsilon_{im}$. Thus, each posterior-specific replacement follows the corresponding conditional distribution truncated to $[L_j,H_j]$.

After at most $10{,}000$ rejection rounds, any remaining replacements are generated using a bounded-mixture fallback based on latent-variable tree cutpoints, which samples from the fitted conditional distribution further restricted to $X_j\in[L_j,H_j]$.

\subsection{Additional Implementation Details for Direct Conditional Sampling and Batched \LCVI Computation}
\label{sec:drbart_sampling_batching}

This section provides implementation-level details for our \LCVI computation, including direct conditional sampling, batching, and the corresponding time and space complexity analysis. 

\paragraph{Notation and computational scope.}
Since the analysis is carried out feature by feature and we aim to make the full update procedure explicit, we introduce a separate notation system with additional subscripts to indicate precisely which observation, feature, posterior draw, and replacement each quantity corresponds to.

Consider a fixed feature $j$ and $N$ evaluation observations $\{\bm x_i\}_{i=1}^{N}$, with $\bm x_i\in\mathbb R^{p}$. Here, \(\bm x_i\) denotes the full observation previously written as \(\bm w_i\). The predictive model $g$ is held fixed throughout the calculation. This notation system will continue to be used in \cref{app:dgp} and \cref{app:openxai_scaling}.

Let $\{(f_j^{(b)},\sigma_j^{(b)})\}_{b=1}^{B}$ denote the retained conditional-density posterior functions. We use $M$ replacements per observation, an evaluation batch size $A_{\text{batch}}$, and a replacement-prediction batch size $C_{\text{batch}}$. 

The following complexity statements concern evaluation after fitting. They exclude the input dataset and the storage of the fitted predictive model and posterior tree ensembles. \textbf{In particular, direct sampling does not eliminate the storage required for the $B$ retained posterior functions themselves.} Let $\tau_f$ denote the cost of evaluating both $f_j^{(b)}$ and $\sigma_j^{(b)}$ at one conditioning--latent pair. 

\paragraph{Batched implementation.}
In our implementation, feature-specific calculations are submitted as independent parallel jobs. Within each job, conditional replacements are generated in evaluation batches of size $A_{\text{batch}}$. The implementation then processes each observation separately, evaluating its perturbed inputs in replacement batches of size $C_{\text{batch}}$ and accumulating their losses. Thus, sampling is batched across observations, whereas predictive-model evaluation is batched across replacements of a single observation.

An independent job executes \cref{alg:lcvi_batch} for each feature. These jobs can run concurrently subject to available resources; each job writes its feature-specific result separately.

\begin{algorithm}[h]
\caption{Batched prediction-reference \LCVI estimation per feature}
\label{alg:lcvi_batch}
\begin{algorithmic}[1]
\Require Fixed model $g$; evaluation observations
$\{\bm x_i\}_{i=1}^{N}\subset\mathbb R^p$;
feature $j$; fitted direct conditional sampler for feature $j$;
number of replacements $M$;
evaluation batch size $A_{\text{batch}}$;
replacement-prediction batch size $C_{\text{batch}}$

\State Evaluate the fixed model on all original observations:
$g(\bm x_i)$, $i=1,\ldots,N$

\For{each consecutive observation batch
$\mathcal I\subseteq\{1,\ldots,N\}$ of size at most $A_{\text{batch}}$}
    \State Use \textbf{direct posterior-predictive} (in \cref{alg:drbart_conditional_sampling}) sampling to generate
    $\bm R\in\mathbb R^{|\mathcal I|\times M}$, where
    $R_{im}$ replaces feature $j$ conditional on $\bm x_{i,-j}$

    \For{each observation $i\in\mathcal I$}
        \State Set accumulated loss $T_i\gets0$

        \For{each consecutive replacement batch
        $\mathcal C\subseteq\{1,\ldots,M\}$ of size at most $C_{\text{batch}}$}
            \State Form $\bm x_{im}'$ by replacing coordinate $j$
            of $\bm x_i$ with $R_{im}$, for $m\in\mathcal C$
            \State Evaluate
            $\{g(\bm x_{im}')\}_{m\in\mathcal C}$
            in one model call
            \State Accumulate
            $T_i\gets T_i+
            \sum_{m\in\mathcal C}
            \ell(g(\bm x_{im}'),g(\bm x_i))$
        \EndFor

        \State Compute
        $\widehat \LCVI(\bm x_i)\gets T_i/M$
    \EndFor
\EndFor

\State Save and return $\widehat{\LCVI}$ for feature $j$
\end{algorithmic}
\end{algorithm}

\paragraph{Direct sampling and posterior conditioning.}
For a continuous feature, the unbounded conditional sampler uses
\begin{equation}
\begin{aligned}
U_{im} &\sim \operatorname{Unif}(0,1),\\
\varepsilon_{im} &\sim \mathcal N(0,1),\\
X_{i,j,m}' &=
f_j^{(b_{im})}(\bm x_{i,-j},U_{im})
+
\sigma_j^{(b_{im})}(\bm x_{i,-j},U_{im})
\varepsilon_{im}.
\end{aligned}
\label{eq:drbart_batch_latent_sampler}
\end{equation}
For a posterior-predictive point estimate, the implementation draws
$b_{im}\sim\operatorname{Unif}\{1,\ldots,B\}$
independently for every observation--replacement pair.
For a posterior-specific \LCVI estimate, it instead fixes
$b_{im}=b$ for all observations and all $M$ replacements.

The conditional standard deviation follows the fitted variance mode:
\begin{equation}
\sigma_j^{(b)}(\bm x_{-j},u)=
\begin{cases}
\left[
\phi_{j,b}^{*}\tau_j^{(b)}(\bm x_{-j},u)
\right]^{-1/2},
& \texttt{ux},\\[2pt]
\left[
\phi_{j,b}^{*}\tau_j^{(b)}(\bm x_{-j})
\right]^{-1/2},
& \texttt{x},\\[2pt]
\sigma_{j,b},
& \texttt{const}.
\end{cases}
\label{eq:drbart_batch_variance_modes}
\end{equation}
Here, $\tau_j^{(b)}$ is the fitted precision-tree function, and $\phi_{j,b}^{*}>0$ is the global precision scale retained in posterior draw $b$ for the conditional model of feature $j$. Their product gives the conditional precision. For the constant-variance model, $\sigma_{j,b}$ is the retained
posterior standard deviation. This is the parameterization used by the default DR-BART implementation; we make the conversion from the retained precision-tree representation to the corresponding conditional variance explicit only because the latter is not directly exposed by the default API.

\paragraph{Complexity of direct sampling.}
The following bounds concern unbounded fixed-posterior sampling and exclude fitted-model storage, model loading, and subsequent predictive-model evaluation.

\begin{proposition}[Complexity of direct sampling per feature]
\label{prop:drbart_direct_fixed_complexity}
For one feature and one fixed posterior draw, generating $M$
replacements for each of $N$ evaluation observations takes
\[
T_{\mathrm{sample}}
=
O\!\left(NM(p+\tau_f)\right)
\]
time and
\[
\mathcal S_{\mathrm{sample}}
=
O(A_{\text{batch}}Mp)
\]
peak sampling storage.
\end{proposition}

\begin{proof}
For a batch of $a\leq A_{\text{batch}}$ observations (the final batch may contain fewer than \(A_{\text{batch}}\) observations), the implementation materializes an $aM\times(p-1)$ conditioning array and evaluates the conditional mean and standard deviation once per replacement. Thus, processing all $NM$ replacements takes $O(NM(p+\tau_f))$ time, while the current batch requires $O(A_{\text{batch}}Mp)$ storage.
\end{proof}

Materializing each batch enables vectorized tree evaluation and reduces per-call overhead. Using $A_{\text{batch}}\ll N$ reduces peak sampling storage from $O(NMp)$ to $O(A_{\text{batch}}Mp)$. Subsequent predictive-model evaluation uses replacement batches of size $C_{\text{batch}}$, requiring an additional $O(C_{\text{batch}}p)$ perturbed-input workspace; it does not reduce the replacements already generated by the sampler. Feature-specific calculations can be distributed across independent parallel jobs, subject to available computational resources. The storage bound above is per job; concurrent jobs increase aggregate memory usage.

\textbf{Note:} Direct sampling does not require storing a conditional-density evaluation matrix. By generating one replacement per observation at a time and immediately accumulating its loss contribution, the sampling workspace can, in principle, be reduced to $O(A_{\text{batch}}p)$ while still using $M$ replacements per observation. The current implementation instead materializes all $M$ replacements within each evaluation batch, using $O(A_{\text{batch}}Mp)$ sampling workspace to enable vectorized tree evaluation and reduce per-call overhead. This creates a trade-off between memory usage and computational efficiency when processing very large datasets.

\subsection{Experimental Setup}

\paragraph{Real-world datasets.}
Our primary real-data case study uses the 2011 hourly subset of the UCI Bike Sharing dataset \citep{fanaee2014event}, containing 8,645 observations and distributed as \texttt{Bikeshare} in the \texttt{ISLR2} R package.\footnote{\url{https://search.r-project.org/CRAN/refmans/ISLR2/html/Bikeshare.html}}
The UCI dataset is released under the Creative Commons
Attribution 4.0 International (CC BY 4.0) license.\footnote{
\url{https://archive.ics.uci.edu/dataset/275/bike+sharing+dataset}}

To further evaluate scalability and robustness on higher-dimensional tabular data, we additionally use four datasets from the OpenXAI \citep{NEURIPS2022_65398a0e} benchmark: Framingham Heart Study (Heart), German Credit, HELOC, and Give Me Some Credit (GMSC). We obtain the benchmark versions of these datasets from the OpenXAI Harvard Dataverse archive, where the deposited dataset is listed under the CC0 1.0 Universal Public Domain Dedication.\footnote{\url{https://doi.org/10.7910/DVN/HIBCFL}}

\paragraph{Experiment Platform.}
Experiments were submitted as Slurm jobs on a shared high-performance computing cluster. The nodes used included Intel Xeon Gold 6226 and AMD EPYC 9554 processors, with node memory ranging from 768\,GB to approximately 1\,TB. Standard feature-specific worker tasks requested one CPU core and 32\,GB RAM, with no GPU resources requested. Independent feature and fold tasks were scheduled concurrently. The software environment used Python 3.14.4 and R 4.5.2.

\paragraph{Software Packages.}
We used the following upstream packages together with the feature-specific sampling and importance calculations implemented in our repository.

\textbf{DR-BART:}
Continuous conditional models were fitted using the R package
\texttt{drbart} (development version 0.0.0.9000), available at
\url{https://github.com/vittorioorlandi/drbart}.
Direct latent-variable sampling, batching, and posterior
\LCVI calculations were implemented in our repository.
Binary and categorical conditional models used
\texttt{pymc-bart} 0.13.1 and \texttt{pymc} 6.3.2, available at
\url{https://github.com/pymc-devs/pymc-bart} and
\url{https://github.com/pymc-devs/pymc}, respectively.

\textbf{ARF:}
We used the R package \texttt{arf} 0.2.5, available at
\url{https://github.com/bips-hb/arf}, for adversarial random-forest
fitting, density estimation, and conditional generation.
Its random-forest backend was \texttt{ranger} 0.18.0.

\textbf{cs-PFI:}
Our conditional-subgroup sampler used
\texttt{trtf} 0.4-3, \texttt{mlt} 1.8-2, and
\texttt{partykit} 1.3-0.
These packages are available at
\url{https://CRAN.R-project.org/package=trtf},
\url{https://CRAN.R-project.org/package=mlt}, and
\url{https://CRAN.R-project.org/package=partykit}.
The conditional-subgroup method's original code is available at
\url{https://github.com/christophM/paper_conditional_subgroups}.
Our repository implements the sampler and its integration with
the common \LCVI evaluation routine.

\textbf{LIME:}
We used \texttt{lime} 0.2.0.1 and its
\texttt{LimeTabularExplainer}, available at
\url{https://github.com/marcotcr/lime}.

\textbf{SHAP:}
We used \texttt{shap} 0.52.0 through
\texttt{shap.Explainer}, available at
\url{https://github.com/shap/shap}.

\textbf{Conditional SHAP:}
We used \texttt{pyshapr} 0.6.0 with the R package
\texttt{shapr} 1.1.0, available at
\url{https://github.com/NorskRegnesentral/shapr}.
Conditional expectations were estimated using its
\texttt{arf} approach. In some experimental figures and tables, conditional SHAP is labeled “SHAPr”.

\textbf{Other baselines and predictive models:}
CMR, GCMR, CLIQUE, ICI, local LOCO, and the analytical
\LCVI/LOCO benchmarks were implemented in our repository.
CMR, GCMR, and local LOCO used \texttt{xgboost} 3.4.1,
available at \url{https://github.com/dmlc/xgboost}.
The same package supplied predictive models for the real-data
experiments; synthetic experiments used the known DGP prediction
functions.
The CLIQUE baseline used our quantile-replacement implementation,
rather than the upstream CLIQUE R package.

\textbf{Numerical and data-processing tools:}
We used NumPy 2.5.3, SciPy 1.18.1, pandas 3.0.6, and
scikit-learn 1.9.1.
Figures were generated using Matplotlib 3.11.2.
The Python interface to the R SHAPr implementation used
rpy2 3.6.8, and resource measurements used psutil 7.2.2.

\paragraph{Parameters.}
We did not perform hyperparameter tuning. Unless specified below, package hyperparameters were left at their defaults.

For continuous-feature DR-BART models, we used 2,000 burn-in iterations and retained 2,000 posterior draws without thinning (\texttt{nburn}=2,000, \texttt{nsim}=2,000, \texttt{nthin}=1). The conditional mean and precision ensembles contained 200 and 100 trees, respectively (\texttt{m\_mean}=200, \texttt{m\_var}=100). We used variance mode \texttt{ux}, allowing the conditional variance to depend on both the conditioning features and the latent variable.
Group conditional models used the same settings.

For binary and categorical features, the conditional BART models used 50 trees, 1,000 tuning iterations, and 1,000 posterior draws from one chain.

Point-estimation experiments used $M=1,000$ replacements per observation, with evaluation and replacement-prediction batch sizes $A_{\text{batch}}=C_{\text{batch}}=100$. The default experiment seed was 2026.

For real-data experiments, we used five-fold cross-fitting. For each held-out fold, we trained an XGBoost predictive model $g$ on the other four folds using default hyperparameters, without tuning, and held it fixed when computing \LCVI on the held-out observations. This yielded five separately fitted predictive models, each shared across the compared methods within its corresponding fold.
\newpage
\section{Additional DGP Results}
\label{app:dgp}
In this section, we provide the full experimental results, analytical calculations, and additional analyses for the synthetic DGPs considered in the main text. Because this section focuses on concrete simulation settings, we slightly simplify the notation used in the main paper. We use $X_j$ to denote the feature of interest and $\bm X_{-j}$ to denote the remaining features, with lowercase $x_j$ and $\bm x_{-j}$ denoting their observed values. We use $n$ (from now on) for the number of observations and $p$ for the number of features. For the DGPs in the main paper, we set $n=5000$. Unless otherwise stated, all \LCVI quantities are defined with respect to the fixed prediction function used in each DGP.

\subsection{Synthetic Data-Generating Processes}
We consider five synthetic DGPs in the main paper designed to isolate different aspects of local conditional variable importance. Again, for each DGP, Table~\ref{app:tab:dgp-summary} reports the data-generating process, the analytical \LCVI for each relevant feature, and the corresponding global conditional permutation feature importance (CPFI).
\begin{table}[h]
\centering
\caption{Synthetic DGPs and analytical \LCVI and analytical CPFI under squared loss.}
\label{app:tab:dgp-summary}
\renewcommand{\arraystretch}{1.05}
\setlength{\tabcolsep}{4pt}
\begin{tabularx}{\linewidth}{
    @{}
    >{\raggedright\arraybackslash}p{0.16\linewidth}
    >{\raggedright\arraybackslash}X
    >{\raggedright\arraybackslash}p{0.27\linewidth}
    >{\raggedright\arraybackslash}p{0.16\linewidth}
    @{}
}
\toprule
\textbf{DGP}
& \textbf{Data-generating process}
& \textbf{Analytical \LCVI}
& \textbf{CPFI} \\
\midrule

\textbf{Discontinuity}
&
$X_1,X_2,X_3\stackrel{\mathrm{iid}}{\sim}\mathcal N(0,1)$,
$g(\bm x)=x_1\mathbf 1_{x_3>0}+x_2\mathbf 1_{x_3\leq0}$.
&
$I_1=(x_1^2+1)\mathbf 1_{x_3>0}$;
$I_2=(x_2^2+1)\mathbf 1_{x_3\leq0}$;
$I_3=\frac12(x_1-x_2)^2$.
&
$\mathrm{CPFI}_1=1$;
$\mathrm{CPFI}_2=1$;
$\mathrm{CPFI}_3=1$.
\\
\midrule

\textbf{Correlated}
&
$(X_1,X_2)^\top\sim\mathcal N_2(0,\Sigma_\rho)$,
$\Sigma_{\rho,11}=\Sigma_{\rho,22}=1$, $\Sigma_{\rho,12}=\rho$, $\rho=0.9$;
$g(\bm x)=x_1$.
&
$I_1=1-\rho^2+(x_1-\rho x_2)^2$;
$I_2=0$.
&
$\mathrm{CPFI}_1=2(1-\rho^2)$;
$\mathrm{CPFI}_2=0$.
\\
\midrule

\textbf{Heteroskedastic}
&
$X_2\sim\mathcal N(0,1)$,
$X_1\mid X_2=x_2\sim\mathcal N(0,x_2^2)$;
$g(\bm x)=x_1$.
&
$I_1=x_1^2+x_2^2$;
$I_2=0$.
&
$\mathrm{CPFI}_1=2$;
$\mathrm{CPFI}_2=0$.
\\
\midrule

\textbf{Interacted discontinuity}
&
$X_2,X_3\stackrel{\mathrm{iid}}{\sim}\mathcal N(0,1)$,
$\mu=4(\mathbf 1_{X_2>0}\oplus\mathbf 1_{X_3>0})-2$;
$X_1\mid X_2,X_3\sim\mathcal N(\mu,1)$;
$g(\bm x)=x_1$.
&
$I_1=1+(x_1-\mu)^2$;
$I_2=I_3=0$.
&
$\mathrm{CPFI}_1=2$;
$\mathrm{CPFI}_2=\mathrm{CPFI}_3=0$.
\\
\midrule

\textbf{Extreme tail}
&
$X_2\sim\mathcal N(0,1)$,
$X_1\mid X_2=x_2\sim\mathcal N(0,v(x_2))$;
$v(x_2)=1+\mathbf 1_{x_2>0}$,
$H(x_2)=\mathbf 1_{|x_2|>\Phi^{-1}(.99)}$;
$g(\bm x)=x_1+20H(x_2)$.
&
$I_1=x_1^2+v(x_2)$;
$I_2=8+384H(x_2)$.
&
$\mathrm{CPFI}_1=3$;
$\mathrm{CPFI}_2=15.68$.
\\

\bottomrule
\end{tabularx}
\end{table}

\subsection{Analytical LUCID Calculations}
\label{app:analytic_lucid_dgps}

We derive the expressions in Table~\ref{app:tab:dgp-summary}. Under squared loss and using the prediction itself as the reference, the \LCVI of feature $j$ at observation $\bm x$ reduces to
\begin{equation}
I_j(\bm x)=\mathbb E\left[\left\{g(X_j',\bm x_{-j})-g(\bm x)\right\}^2\mid X_{-j}=\bm x_{-j}\right],
\label{app:eq:analytic-lcvi-general}
\end{equation}
where $X_j'\sim P(\cdot\mid \bm X_{-j}=\bm x_{-j})$. The corresponding global conditional permutation feature importance is
\begin{equation}
\mathrm{CPFI}_j=\mathbb E\left[I_j(\bm X)\right].
\label{app:eq:analytic-cpfi-general}
\end{equation}

\paragraph{DGP 1: Discontinuity.} Since the features are mutually independent, the conditional replacement distributions equal their marginals. For $X_1$, replacing $x_1$ affects the prediction only when $x_3>0$, giving
\begin{align}
I_1(\bm x)
&=\mathbf 1_{x_3>0}\mathbb E[(X_1'-x_1)^2]\nonumber\\
&=\mathbf 1_{x_3>0}\left\{\operatorname{Var}(X_1')+(\mathbb E[X_1']-x_1)^2\right\}\nonumber\\
&=(1+x_1^2)\mathbf 1_{x_3>0}.
\end{align}
Similarly,
\begin{equation}
I_2(\bm x)=(1+x_2^2)\mathbf 1_{x_3\leq0}.
\end{equation}
For $X_3$, a replacement changes the prediction only when $X_3'$ falls on the opposite side of zero from the observed $x_3$. Since $\Pr(X_3'>0)=\Pr(X_3'\leq0)=1/2$,
\begin{equation}
I_3(\bm x)=\frac12(x_1-x_2)^2.
\end{equation}
Averaging over the data distribution gives
\begin{align}
\mathrm{CPFI}_1&=\mathbb E[(1+X_1^2)\mathbf 1_{X_3>0}]=2\times\frac12=1,\\
\mathrm{CPFI}_2&=\mathbb E[(1+X_2^2)\mathbf 1_{X_3\leq0}]=2\times\frac12=1,\\
\mathrm{CPFI}_3&=\frac12\mathbb E[(X_1-X_2)^2]=\frac12(1+1)=1.
\end{align}

\paragraph{DGP 2: Correlated.} For the bivariate Gaussian distribution,
\begin{equation}
X_1\mid X_2=x_2\sim\mathcal N(\rho x_2,1-\rho^2).
\end{equation}
Because $g(\bm x)=x_1$, only perturbations of $X_1$ affect the prediction. Hence
\begin{align}
I_1(\bm x)
&=\mathbb E[(X_1'-x_1)^2\mid X_2=x_2]\nonumber\\
&=\operatorname{Var}(X_1'\mid X_2=x_2)+\{\mathbb E[X_1'\mid X_2=x_2]-x_1\}^2\nonumber\\
&=1-\rho^2+(x_1-\rho x_2)^2.
\end{align}
Since $g$ does not depend on $x_2$ directly,
\begin{equation}
I_2(\bm x)=0.
\end{equation}
Moreover, $X_1-\rho X_2$ has variance $1-\rho^2$, so
\begin{align}
\mathrm{CPFI}_1
&=1-\rho^2+\mathbb E[(X_1-\rho X_2)^2]\nonumber\\
&=2(1-\rho^2),
\end{align}
while $\mathrm{CPFI}_2=0$.

\paragraph{DGP 3: Heteroskedastic.} Here
\begin{equation}
X_1\mid X_2=x_2\sim\mathcal N(0,x_2^2),
\end{equation}
so the conditional variance of $X_1$ depends on the local value of $X_2$. Since $g(\bm x)=x_1$,
\begin{align}
I_1(\bm x)
&=\mathbb E[(X_1'-x_1)^2\mid X_2=x_2]\nonumber\\
&=\operatorname{Var}(X_1'\mid X_2=x_2)+\{\mathbb E[X_1'\mid X_2=x_2]-x_1\}^2\nonumber\\
&=x_2^2+x_1^2.
\end{align}
Again, $g$ does not depend directly on $X_2$, so $I_2(\bm x)=0$. Since
\begin{equation}
\mathbb E[X_1^2]=\mathbb E[\operatorname{Var}(X_1\mid X_2)]=\mathbb E[X_2^2]=1,
\end{equation}
we obtain
\begin{equation}
\mathrm{CPFI}_1=\mathbb E[X_1^2+X_2^2]=2,
\qquad
\mathrm{CPFI}_2=0.
\end{equation}

\paragraph{DGP 4: Interacted discontinuity.} Define
\begin{equation}
\mu(X_2,X_3)=4\left\{\mathbf 1_{X_2>0}\oplus\mathbf 1_{X_3>0}\right\}-2.
\end{equation}
Conditional on $(X_2,X_3)$,
\begin{equation}
X_1\mid X_2,X_3\sim\mathcal N(\mu,1).
\end{equation}
Since $g(\bm x)=x_1$, the \LCVI of $X_1$ is
\begin{align}
I_1(\bm x)
&=\mathbb E[(X_1'-x_1)^2\mid X_2=x_2,X_3=x_3]\nonumber\\
&=1+(x_1-\mu)^2.
\end{align}
Although $X_2$ and $X_3$ determine the conditional distribution of $X_1$, neither enters the prediction function once $x_1$ is held fixed. Therefore,
\begin{equation}
I_2(\bm x)=I_3(\bm x)=0.
\end{equation}
Finally, because $X_1-\mu\mid X_2,X_3\sim\mathcal N(0,1)$,
\begin{equation}
\mathrm{CPFI}_1=1+\mathbb E[(X_1-\mu)^2]=2,
\qquad
\mathrm{CPFI}_2=\mathrm{CPFI}_3=0.
\end{equation}

\paragraph{DGP 5: Extreme tail.} Let
\begin{equation}
v(x_2)=1+\mathbf 1_{x_2>0},
\qquad
H(x_2)=\mathbf 1_{|x_2|>c},
\qquad
c=\Phi^{-1}(0.99).
\end{equation}
For $X_1$, conditional on $X_2=x_2$ we have $X_1'\sim\mathcal N(0,v(x_2))$. Since changing $X_1$ leaves the tail indicator fixed,
\begin{align}
I_1(\bm x)
&=\mathbb E[(X_1'-x_1)^2\mid X_2=x_2]\nonumber\\
&=v(x_2)+x_1^2.
\end{align}
Averaging gives
\begin{align}
\mathrm{CPFI}_1
&=\mathbb E[v(X_2)]+\mathbb E[X_1^2]\nonumber\\
&=2\mathbb E[v(X_2)]\nonumber\\
&=2\left(1+\frac12\right)=3.
\end{align}

For $X_2$, the value of $X_1$ remains fixed under perturbation, and therefore
\begin{equation}
g(x_1,X_2')-g(x_1,x_2)=20\{H(X_2')-H(x_2)\}.
\end{equation}
Although $X_1$ and $X_2$ are dependent through $v(X_2)$, the conditional probability that a replacement lies in the extreme tail remains $0.02$. To see this, note that
\begin{equation}
p(x_2\mid x_1)\propto \phi(x_2)\phi\{x_1;0,v(x_2)\},
\end{equation}
where $\phi$ is a Gaussian density\footnote{Here, \(\phi(\cdot)\) denotes the standard Gaussian density, and \(\phi(\cdot;0,v)\) denotes the Gaussian density with mean zero and variance \(v\).}, and $v(x_2)$ is constant within each of the two regions $x_2>0$ and $x_2\leq0$. Thus, conditional on either sign of $X_2$, its distribution remains a standard Gaussian truncated to that half-line. Consequently,
\begin{align}
\Pr(X_2'>c\mid X_1,X_2'>0)&=\frac{\Pr(X_2>c)}{\Pr(X_2>0)}=\frac{0.01}{0.5}=0.02,\\
\Pr(X_2'<-c\mid X_1,X_2'\leq0)&=\frac{\Pr(X_2<-c)}{\Pr(X_2\leq0)}=\frac{0.01}{0.5}=0.02,
\end{align}
and hence
\begin{equation}
\Pr\{H(X_2')=1\mid X_1\}=0.02.
\label{app:eq:extreme-tail-probability}
\end{equation}
Therefore, when $H(x_2)=0$, a perturbation changes the indicator with probability $0.02$, while when $H(x_2)=1$, it changes the indicator with probability $0.98$. It follows that
\begin{align}
I_2(\bm x)
&=400\left[0.02\{1-H(x_2)\}+0.98H(x_2)\right]\nonumber\\
&=8+384H(x_2).
\end{align}
Finally, since $\Pr\{H(X_2)=1\}=0.02$,
\begin{align}
\mathrm{CPFI}_2
&=\mathbb E[8+384H(X_2)]\nonumber\\
&=8+384(0.02)\nonumber\\
&=15.68.
\end{align}

\subsection{Stress Test Across Random Seeds}
\label{app:dgp_stress_test}

To assess the robustness of our synthetic experiments to random variation, we repeat all five DGP experiments over 50 independent random seeds. For each method and DGP, we compute the MSE between the estimated and analytical \LCVI for $X_1$. Table~\ref{tab:dgp_stress_mse} reports the mean MSE together with its standard deviation across the 50 repetitions, thereby showing both average estimation accuracy and its variation across random realizations. We additionally rank the seven \LCVI estimators within each paired DGP--seed experiment and report the resulting mean rank across all $5\times50=250$ settings in Table~\ref{tab:dgp_stress_rank}.

\begin{table}[t]
\centering
\caption{MSE of estimated \LCVI for $X_1$ across 50 independent random seeds. Each entry reports mean $\pm$ standard deviation across seeds. Lower values indicate more accurate and more stable estimation.}
\label{tab:dgp_stress_mse}
\setlength{\tabcolsep}{3.5pt}
\renewcommand{\arraystretch}{1.05}
\begin{tabular}{lccccc}
\toprule
\textbf{Method}
& \textbf{DGP 1}
& \textbf{DGP 2}
& \textbf{DGP 3}
& \textbf{DGP 4}
& \textbf{DGP 5} \\
\midrule
DR-BART
& $0.0225\pm0.0036$
& $\mathbf{0.0029\pm0.0006}$
& $\mathbf{0.3668\pm0.1095}$
& $\mathbf{0.5092\pm0.1293}$
& $0.0865\pm0.0122$ \\

ARF
& $0.1187\pm0.0147$
& $0.0041\pm0.0006$
& $0.6582\pm0.2888$
& $1.2859\pm0.1738$
& $0.2063\pm0.0194$ \\

cs-PFI
& $0.0064\pm0.0019$
& $0.0116\pm0.0016$
& $1.0804\pm0.2386$
& $78.0902\pm5.7078$
& $\mathbf{0.0510\pm0.0125}$ \\

CLIQUE
& $\mathbf{0.0008\pm0.0007}$
& $4.5720\pm0.2497$
& $1.9922\pm0.1122$
& $80.3271\pm1.7660$
& $0.2536\pm0.0027$ \\

CMR
& $0.5556\pm0.0586$
& $0.0106\pm0.0010$
& $2.5597\pm0.2128$
& $1.5994\pm0.2780$
& $0.8806\pm0.0583$ \\

GCMR
& $0.5447\pm0.0603$
& $0.0108\pm0.0010$
& $2.4399\pm0.1563$
& $3.9210\pm0.3313$
& $0.8883\pm0.0590$ \\

Permutation
& $0.0063\pm0.0017$
& $4.5381\pm0.2462$
& $2.0145\pm0.1150$
& $80.0110\pm1.7766$
& $0.2787\pm0.0073$ \\
\bottomrule
\end{tabular}
\end{table}

The repeated experiments confirm that the qualitative conclusions of the main experiments are stable across random seeds. DR-BART achieves the lowest mean MSE in DGPs~2--4 and remains highly competitive in DGP~5. In DGP~1, where the relevant features are independent and no conditional density estimation is required, simpler permutation-based approaches attain smaller errors. DGP~5 is similarly favorable to cs-PFI, although DR-BART remains close and exhibits low variability across repetitions. Most importantly, DR-BART does not rely on a single favorable dependence structure: its MSE remains consistently low across linear dependence, heteroskedasticity, discontinuous interactions, and extreme-tail behavior. The relatively small standard deviations further indicate that this performance is not driven by a small number of favorable random realizations.

\begin{table}[t]
\centering
\caption{Mean MSE rank for $X_1$ across the five DGPs and 50 independent random seeds, giving $250$ paired DGP--seed comparisons. Lower is better.}
\label{tab:dgp_stress_rank}
\setlength{\tabcolsep}{7pt}
\renewcommand{\arraystretch}{1.03}
\begin{tabular}{lc}
\toprule
\textbf{Method} & \textbf{Mean rank} \\
\midrule
DR-BART & \textbf{1.80} \\
ARF & 2.84 \\
cs-PFI & 3.29 \\
CLIQUE & 4.60 \\
Permutation & 4.80 \\
CMR & 5.16 \\
GCMR & 5.51 \\
\bottomrule
\end{tabular}
\end{table}

DR-BART also obtains the best overall mean rank of $1.80$ across the $250$ paired settings, followed by ARF ($2.84$) and cs-PFI ($3.29$). Together, the MSE distributions and ranking results show that the relative performance observed in the main experiments persists across repeated data realizations rather than being driven by a particular random seed.

\subsection{DGP 1: Discontinuity.}
\label{app:subsec:discontinuity}
 Additional results are provided in Figures~\ref{app:fig:cde-regime-dependent} and~\ref{app:fig:regime-dependent-all-methods}.

\subsection{DGP 2: Correlated Features}
\label{app:subsec:correlated}
We show the results for $\rho=0.9$ in Figure~\ref{app:fig:strong-conditional-dependence-all-methods}. The first row contains the conditional permutation-based methods, which accurately capture the correlation-induced restriction on the replacement distribution in this relatively simple setting. In contrast, marginal permutation does not account for this dependence structure and therefore fail to recover the corresponding local importance pattern. Additional results are provided in Figures~\ref{app:fig:cde-strong-conditional-dependence}.

\subsection{DGP 3: Heteroskedastic Conditional Dependence}
\label{app:subsec:Heteroskedastic}
Recall that $X_2\sim N(0,1)$, $X_1\mid X_2=x_2\sim N(0,x_2^2)$, and $g(\bm x)=x_1$. Under squared loss, the analytical importance of $X_1$ is
\[
I_1(\bm x)=\mathbb E[(X_1'-x_1)^2\mid X_2=x_2]=x_1^2+x_2^2.
\]
Now consider any replacement scheme that uses the same distribution $Q$ for every observation, with $\mathbb E_Q[(X_1')^2]<\infty$. Its population local score is
\[
I_1^Q(\bm x):=\mathbb E_Q[(X_1'-x_1)^2]
=x_1^2-2x_1\mathbb E_Q[X_1']+\mathbb E_Q[(X_1')^2].
\]
Consequently,
\[
I_1^Q(\bm x)-I_1(\bm x)
=\mathbb E_Q[(X_1')^2]-x_2^2-2x_1\mathbb E_Q[X_1'].
\]
Since $\mathbb E[X_1\mid X_2]=0$, the cross term vanishes:
\[
\mathbb E\!\left[X_1\left\{\mathbb E_Q[(X_1')^2]-X_2^2\right\}\right]
=\mathbb E\!\left[\mathbb E[X_1\mid X_2]\left\{\mathbb E_Q[(X_1')^2]-X_2^2\right\}\right]=0.
\]
Using $\mathbb E[X_1^2]=\mathbb E[X_2^2]=1$ and $\mathbb E[X_2^4]=3$, we obtain
\begin{align*}
\mathbb E[(I_1^Q(\bm X)-I_1(\bm X))^2]
&=\mathbb E\!\left[\left\{\mathbb E_Q[(X_1')^2]-X_2^2\right\}^2\right]
+4\{\mathbb E_Q[X_1']\}^2\mathbb E[X_1^2]\\
&=\{\mathbb E_Q[(X_1')^2]\}^2-2\mathbb E_Q[(X_1')^2]+3
+4\{\mathbb E_Q[X_1']\}^2\\
&=\left\{\mathbb E_Q[(X_1')^2]-1\right\}^2+2
+4\{\mathbb E_Q[X_1']\}^2\\
&\ge 2.
\end{align*}
Thus, every common replacement distribution incurs a population local MSE of at least $2$ in this DGP, even when its induced score is computed exactly.

Nevertheless, the global averages satisfy
\[
\mathbb E[I_1(\bm X)]=2,\qquad
\mathbb E[I_1^Q(\bm X)]=1+\mathbb E_Q[(X_1')^2].
\]
In particular, the marginal distribution of $X_1$, which is also the pooled residual distribution under the correct conditional mean, satisfies $\mathbb E_Q[X_1']=0$ and $\mathbb E_Q[(X_1')^2]=1$. It therefore yields
\[
I_1^Q(\bm x)=x_1^2+1,\qquad
\mathbb E[I_1^Q(\bm X)]=\mathbb E[I_1(\bm X)]=2,\qquad
\mathbb E[(I_1^Q(\bm X)-I_1(\bm X))^2]=2.
\]
Hence, exact recovery of global-average importance does not imply accurate local importance. The local error arises from replacing the context-dependent conditional variance $x_2^2$ with a common variance and persists even with the correct conditional mean and arbitrarily many Monte Carlo replacements.

Additional results are provided in Figures~\ref{app:fig:cde-heteroskedastic} and~\ref{app:fig:heteroskedastic-conditional-dependence-all-methods}.

\subsection{DGP 4: Interacted Discontinuity.}
\label{app:subsec:interacted}
\textbf{What happens to local VI when the covariate dependence involves interactions?} 
This case examines dependence driven by an XOR-type interaction between $X_2$ and $X_3$ in the conditional distribution of $X_1$. Such interactions can be difficult for recursive tree methods based on univariate split selection, because their joint signal is not apparent from individual univariate splits \citep{hornung2022interaction}. We use this setting to assess how this difficulty affects local importance estimation. As shown in the results for DGP 4 in Table~\ref{tab:dgp-x1-lcvi-mse}, DR-BART (which explores ensembles of tree structures through Markov Chain Monte Carlo simulations \citep{orlandi2021density}) yields estimates closest to the analytical \LCVI. The cs-PFI estimates closely resemble marginal permutation, consistent with its transformation-tree partition failing to recover the pure interaction. ARF substantially improves on marginal permutation, but its MSE is still approximately 3$\times$ higher than our DR-BART implementation. These results highlight the importance of recovering interaction structure in the conditional distribution when estimating \LCVI. Additional results are provided in Figures~\ref{app:fig:cde-interaction} and~\ref{app:fig:interaction-conditional-dependence-all-methods}.

\begin{figure}[t]
    \centering
    \includegraphics[width=\columnwidth]{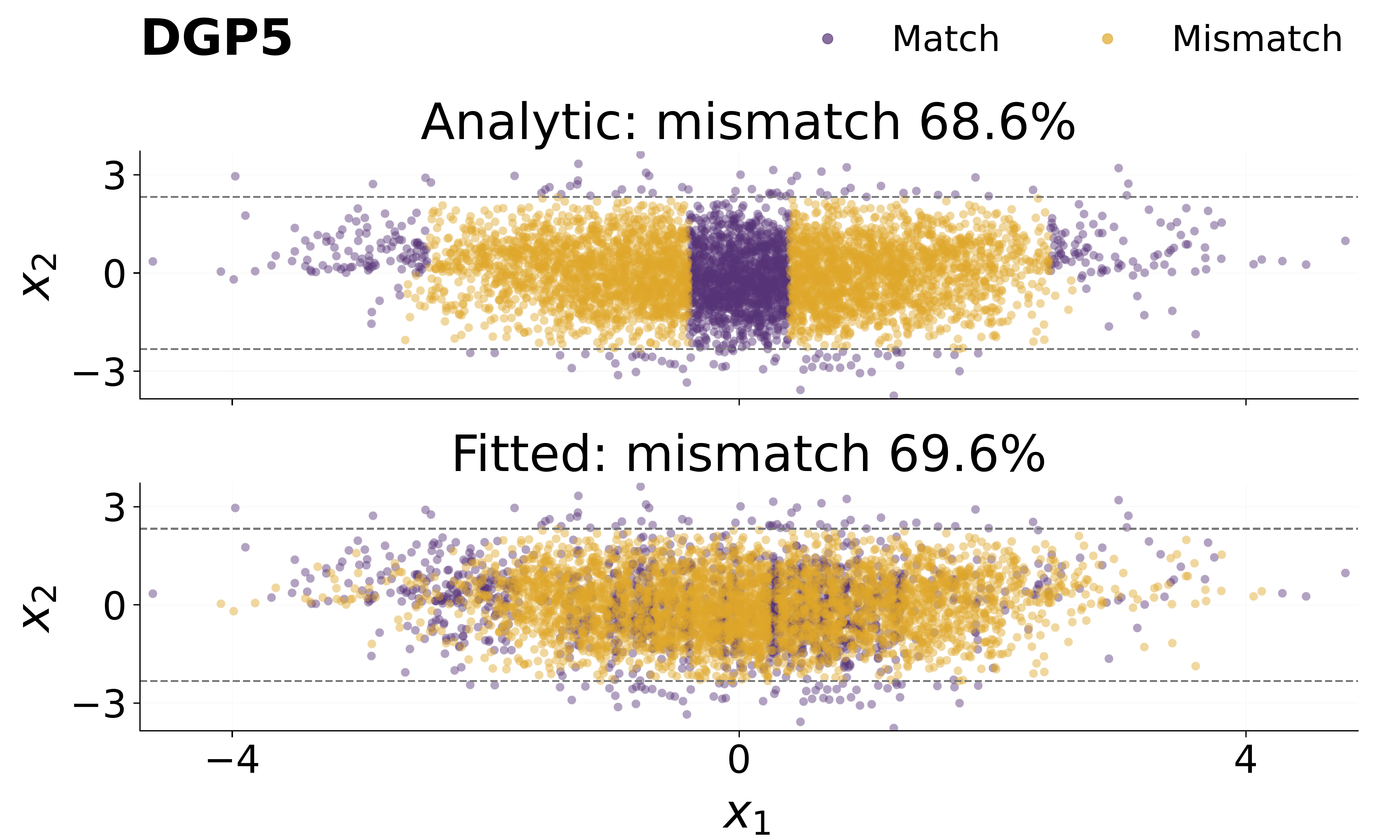}
    \caption{\textbf{Extreme tail.} Match and mismatch regions for the analytical and fitted \LCVI structures in DGP~5. Match and mismatch regions between LUCID and local LOCO feature rankings in DGP 5, using analytical and estimated importance measures.}
    \label{fig:main_experiments_dgp_5}
\end{figure}
\subsection{DGP 5: Extreme Tail.}
\label{app:subsec:extrem_tail}
\textbf{How does \LCVI differ from local LOCO?} 
LOCO provides a natural comparison for \LCVI because both assess a feature's contribution given the information available in the remaining features. Although global LOCO and global conditional permutation importance differ only by a factor of two under squared loss and an oracle regression predictor \citep{reyero2025conditional}, this proportionality does not generally hold locally, as established in \cref{subsec:lcvi_loco}. Compared to local LOCO, \LCVI additionally accounts for prediction variability under conditional feature replacement, reflecting how important it is to know the exact feature value; omitting this component can substantially alter local feature rankings. We construct a heteroskedastic DGP with an extreme-tail component in the prediction function to examine the consequences of this distinction. Analytical local LOCO and analytical \LCVI yield different feature rankings with an analytical probability of $67.4\%$ (Analytical local LOCO and \LCVI assign different feature rankings to 68.6\% of the evaluated observations.). Consistent with this theoretical discrepancy, estimated local LOCO disagrees with estimated \LCVI on $69.6\%$ of the evaluated instances, as shown in Figure~\ref{fig:main_experiments_dgp_5}. The disagreement between the analytical rankings demonstrates that this discrepancy is intrinsic to the two importance measures and cannot be eliminated by improving LOCO estimation alone. This example shows that \LCVI provides fuller local feature information about dispersion, leading to substantially different feature rankings from local LOCO despite their proportionality at the global level. Additional results are provided in Figures~\ref{app:fig:cde-extreme-tail} and~\ref{app:fig:extreme-tail-all-methods}. 

\subsection{Complete Conditional Replacement Distributions}
\label{app:dgp_cde_complete}

We further show complete results of conditional replacement distributions produced by each sampling method across representative covariates. For each DGP, we evaluate the conditional distribution of $X_1$ at five empirical quantiles of $X_2$ and compare the estimated replacement distributions with the analytical conditional distribution. Across the five settings, DR-BART is the most stable method overall: it consistently tracks changes in the location, scale, and shape of the true conditional distribution. The remaining methods can perform well when their structural assumptions align with the DGP, but their performance is less uniform across the different conditional structures.

\paragraph{DGP 1: Discontinuity.} Figure~\ref{app:fig:cde-regime-dependent} shows the conditional replacement distributions for the discontinuity DGP. Because $X_1$ is independent of the remaining features, its true conditional distribution does not vary with $X_2$. DR-BART closely preserves this invariant distribution across all five conditioning points. Most methods also perform reasonably well in this simple setting, as no conditional dependence needs to be learned, although ARF exhibits noticeable shifts in location or scale at the more extreme conditioning points.

\begin{figure}[p]
\centering
\includegraphics[width=\linewidth,height=0.84\textheight,keepaspectratio]{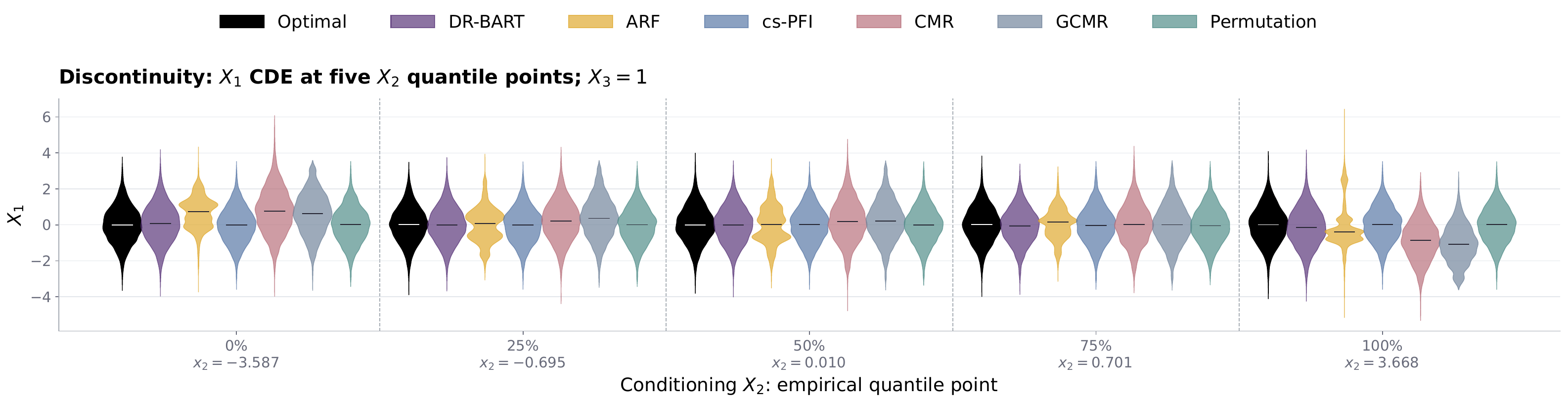}
\caption{Estimated conditional replacement distributions for $X_1$ in the discontinuity DGP at five empirical quantiles of $X_2$, with $X_3=1$. The analytical conditional distribution is shown as Optimal.}
\label{app:fig:cde-regime-dependent}
\end{figure}

\paragraph{DGP 2: Correlated features.} Figure~\ref{app:fig:cde-strong-conditional-dependence} illustrates the effect of strong linear dependence. As $X_2$ moves from its lower to upper quantiles, the true conditional distribution of $X_1$ shifts correspondingly while remaining substantially narrower than its marginal distribution. DR-BART closely follows both the changing conditional location and the restricted conditional spread across the full range of $X_2$. In contrast, marginal permutation remains centered around the marginal distribution and therefore entirely misses the correlation-induced restriction, while several other conditional methods capture the trend but exhibit larger errors in location or scale at the extremes.

\begin{figure}[p]
\centering
\includegraphics[width=\linewidth,height=0.84\textheight,keepaspectratio]{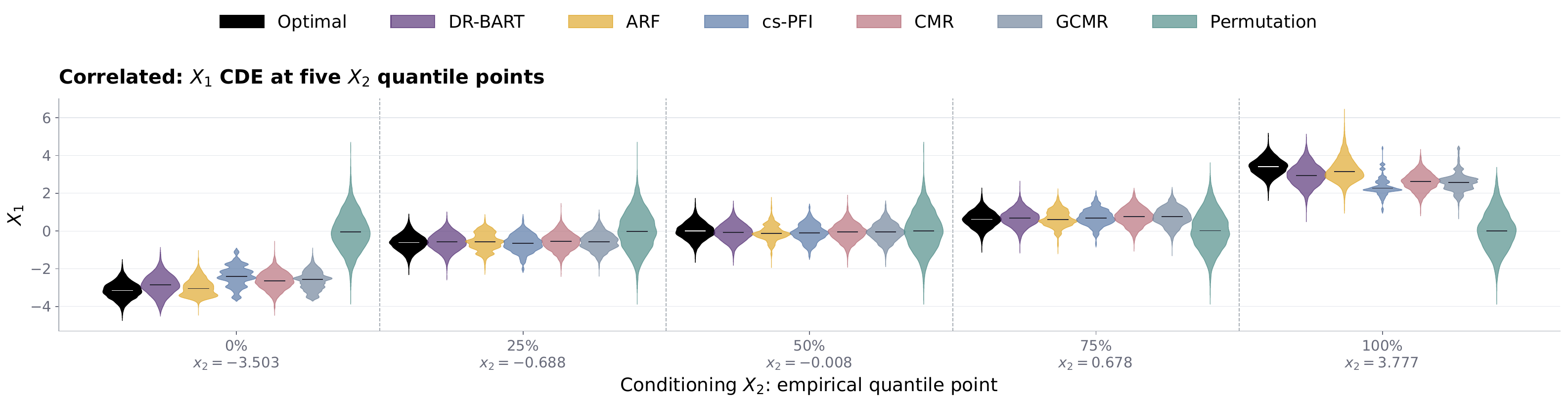}
\caption{Estimated conditional replacement distributions for $X_1$ in the correlated-features DGP at five empirical quantiles of $X_2$.}
\label{app:fig:cde-strong-conditional-dependence}
\end{figure}

\paragraph{DGP 3: Heteroskedastic conditional dependence.} Figure~\ref{app:fig:cde-heteroskedastic} isolates changes in conditional variance while keeping the conditional mean fixed at zero. The analytical distribution becomes increasingly concentrated as $X_2$ approaches zero and substantially more dispersed in the tails. DR-BART reproduces this local scale variation particularly closely, including the near-degenerate distribution around the median of $X_2$ and the much wider distributions at the extreme quantiles. Methods that rely on permuting residual distributions cannot preserve this heteroskedastic structure, while subgroup-based methods recover some scale variation but less uniformly across conditioning points.

\begin{figure}[p]
\centering
\includegraphics[width=\linewidth,height=0.84\textheight,keepaspectratio]{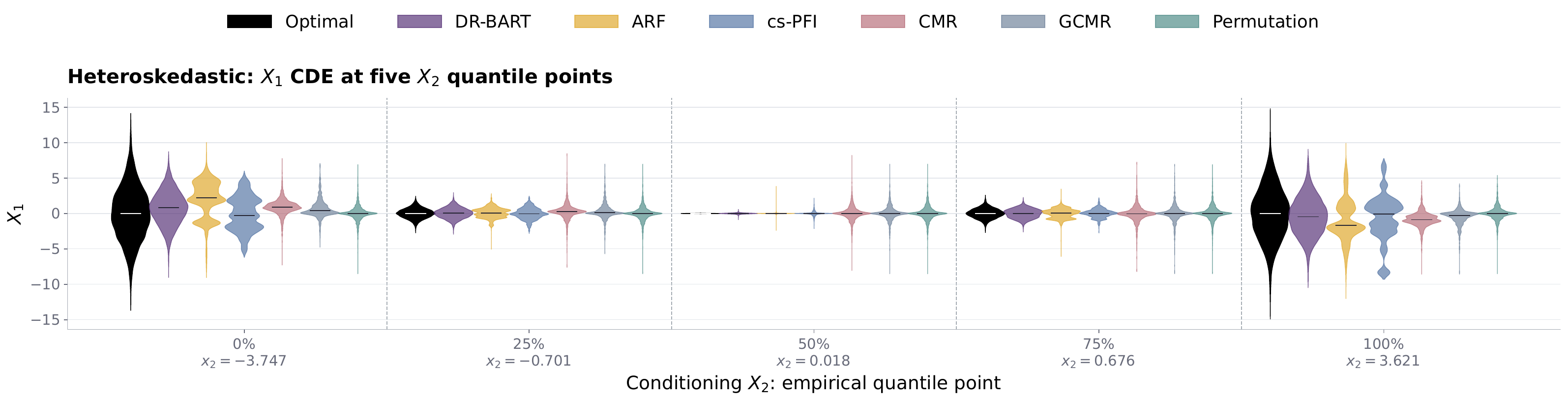}
\caption{Estimated conditional replacement distributions for $X_1$ in the heteroskedastic DGP at five empirical quantiles of $X_2$.}
\label{app:fig:cde-heteroskedastic}
\end{figure}

\paragraph{DGP 4: Interacted discontinuity.} Figure~\ref{app:fig:cde-interaction} presents the most structurally challenging conditional mean in our five DGPs. With $X_3=1$, the conditional distribution of $X_1$ undergoes an abrupt location shift as $X_2$ crosses zero while maintaining approximately constant conditional variance. DR-BART accurately captures both sides of this discontinuity and transitions sharply at the relevant boundary. ARF, CMR, and GCMR also capture much of the location shift, whereas cs-PFI and marginal permutation fail to reproduce the correct regime-specific conditional location in this configuration. This example highlights the importance of representing interactions in the conditioning variables rather than only marginal or coarse subgroup structure.

\begin{figure}[p]
\centering
\includegraphics[width=\linewidth,height=0.84\textheight,keepaspectratio]{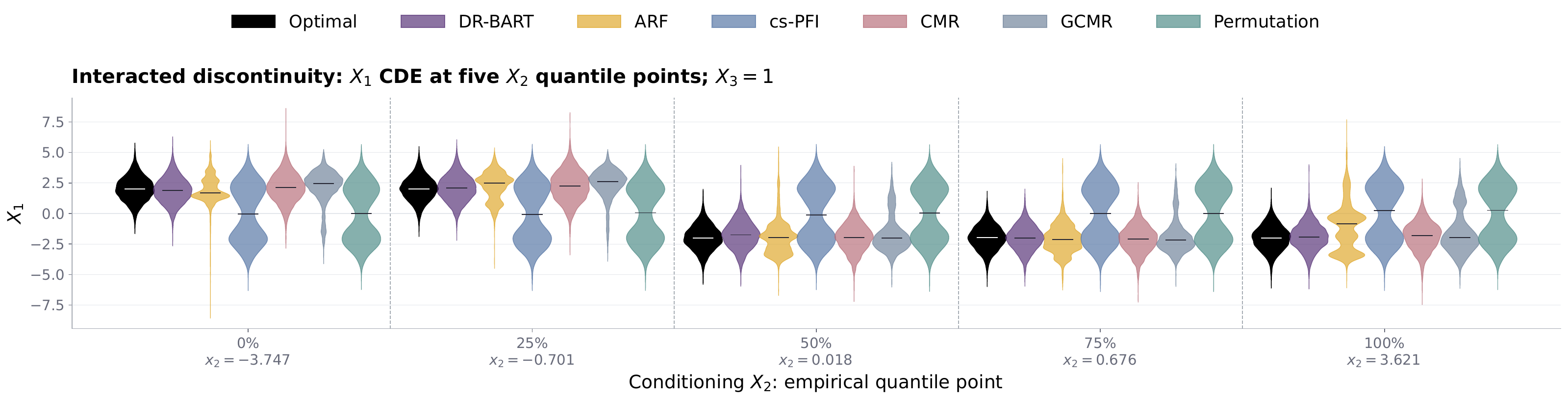}
\caption{Estimated conditional replacement distributions for $X_1$ in the interacted-discontinuity DGP at five empirical quantiles of $X_2$, with $X_3=1$.}
\label{app:fig:cde-interaction}
\end{figure}

\paragraph{DGP 5: Extreme tail.} Figure~\ref{app:fig:cde-extreme-tail} shows the conditional distributions in the extreme-tail setting. Although the conditional mean remains zero, the conditional variance changes across the sign of $X_2$, producing a moderate but important change in distributional spread. DR-BART again closely follows the analytical distribution across all five conditioning points and preserves the increase in conditional dispersion for positive $X_2$. Several alternative methods introduce noticeable local shifts in the conditional mean or fail to reproduce the change in scale, while marginal permutation necessarily uses essentially the same unconditional replacement distribution throughout.

\begin{figure}[p]
\centering
\includegraphics[width=\linewidth,height=0.84\textheight,keepaspectratio]{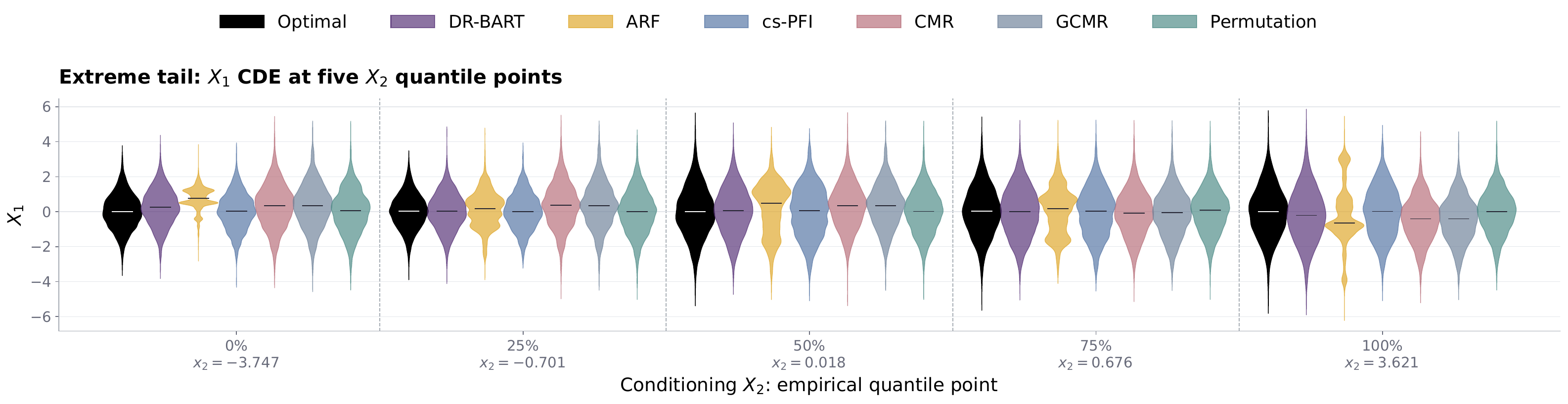}
\caption{Estimated conditional replacement distributions for $X_1$ in the extreme-tail DGP at five empirical quantiles of $X_2$.}
\label{app:fig:cde-extreme-tail}
\end{figure}

Taken together, these results show that the advantage of DR-BART is not tied to a single type of conditional dependence. It remains stable across independence, smooth location shifts, heteroskedasticity, discontinuous interactions, and sign-dependent changes in scale. This flexibility is particularly relevant for \LCVI, since errors in the local replacement distribution directly propagate to errors in observation-level importance.

\subsection{Complete Pointwise Results}
\label{app:dgp_pointwise}

We provide the complete observation-level local variable importance for all methods and all five DGPs in Figures~\ref{app:fig:regime-dependent-all-methods}--\ref{app:fig:extreme-tail-all-methods}. Each panel compares the estimated \LCVI with its analytical value across evaluation observations, with the corresponding mean squared error (MSE) reported above each feature. These figures complement the aggregate MSE comparisons in the main text by showing where estimation errors arise across the covariate space.

\textbf{Note: MSE is reported uniformly across all methods for completeness; however, the values for SHAP, conditional SHAP, and LIME should not be interpreted as directly comparable estimation errors because these methods target different estimands from \LCVI.}

\begin{figure}[p]
\centering
\includegraphics[width=\linewidth,height=0.88\textheight,keepaspectratio]{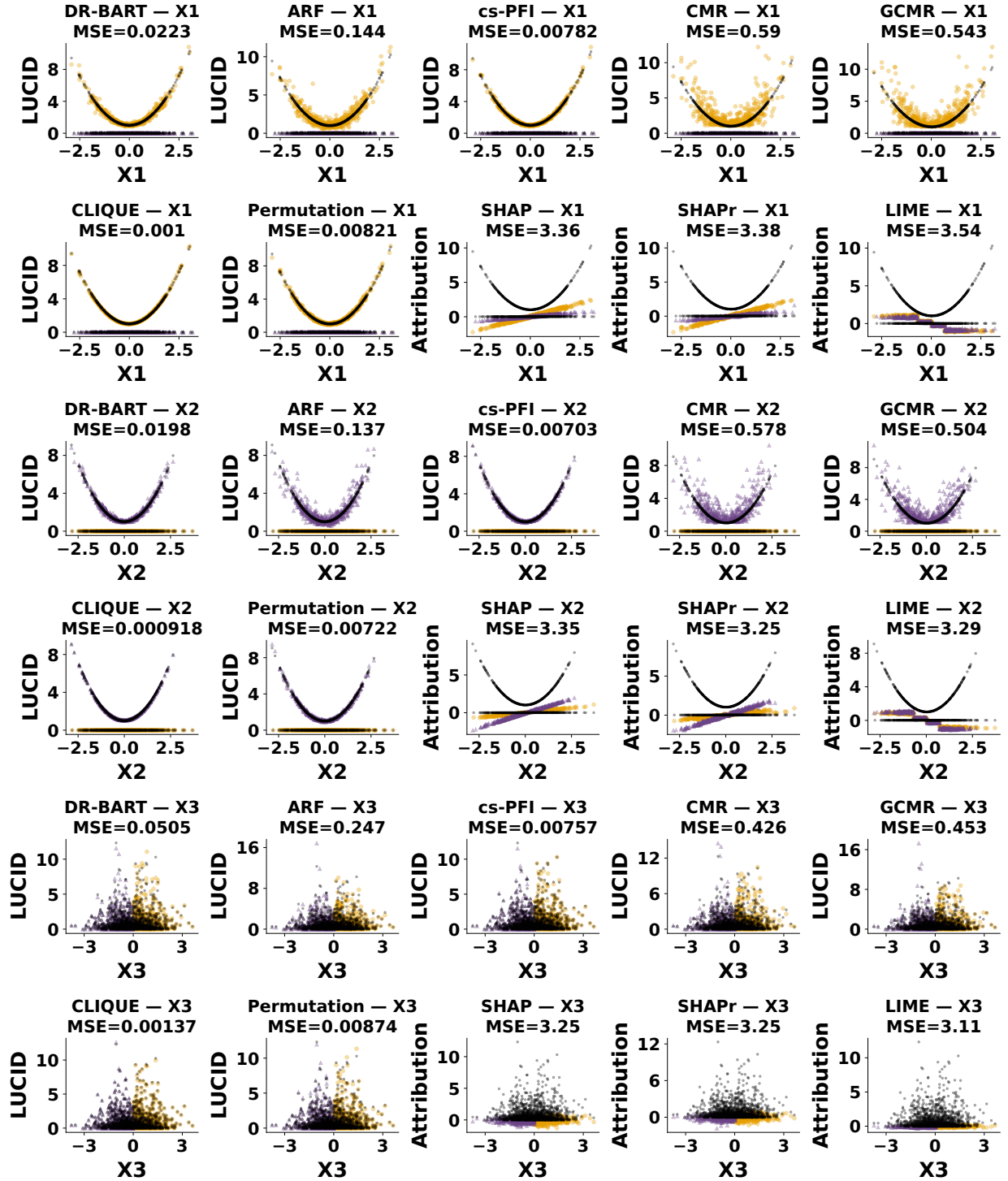}
\caption{Complete pointwise \LCVI results for the \textbf{regime-dependent} DGP. Estimated local variable importance values from all conditional sampling methods are compared with the corresponding analytical \LCVI across evaluation observations.}
\label{app:fig:regime-dependent-all-methods}
\end{figure}

\begin{figure}[p]
\centering
\includegraphics[width=\linewidth,height=0.88\textheight,keepaspectratio]{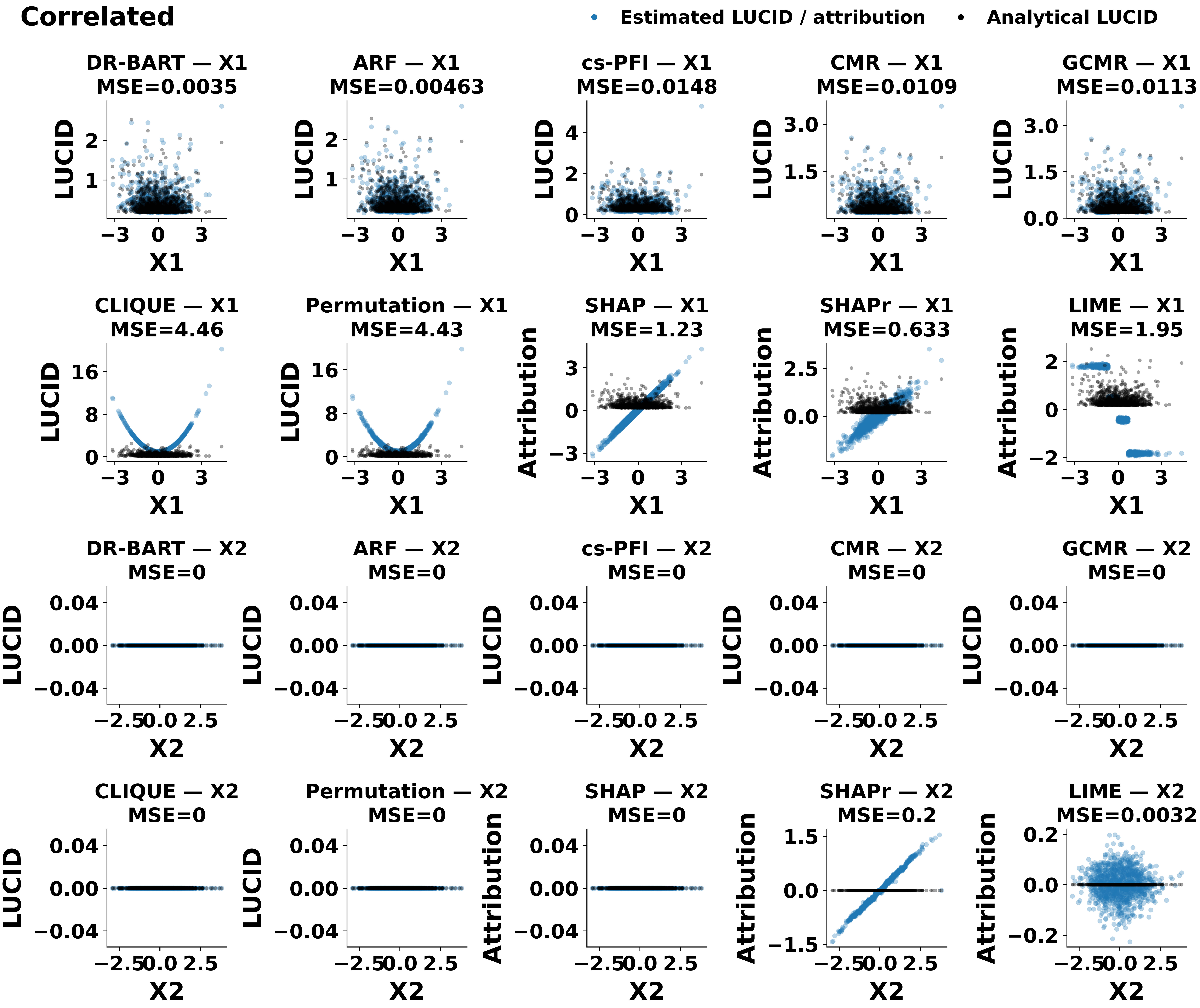}
\caption{Complete pointwise \LCVI results for the \textbf{strong conditional dependence} DGP. Estimated local variable importance values from all conditional sampling methods are compared with the corresponding analytical \LCVI across evaluation observations.}
\label{app:fig:strong-conditional-dependence-all-methods}
\end{figure}

\begin{figure}[p]
\centering
\includegraphics[width=\linewidth,height=0.88\textheight,keepaspectratio]{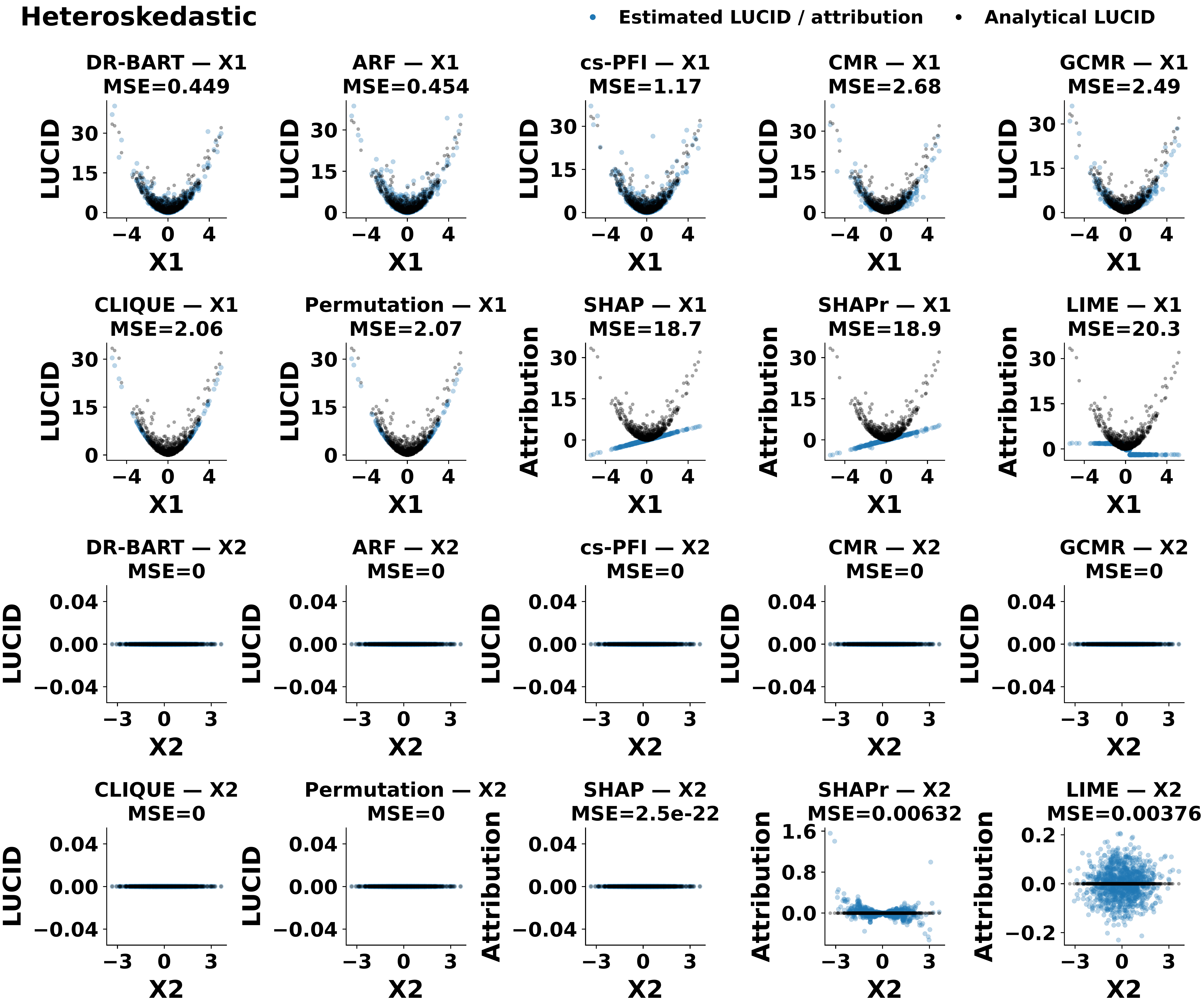}
\caption{Complete pointwise \LCVI results for the \textbf{heteroskedastic conditional dependence} DGP. Estimated local variable importance values from all conditional sampling methods are compared with the corresponding analytical \LCVI across evaluation observations.}
\label{app:fig:heteroskedastic-conditional-dependence-all-methods}
\end{figure}

\begin{figure}[p]
\centering
\includegraphics[width=\linewidth,height=0.88\textheight,keepaspectratio]{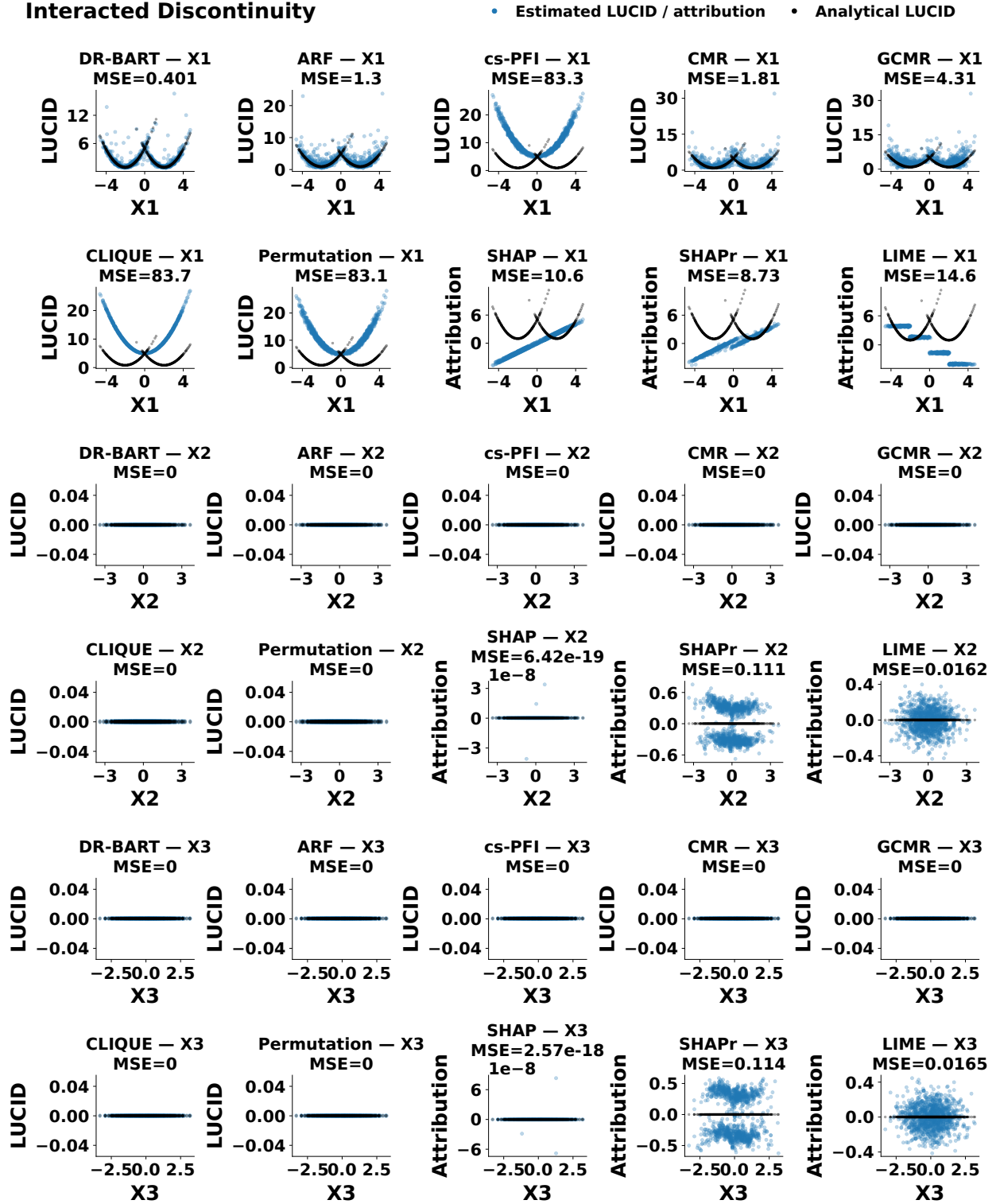}
\caption{Complete pointwise \LCVI results for the \textbf{interaction conditional dependence} DGP. Estimated local variable importance values from all conditional sampling methods are compared with the corresponding analytical \LCVI across evaluation observations.}
\label{app:fig:interaction-conditional-dependence-all-methods}
\end{figure}

\begin{figure}[p]
\centering
\includegraphics[width=\linewidth,height=0.88\textheight,keepaspectratio]{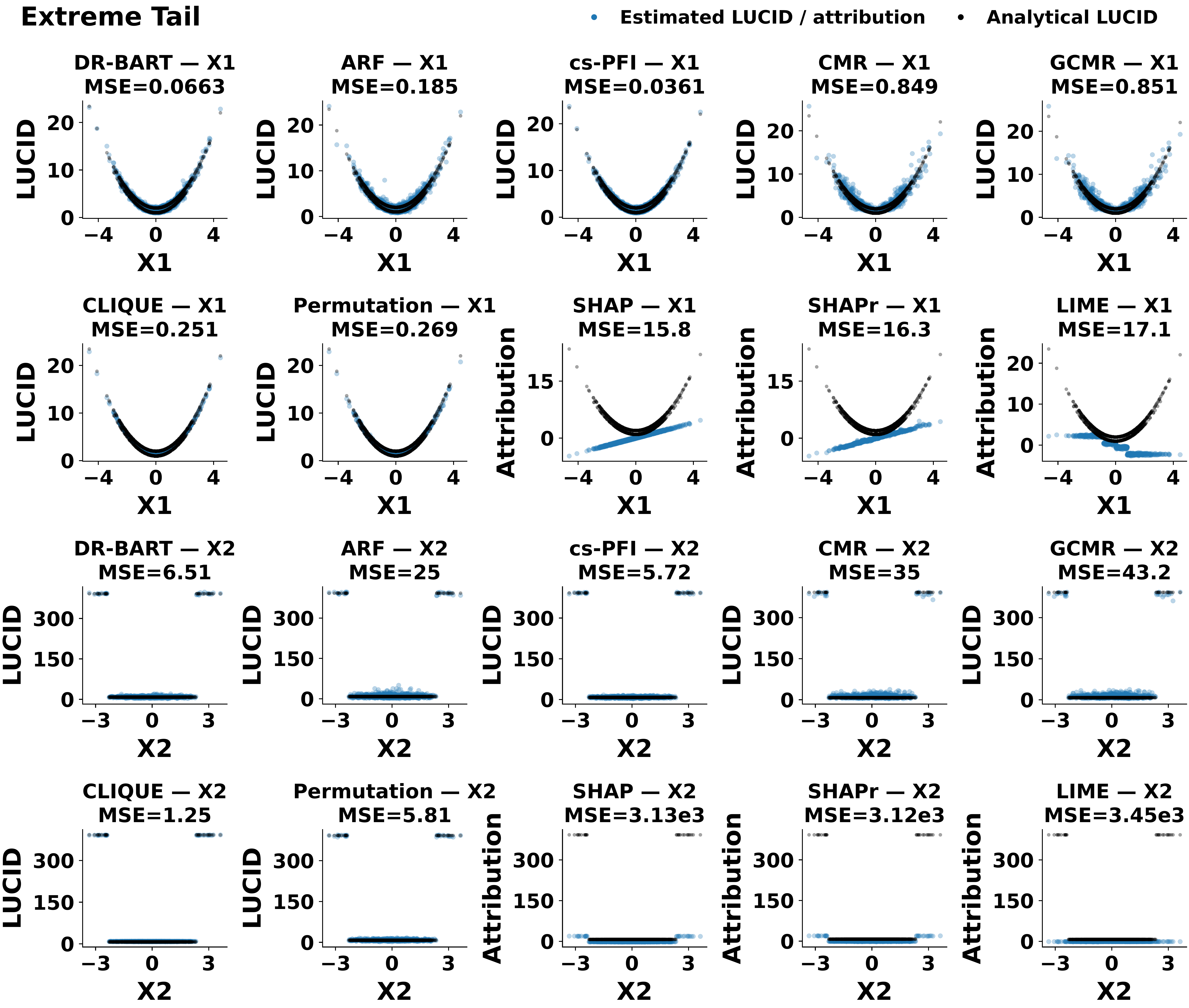}
\caption{Complete pointwise \LCVI results for the \textbf{extreme-tail} DGP. Estimated local variable importance values from all conditional sampling methods are compared with the corresponding analytical \LCVI across evaluation observations.}
\label{app:fig:extreme-tail-all-methods}
\end{figure}
\subsection{Additional Posterior Credible Interval Results}
\label{app:credible_interval_additional}

We further evaluate the calibration of the DR-BART posterior \LCVI intervals across multiple nominal coverage levels for observations where \(X_1\) is active in the prediction function. Table~\ref{tab:dgp-credible-interval-additional} reports the empirical coverage ($C$) and average interval width ($W$) for nominal levels of $50\%$, $80\%$, $90\%$, and $95\%$. Across the five DGPs, empirical coverage generally increases with the nominal level as expected, while interval widths increase accordingly. Coverage is close to or above nominal in most settings, except that DGP 3 shows mild undercoverage at the 95\% nominal level, where the more extreme heteroskedastic observations remain challenging in finite samples. Figures~\ref{app:fig:dgp-ci-50}--\ref{app:fig:dgp-ci-95} show the corresponding pointwise posterior credible intervals for $X_1$ across all five DGPs. Observations are ordered by the observed value of $X_1$, and we randomly select only $1000$ points to avoid a crowded plot. Each figure compares the posterior mean and credible interval with the analytical \LCVI, highlighting observations for which the analytical value is covered or missed.

\begin{table}[t]
\centering
\caption{Empirical coverage ($C$) and average width ($W$) of posterior \LCVI credible intervals for observations where \(X_1\) is active in the prediction function across nominal levels.}
\label{tab:dgp-credible-interval-additional}
\setlength{\tabcolsep}{4pt}
\renewcommand{\arraystretch}{1.03}
\begin{tabular}{lrrrrrrrr}
\toprule
& \multicolumn{2}{c}{50\%} & \multicolumn{2}{c}{80\%} & \multicolumn{2}{c}{90\%} & \multicolumn{2}{c}{95\%} \\
\cmidrule(lr){2-3}
\cmidrule(lr){4-5}
\cmidrule(lr){6-7}
\cmidrule(lr){8-9}
\textbf{DGP} & $C$ & $W$ & $C$ & $W$ & $C$ & $W$ & $C$ & $W$ \\
\midrule
DGP 1 & 0.686 & 0.311 & 0.948 & 0.596 & 0.984 & 0.766 & 0.994 & 0.910 \\
DGP 2 & 0.519 & 0.063 & 0.828 & 0.121 & 0.917 & 0.155 & 0.958 & 0.184 \\
DGP 3 & 0.640 & 0.382 & 0.871 & 0.734 & 0.913 & 0.944 & 0.929 & 1.123 \\
DGP 4 & 0.606 & 0.393 & 0.872 & 0.753 & 0.932 & 0.969 & 0.956 & 1.152 \\
DGP 5 & 0.676 & 0.387 & 0.929 & 0.742 & 0.971 & 0.956 & 0.985 & 1.138 \\
\bottomrule
\end{tabular}
\end{table}

\begin{figure}[p]
\centering
\includegraphics[width=\linewidth,height=0.84\textheight,keepaspectratio]{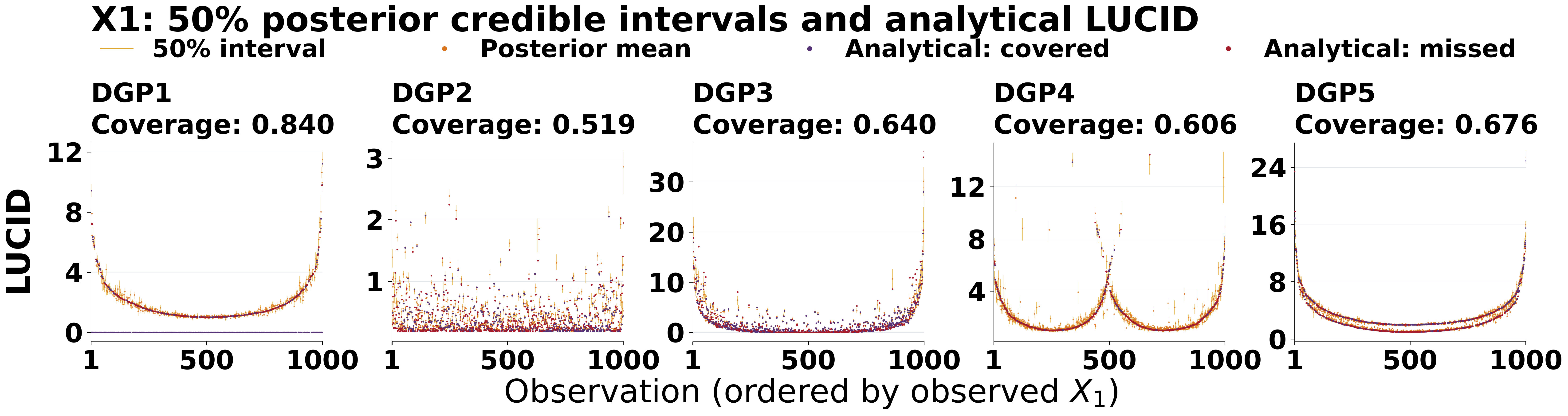}
\caption{\textbf{50\% posterior credible intervals.} Pointwise posterior \LCVI credible intervals for $X_1$ across the five DGPs at the $50\%$ nominal level. Posterior means are compared with analytical \LCVI values, with covered and missed analytical values indicated separately.}
\label{app:fig:dgp-ci-50}
\end{figure}

\begin{figure}[p]
\centering
\includegraphics[width=\linewidth,height=0.84\textheight,keepaspectratio]{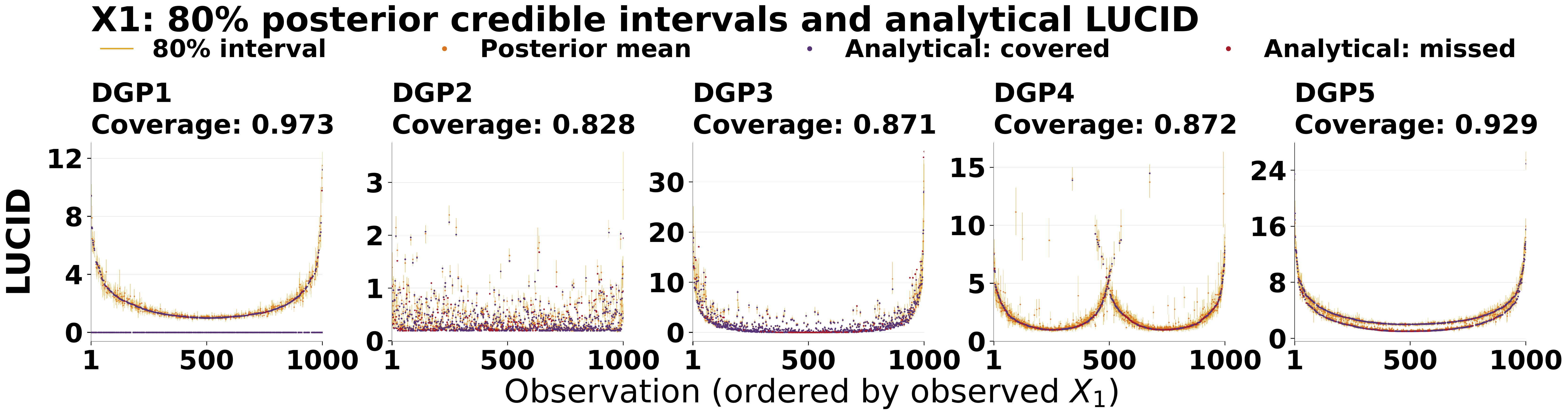}
\caption{\textbf{80\% posterior credible intervals.} Pointwise posterior \LCVI credible intervals for $X_1$ across the five DGPs at the $80\%$ nominal level. Posterior means are compared with analytical \LCVI values, with covered and missed analytical values indicated separately.}
\label{app:fig:dgp-ci-80}
\end{figure}

\begin{figure}[p]
\centering
\includegraphics[width=\linewidth,height=0.84\textheight,keepaspectratio]{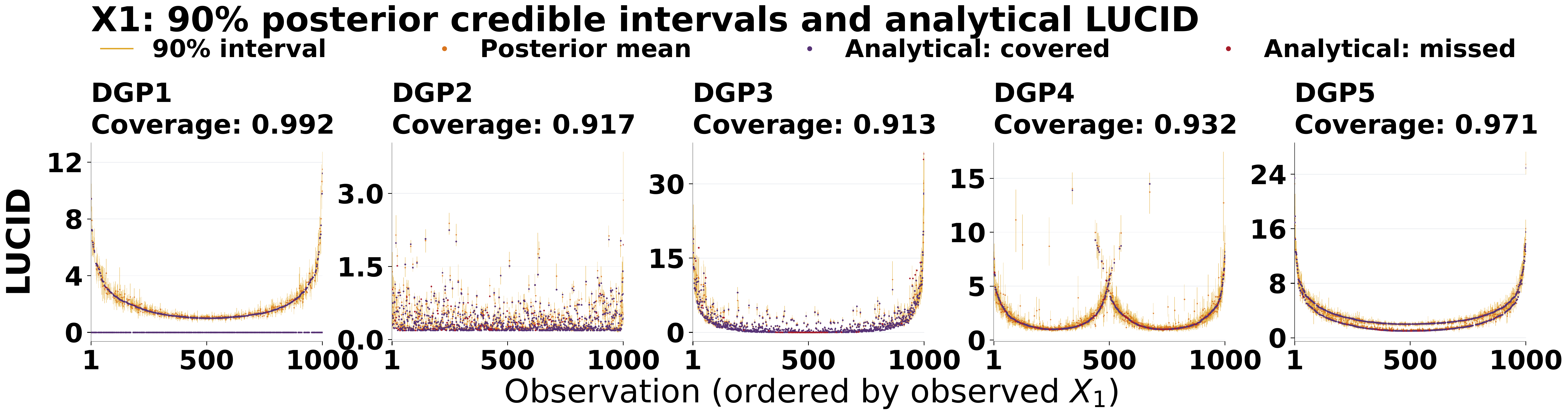}
\caption{\textbf{90\% posterior credible intervals.} Pointwise posterior \LCVI credible intervals for $X_1$ across the five DGPs at the $90\%$ nominal level. Posterior means are compared with analytical \LCVI values, with covered and missed analytical values indicated separately.}
\label{app:fig:dgp-ci-90}
\end{figure}

\begin{figure}[p]
\centering
\includegraphics[width=\linewidth,height=0.84\textheight,keepaspectratio]{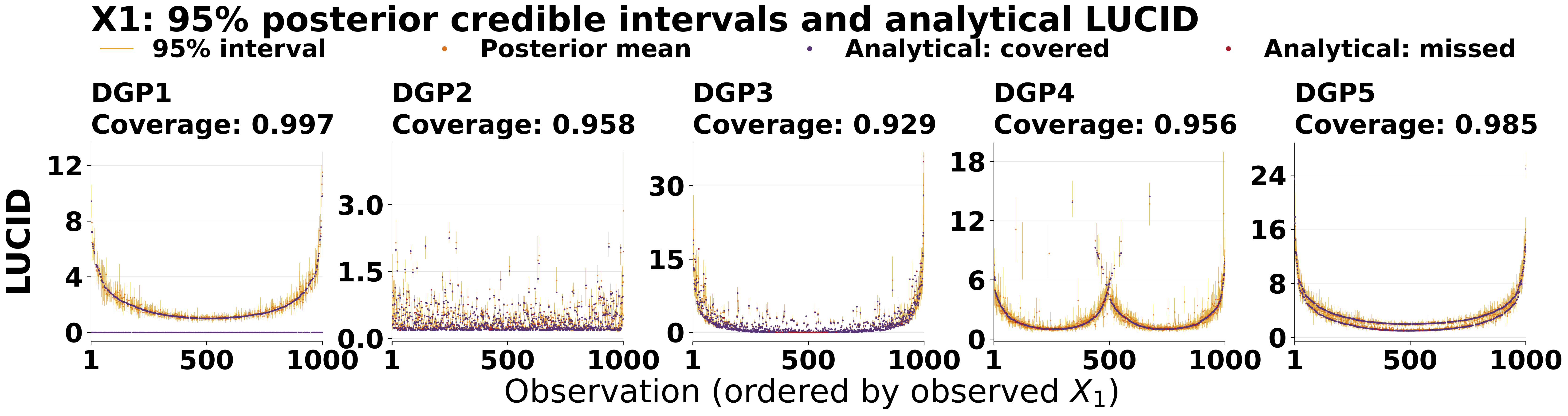}
\caption{\textbf{95\% posterior credible intervals.} Pointwise posterior \LCVI credible intervals for $X_1$ across the five DGPs at the $95\%$ nominal level. Posterior means are compared with analytical \LCVI values, with covered and missed analytical values indicated separately.}
\label{app:fig:dgp-ci-95}
\end{figure}

\clearpage
\subsection{Additional Results on the cs-PFI Simulation}
\label{app:add_results_cspfi}

To further evaluate \LCVI estimation in higher-dimensional settings, we consider the ground-truth simulation introduced by \citet{molnar2024model} for evaluating conditional permutation feature importance. We retain the original scenario names---\textit{independent}, \textit{linear}, \textit{multiple linear dependencies}, and \textit{non-linear}---and vary the sample size and ambient dimension as in their experiment. Following their setting~(I), we evaluate the true prediction function directly. Outcome noise is omitted because the \LCVI calculation does not use the outcome. We only care about the given model.

The prediction function is
\begin{equation}
g(\bm x)=x_1x_2+\sum_{j=1}^{10}x_j,
\label{app:eq:cspfi-model}
\end{equation}
where all features except $X_1$ are independently distributed as $X_j\sim\mathcal N(0,1)$. Feature $X_1$ is generated according to
\begin{equation}
X_1=h(\bm X_{-1})+\epsilon_x,
\label{app:eq:cspfi-x1}
\end{equation}
with $h$ and the distribution of $\epsilon_x$ varying across the four scenarios. The original simulation was designed to range from independence to increasingly complex conditional dependence structures, including dependence on many covariates and discontinuous changes in both the conditional mean and variance \citep{molnar2024model}. In the first three scenarios, \(\epsilon_x\) is independent of \(\bm X_{-1}\) and \(\sigma_x(\bm X_{-1})\) in the last scenario denotes the conditional variance, rather than the conditional standard deviation.

\begin{table}[h]
\centering
\caption{Data-generating processes from the cs-PFI simulation and the corresponding analytical \LCVI and CPFI for $X_1$ under squared loss.}
\label{tab:cspfi-dgp-summary}
\renewcommand{\arraystretch}{1.05}
\setlength{\tabcolsep}{4pt}
\begin{tabularx}{\linewidth}{
    @{}
    >{\raggedright\arraybackslash}p{0.19\linewidth}
    >{\raggedright\arraybackslash}X
    >{\raggedright\arraybackslash}p{0.28\linewidth}
    >{\raggedright\arraybackslash}p{0.12\linewidth}
    @{}
}
\toprule
\textbf{Setting} & \textbf{Conditional distribution of $X_1$} & \textbf{Analytical \LCVI} & \textbf{CPFI} \\
\midrule
\textbf{Independent} & $h(\bm X_{-1})=0$, $\epsilon_x\sim\mathcal N(0,1)$ & $(x_2+1)^2(1+x_1^2)$ & $4$ \\
\\\midrule\\
\textbf{Linear} & $h(\bm X_{-1})=X_2$, $\epsilon_x\sim\mathcal N(0,1)$ & $(x_2+1)^2\{1+(x_1-x_2)^2\}$ & $4$ \\
\\\midrule\\
\textbf{Multiple linear dependencies} & $h(\bm X_{-1})=\sum_{j=2}^{10}X_j$, $\epsilon_x\sim\mathcal N(0,5)$ & $(x_2+1)^2\left\{5+\left(x_1-\sum_{j=2}^{10}x_j\right)^2\right\}$ & $20$ \\
\\\midrule\\
\textbf{Non-linear} & $\begin{aligned} h(\bm X_{-1})={}&3\mathbf 1(X_2>0)-3\mathbf 1(X_2\leq0)\mathbf 1(X_3>0),\\ \epsilon_x\mid \bm X_{-1}&\sim\mathcal N(0,\sigma_x(X)),\\ \sigma_x(X)={}&\mathbf 1(X_2>0)\\&+2\mathbf 1(X_2\leq0)\mathbf 1(X_3>0)\\&+5\mathbf 1(X_2\leq0)\mathbf 1(X_3\leq0) \end{aligned}$ & $(x_2+1)^2\left[\sigma_x(x)+\{x_1-h(x_{-1})\}^2\right]$ & $5.011$ \\
\bottomrule
\end{tabularx}
\end{table}

\paragraph{Analytical calculation.} For all four scenarios, the dependence of the prediction function on $X_1$ can be isolated as
\begin{equation}
g(x_1',\bm x_{-1})-g(x_1,\bm x_{-1})=(x_2+1)(x_1'-x_1).
\end{equation}
Let
\[
m(\bm x_{-1})=\mathbb E[X_1\mid \bm X_{-1}=\bm x_{-1}],\qquad v(\bm x_{-1})=\operatorname{Var}(X_1\mid \bm X_{-1}=\bm x_{-1}).
\]
Under squared loss, the observation-level \LCVI for $X_1$ is therefore
\begin{align}
\mathrm{\LCVI}_{i,1}
&=\mathbb E\left[\{g(X_1',\bm x_{i,-1})-g(x_{i,1},\bm x_{i,-1})\}^2\mid \bm x_{i,-1}\right]\nonumber\\
&=(x_{i,2}+1)^2\mathbb E\left[(X_1'-x_{i,1})^2\mid \bm x_{i,-1}\right]\nonumber\\
&=(x_{i,2}+1)^2\left[v(\bm x_{i,-1})+\{x_{i,1}-m(\bm x_{i,-1})\}^2\right].
\label{app:eq:cspfi-lcvi-general}
\end{align}
Averaging over observations gives the corresponding conditional permutation feature importance,
\begin{align}
\mathrm{CPFI}_1
&=\mathbb E[\mathrm{\LCVI}_{1}]\nonumber\\
&=2\,\mathbb E\left[(X_2+1)^2v(\bm X_{-1})\right],
\label{app:eq:cspfi-cpfi-general}
\end{align}
where the second equality follows because, conditional on $\bm X_{-1}$,
\[
\mathbb E\left[\{X_1-m(\bm X_{-1})\}^2\mid \bm X_{-1}\right]=v(\bm X_{-1}).
\]

For the \textit{independent} setting, $m(\bm X_{-1})=0$ and $v(\bm X_{-1})=1$, giving
\[
\mathrm{\LCVI}_{1}=(X_2+1)^2(1+X_1^2),\qquad \mathrm{CPFI}_1=2\mathbb E[(X_2+1)^2]=4.
\]
For the \textit{linear} setting, $m(\bm X_{-1})=X_2$ and $v(\bm X_{-1})=1$, so
\[
\mathrm{\LCVI}_{1}=(X_2+1)^2\{1+(X_1-X_2)^2\},\qquad \mathrm{CPFI}_1=4.
\]
For the \textit{multiple linear dependencies} setting,
\[
m(\bm X_{-1})=\sum_{j=2}^{10}X_j,\qquad v(\bm X_{-1})=5,
\]
and hence
\[
\mathrm{\LCVI}_{1}=(X_2+1)^2\left[5+\left(X_1-\sum_{j=2}^{10}X_j\right)^2\right],\qquad \mathrm{CPFI}_1=20.
\]

Finally, in the \textit{non-linear} setting,
\[
m(\bm X_{-1})=h(\bm X_{-1}),\qquad v(\bm X_{-1})=\sigma_x(X),
\]
which yields
\[
\mathrm{\LCVI}_{1}=(X_2+1)^2\left[\sigma_x(X)+\{X_1-h(\bm X_{-1})\}^2\right].
\]
Since $X_2$ and $X_3$ are independent standard Gaussian variables,
\begin{align}
\mathrm{CPFI}_1
&=2\mathbb E[(X_2+1)^2\sigma_x(X)]\nonumber\\
&=2\left[\mathbb E\{(X_2+1)^2\mathbf 1(X_2>0)\}+\frac{2+5}{2}\mathbb E\{(X_2+1)^2\mathbf 1(X_2\leq0)\}\right]\nonumber\\
&=2\left[1+\sqrt{\frac{2}{\pi}}+\frac{7}{2}\left(1-\sqrt{\frac{2}{\pi}}\right)\right]\nonumber\\
&\approx 5.011.
\end{align}

\paragraph{Results.} Table~\ref{tab:cspfi-highdim-mse} reports the MSE between estimated and analytical \LCVI for $X_1$. We additionally report the mean rank across all combinations of sample size ($n$), dimension ($p$), and dependence structure.
\begin{table}[t]
\centering
\caption{MSE between estimated and analytical \LCVI for $X_1$ in the cs-PFI simulation settings. Lower values indicate more accurate estimation. The lowest MSE in each setting is shown in bold.}
\label{tab:cspfi-highdim-mse}
\setlength{\tabcolsep}{3.5pt}
\renewcommand{\arraystretch}{1.03}
\begin{tabular}{lrrrrrrr}
\toprule
\textbf{Setting} & \textbf{DR-BART} & \textbf{ARF} & \textbf{cs-PFI} & \textbf{CMR} & \textbf{GCMR} & \textbf{CLIQUE} & \textbf{Perm.} \\
\midrule
\multicolumn{8}{l}{\textbf{Independent}} \\
\quad $n=300,\ p=10$ & 3.30 & 13.23 & 1.29 & 21.94 & 15.68 & \textbf{0.13} & 0.95 \\
\quad $n=300,\ p=90$ & 5.03 & 4.04 & 0.31 & 17.31 & 10.32 & 0.26 & \textbf{0.25} \\
\quad $n=3000,\ p=10$ & 1.21 & 15.65 & 0.22 & 10.89 & 16.05 & \textbf{0.014} & 0.21 \\
\quad $n=3000,\ p=90$ & 1.21 & 7.44 & 0.10 & 9.53 & 8.23 & \textbf{0.004} & 0.10 \\[2pt]
\multicolumn{8}{l}{\textbf{Linear}} \\
\quad $n=300,\ p=10$ & \textbf{26.52} & 170.37 & 73.22 & 60.85 & 45.93 & 302.60 & 367.92 \\
\quad $n=300,\ p=90$ & 30.00 & 206.58 & 24.65 & \textbf{21.86} & 30.50 & 201.73 & 197.66 \\
\quad $n=3000,\ p=10$ & \textbf{5.59} & 31.87 & 7.19 & 18.80 & 16.46 & 304.57 & 312.82 \\
\quad $n=3000,\ p=90$ & \textbf{2.51} & 437.41 & 5.66 & 12.80 & 16.21 & 437.32 & 443.73 \\[2pt]
\multicolumn{8}{l}{\textbf{Multiple linear dependencies}} \\
\quad $n=300,\ p=10$ & \textbf{239.59} & 2336.59 & 8575.39 & 1629.81 & 3037.15 & 7386.04 & 8688.08 \\
\quad $n=300,\ p=90$ & \textbf{848.80} & 4147.11 & 2562.16 & 1668.91 & 1712.87 & 3351.69 & 2931.48 \\
\quad $n=3000,\ p=10$ & \textbf{125.38} & 1028.88 & 2742.75 & 639.26 & 741.14 & 6823.13 & 6856.28 \\
\quad $n=3000,\ p=90$ & \textbf{139.18} & 8502.58 & 4116.04 & 1146.00 & 837.31 & 7948.26 & 7991.29 \\[2pt]
\multicolumn{8}{l}{\textbf{Non-linear}} \\
\quad $n=300,\ p=10$ & 41.87 & 1118.96 & \textbf{6.71} & 122.98 & 246.57 & 1832.60 & 1989.33 \\
\quad $n=300,\ p=90$ & 101.58 & 1619.07 & \textbf{6.26} & 66.22 & 71.33 & 1418.45 & 1429.19 \\
\quad $n=3000,\ p=10$ & 3.86 & 196.79 & \textbf{0.74} & 59.03 & 80.88 & 1572.78 & 1588.26 \\
\quad $n=3000,\ p=90$ & 3.97 & 1654.96 & \textbf{0.26} & 50.32 & 51.35 & 1897.84 & 1899.32 \\
\midrule
\textbf{Mean rank} & \textbf{2.31} & 5.38 & 2.81 & 3.56 & 4.06 & 4.56 & 5.31 \\
\bottomrule
\end{tabular}
\end{table}
\footnote{Ranks are computed from unrounded MSE values.}

Across all 16 settings, DR-BART achieves the best overall mean rank of $2.31$, followed by cs-PFI with a mean rank of $2.81$. The relative performance varies substantially with the structure of the conditional distribution. In the \textit{linear} setting, DR-BART is highly competitive and attains the lowest MSE in three of the four configurations. Its advantage is particularly pronounced in the \textit{multiple linear dependencies} setting, where $X_1$ depends jointly on nine covariates and DR-BART attains the lowest MSE in all four configurations. This setting is challenging for subgroup-based permutation methods because a tree must approximate several simultaneous linear dependencies through repeated partitioning; indeed, \citet{molnar2024model} similarly identify this setting as particularly difficult for cs-PFI.

The \textit{independent} setting represents a different edge case. Because $X_1\perp \bm X_{-1}$, the correct conditional distribution is simply the marginal distribution of $X_1$, so no conditional density estimation is required. Methods that estimate a conditional distribution can therefore incur additional finite-sample estimation error without gaining information, whereas marginal or near-marginal perturbation schemes can perform particularly well. This explains why the simplest methods obtain very small errors in this setting and why the independent case should not be interpreted as favoring a more flexible conditional estimator.

In the \textit{non-linear} setting, cs-PFI performs best in all four configurations, while DR-BART ranks second in three of the four configurations. This DGP is particularly favorable to the subgroup construction used by cs-PFI: both the conditional mean and variance of $X_1$ change across only a small number of axis-aligned regions defined by the signs of $X_2$ and $X_3$. A transformation tree can therefore represent the relevant conditional structure using only a few splits. In contrast, DR-BART estimates the full conditional density through a flexible mixture representation, which is substantially more general than is required for this piecewise structure and consequently pays a larger finite-sample estimation cost. Nevertheless, its error decreases sharply as the sample size increases. For example, when $p=10$, the MSE decreases from $41.87$ at $n=300$ to $3.86$ at $n=3000$, while for $p=90$ it decreases from $101.58$ to $3.97$. Thus, although the finite-sample problem is substantially harder for DR-BART in this setting, the improvement with increasing sample size indicates that it is able to recover the underlying conditional structure with sufficient data.

Another notable result is the poor performance of ARF in these higher-dimensional settings, where it is consistently among the worst-performing methods and often deteriorates markedly as the ambient dimension increases. This provides further empirical support for our motivation to estimate the conditional density directly: approaches such as ARF recover the conditional distribution through an estimated joint distribution of all covariates, a task that becomes increasingly challenging as the feature dimension grows \citep{izbicki2016nonparametric}.

Overall, these experiments emphasize that no conditional sampler uniformly dominates across all dependence structures. Subgroup permutation is particularly effective when the conditional distribution admits a simple tree partition, whereas direct conditional density estimation is substantially more competitive when dependence is smooth or distributed across many covariates. Importantly, \LCVI itself is agnostic to this choice of conditional density estimator; these results therefore characterize the statistical trade-offs of the sampling mechanism rather than a restriction of the \LCVI framework.

\newpage

\section{Additional Bike Sharing Results}
\label{app:bike}

\subsection{Complete Individual \LCVI Results}
\label{app:bike_individual}

We provide the complete individual \LCVI results for all five features in Figure~\ref{fig:main_experiments_bike_sharing_temp_hr_workingday_hum_season} as a complement to the main paper's results. For each feature, the first column shows how its mean \LCVI varies over hour of day on working and non-working days, while the remaining columns show the corresponding DR-BART conditional replacement distributions across season, working-day status, and hour periods, together with the marginal permutation distribution for comparison. The results reveal substantial heterogeneity for temperature, humidity, and season: their plausible replacement values shift across different covariate contexts, and their local importance varies accordingly. In contrast, the conditional replacement distributions of hour and working day are comparatively stable across the displayed contexts.

\begin{figure}[t]
    \centering
    \includegraphics[width=\linewidth]{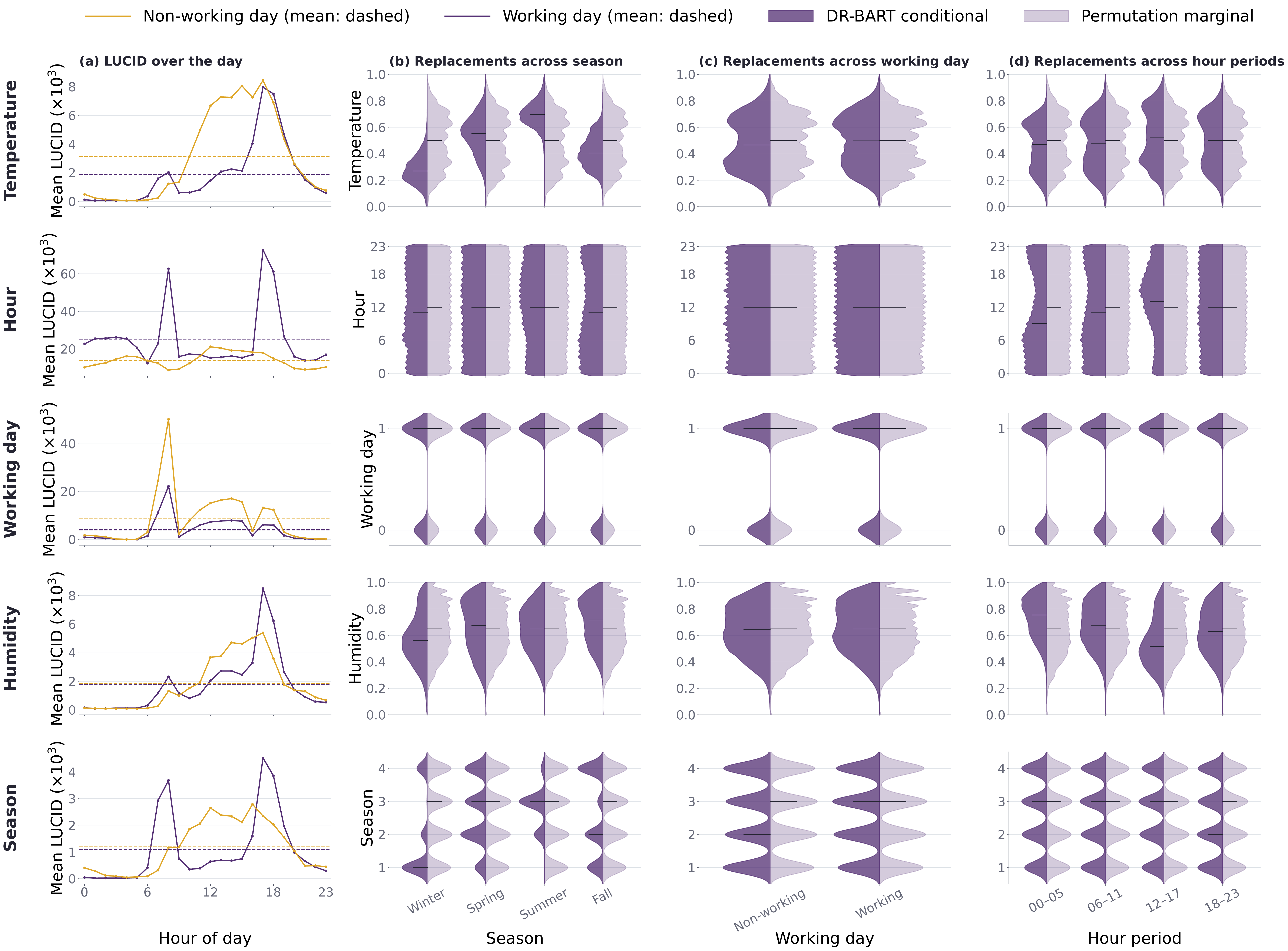}
    \caption{\textbf{Complete individual \LCVI results for the Bike Sharing data.} For temperature, hour, working day, humidity, and season, panel (a) shows mean \LCVI across hour of day separately for working and non-working days. Panels (b)--(d) compare the DR-BART conditional replacement distributions across season, working-day status, and hour periods with the corresponding marginal permutation distributions.}
    \label{fig:main_experiments_bike_sharing_temp_hr_workingday_hum_season}
\end{figure}

\subsection{Feature Dependence and the Interpretation of Individual \LCVI}
\label{app:bike_relationship}

\paragraph{How do feature dependencies inform the interpretation of individual \LCVI?} Figure~\ref{fig:bike_sharing_drbart_relationship_compact} summarizes feature dependencies alongside the estimated individual \LCVI{}s. We report pairwise correlations to describe linear associations and quantify each feature's predictability from the remaining features as
\[
R_j^2
:=1-\frac{\mathbb{E}[\operatorname{Var}(X_j\mid X_{-j})]}{\operatorname{Var}(X_j)}.
\]
High predictability indicates that conditional resampling permits less variation on average relative to marginal resampling because information in this feature is already captured by the remaining features. Consequently, a small individual \LCVI may reflect either limited model sensitivity to the feature or limited variation under conditional replacement. If the remaining features completely determine the feature of interest, then conditional replacement introduces no variation and \LCVI is zero. Conversely, low predictability alone does not imply large \LCVI, because the effect of replacement also depends on the model's response: if the model does not use the feature of interest, its \LCVI remains zero. Moreover, $R_j^2=0$ does not imply independence, since conditional variance or other distributional properties may still depend on the remaining features, as illustrated by the heteroskedastic synthetic example.

Pairwise correlations provide a complementary but weaker diagnostic. By the law of total variance,
\[
R_j^2
=
\frac{\operatorname{Var}\{\mathbb{E}(X_j\mid X_{-j})\}}{\operatorname{Var}(X_j)},
\]
which is the squared correlation ratio of \(X_j\) with respect to \(X_{-j}\). Therefore, for any $k\neq j$,
\[
\operatorname{Corr}(X_j,X_k)^2
\leq
\sup_{\bm a}\operatorname{Corr}(X_j,\bm a^\top\bm X_{-j})^2
\leq
R_j^2.
\]
Thus, pairwise correlation does not capture all forms of predictability, but a strong pairwise correlation is sufficient evidence that the feature of interest may be highly predictable from the remaining covariates and that conditional importance should be interpreted accordingly. In Figure~\ref{fig:bike_sharing_drbart_relationship_compact}, temperature has particularly high predictability, while humidity and season are also meaningfully predictable from the remaining features. In contrast, hour and working day have very low predictability. These patterns help distinguish whether a small or large \LCVI is primarily associated with restrictions in the conditional replacement distribution or with the model's response to the feature.

\textbf{Note:} For binary and categorical features, we replace variance by Gini impurity, $G(p)=1-\sum_c p_c^2$, and compute predictability as $1-\mathbb{E}\{G(p(\cdot\mid X_{-j}))\}/G(p)$. This avoids dependence on arbitrary category codes and agrees with the variance-based measure for binary features. We estimate conditional uncertainty using 200 DR-BART draws per observation, with a finite-sample correction for Gini impurity. DR-BART samples continuous values from its posterior predictive density and discrete values from BART-based class probabilities; ARF samples from its conditional forest distribution; cs-PFI resamples within fitted tree leaves; CMR and GCMR resample residuals in the original and Gaussianized spaces, respectively, without enforcing categorical support. Permutation samples from the empirical marginal, while CLIQUE uses marginal quantiles or category-frequency grids.

\begin{figure}[t]
    \centering
    \includegraphics[width=\linewidth]{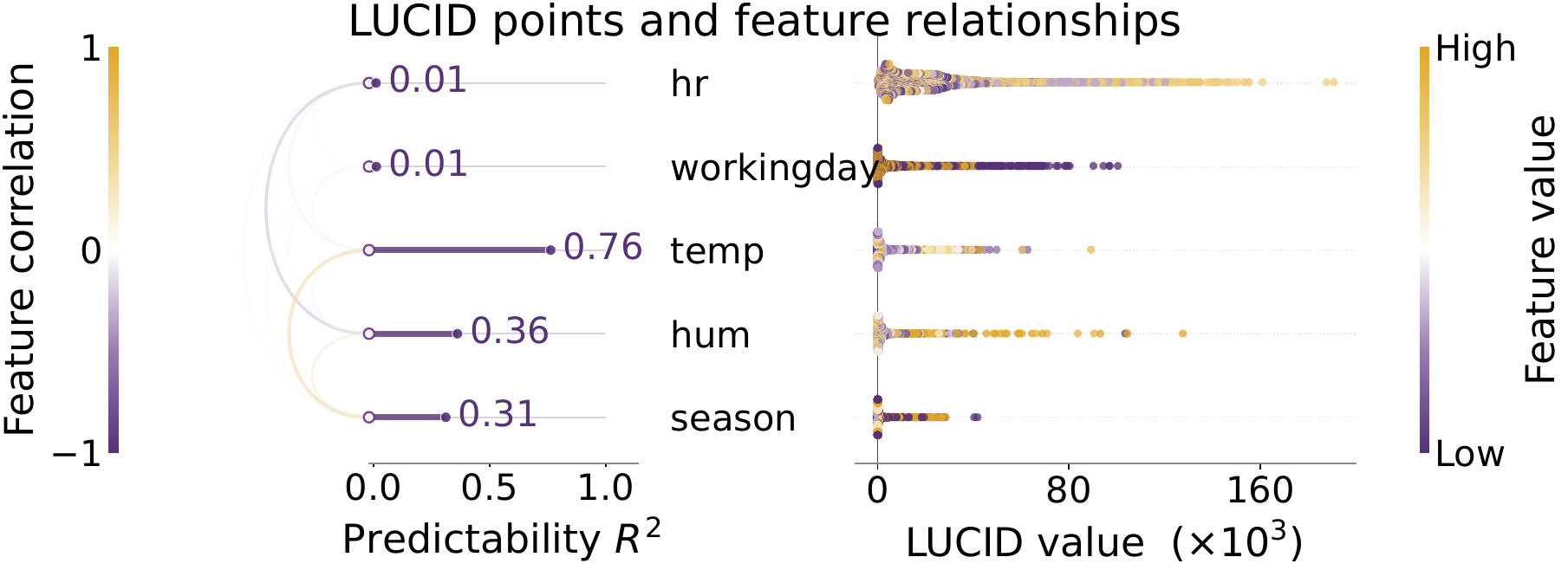}
    \caption{\textbf{Feature relationships and individual \LCVI in the Bike Sharing data.} The left panel summarizes pairwise correlations and each feature's predictability $R_j^2$ from the remaining covariates. The right panel shows observation-level \LCVI values, colored by the observed feature value. Temperature is highly predictable from the remaining features, while hour and working day have very low predictability, providing additional context for interpreting their individual \LCVI values.}
    \label{fig:bike_sharing_drbart_relationship_compact}
\end{figure}

\subsection{Predictability Diagnostics for Feature Groups}
\label{app:bike_group_predictability}

The main paper shows that group \LCVI can complement individual importance when several features contain overlapping or shared information. Here, we provide two additional predictability diagnostics to help interpret a candidate feature group and assess whether its members meaningfully cover one another's information. As defined above, let \(G\) denote a group of \(d=|G|\) features, and let \(X_{-G}\) denote all features outside the group. The first metric is the group outside predictability,
\begin{equation}
R_G^2
=
1-\frac{1}{d}
\sum_{j\in G}
\frac{
\mathbb{E}\!\left[
\operatorname{Var}(X_j\mid X_{-G})
\right]
}{
\operatorname{Var}(X_j)
},
\label{eq:group_outside_predictability}
\end{equation}
which measures, under equal weighting, how well the features outside the group can explain the features within the group.

The second metric, denoted by $R_{G,\mathrm{within}}^2$, measures how much the other features within the group further improve the prediction of each group member relative to using only features outside the group:
\begin{equation}
R_{G,\mathrm{within}}^2
=
1-\frac{1}{d}
\sum_{j\in G}
\frac{
\mathbb{E}\!\left[
\operatorname{Var}(X_j\mid X_{-j})
\right]
}{
\mathbb{E}\!\left[
\operatorname{Var}(X_j\mid X_{-G})
\right]
}.
\label{eq:group_within_predictability}
\end{equation}

At the population level, both metrics lie in \([0,1]\) when their denominators are positive. Finite-sample estimates need not satisfy these bounds exactly. A large $R_G^2$ indicates that the group can already be predicted well from features outside the group, suggesting that the group may contribute limited information beyond the remaining covariates. In contrast, a small $R_G^2$ indicates that the group retains substantial information that outside features cannot easily reconstruct. A large $R_{G,\mathrm{within}}^2$ indicates that on average across group members, the other members substantially reduce residual prediction variance, making joint perturbation particularly informative relative to separate individual perturbations. If $R_{G,\mathrm{within}}^2$ is very small, adding the other group members contributes little additional predictive information, which may suggest reconsidering whether those variables form a substantively meaningful group. 

\begin{table}[t]
\centering
\caption{Group outside predictability $R_G^2$ and within-group predictability $R_{G,\mathrm{within}}^2$ for candidate weather-related feature groups in the Bike Sharing data.}
\label{tab:bike_group_predictability}
\setlength{\tabcolsep}{7pt}
\renewcommand{\arraystretch}{1.05}
\begin{tabular}{lcc}
\toprule
\textbf{Group} & $\boldsymbol{R_G^2}$ & $\boldsymbol{R_{G,\mathrm{within}}^2}$ \\
\midrule
Season + temperature & 0.198 & 0.450 \\
Temperature + humidity & 0.482 & 0.190 \\
Humidity + season & 0.298 & 0.054 \\
Season + humidity + temperature & 0.095 & 0.410 \\
\bottomrule
\end{tabular}
\end{table}

Table~\ref{tab:bike_group_predictability} shows substantial differences among the candidate groups. The three-feature group consisting of season, humidity, and temperature has the lowest outside predictability ($R_G^2=0.095$) while maintaining relatively high within-group predictability ($R_{G,\mathrm{within}}^2=0.410$). This makes it a particularly natural candidate for group \LCVI: the group is difficult to reconstruct from the remaining covariates, yet its members contain substantial overlapping information that can be revealed by joint perturbation. The season--temperature group shows a similar pattern, with low outside predictability ($0.198$) and the highest within-group predictability ($0.450$), also supporting a meaningful grouping interpretation.

In contrast, the temperature--humidity group has substantially higher outside predictability ($R_G^2=0.482$), indicating that much of the group's information can already be recovered from the remaining covariates, while its within-group predictability is more modest ($0.190$). The humidity--season group has moderate outside predictability ($0.298$) but very low within-group predictability ($0.054$), suggesting that these two features provide little additional predictive information about one another once the remaining covariates are known. Thus, among the groups considered here, the three-feature weather group and the season--temperature pair provide the strongest evidence for complementary group analysis, whereas the humidity--season pair is less naturally motivated by shared conditional information.

These metrics are intended as diagnostics rather than a unique grouping rule; substantive knowledge and the scientific question remain important when choosing feature groups. For broader discussions of feature grouping strategies, we refer readers to the literature on grouped variable importance \citep{gregorutti2015grouped,au2022grouped,chamma2024variable}.

Together with Figure \ref{fig:bike_sharing_group_combination_main_2}, these results show that a group's \LCVI is not fully reflected by the sum of its members' individual \LCVI{}s, and that group importance itself varies across the remaining covariate context. Individual and group \LCVI therefore provide complementary views of the learned model's local predictive reliance. As shown throughout this section, the \LCVI formulation and estimator accommodate both perspectives and support richer exploration of the fitted black-box model than marginal feature importance measures alone.

\subsection{Dispersion--Displacement Decomposition}
\label{app:bike_decomposition}

We further examine the squared-loss decomposition of \LCVI into its \textcolor{mypurple}{dispersion} and \textcolor{myyellow}{displacement} components in practice. In Figure~\ref{fig:app_experiments_bike_sharing_decomposition}, $V$ denotes the \textcolor{mypurple}{dispersion} term and $D$ denotes the \textcolor{myyellow}{displacement} term; observations with similar contributions from the two components appear near the balanced region. Averaged over all observations, the decomposition is nearly even, with $V=51.3\%$ and $D=48.7\%$, consistent with Corollary \ref{app:cor:squared-averages}. Importantly, this near equality in the global average masks substantial observation-level heterogeneity in the mechanism driving local importance.

We use the hour feature as an example to show how to interpret the decomposition locally. For hour, which is only weakly predictable from the remaining features, conditional replacement leaves a relatively broad range of plausible hour values. The labeled examples in Figure~\ref{fig:app_experiments_bike_sharing_decomposition} illustrate several qualitatively different regimes. At late-night observations such as example~(1), the observed prediction can be far below the conditional replacement mean, producing a substantial displacement contribution. During the morning and evening commuting peaks on working days, as in example~(2), the original model prediction is substantially higher than the conditional mean of the replacement predictions, so local importance is strongly displacement-driven. At similar hours on non-working days, the observed prediction is often closer to the conditional replacement mean, consistent with the less pronounced commuting pattern on non-working days and the remaining importance is driven more by dispersion across plausible replacements. Around midday, as illustrated by example~(3), displacement is smaller, and variation across replacements contributes a larger share of the total \LCVI. Thus, even for the same feature, a large \LCVI can arise through different local mechanisms.

\begin{figure}[t]
    \centering
    \includegraphics[width=\linewidth]{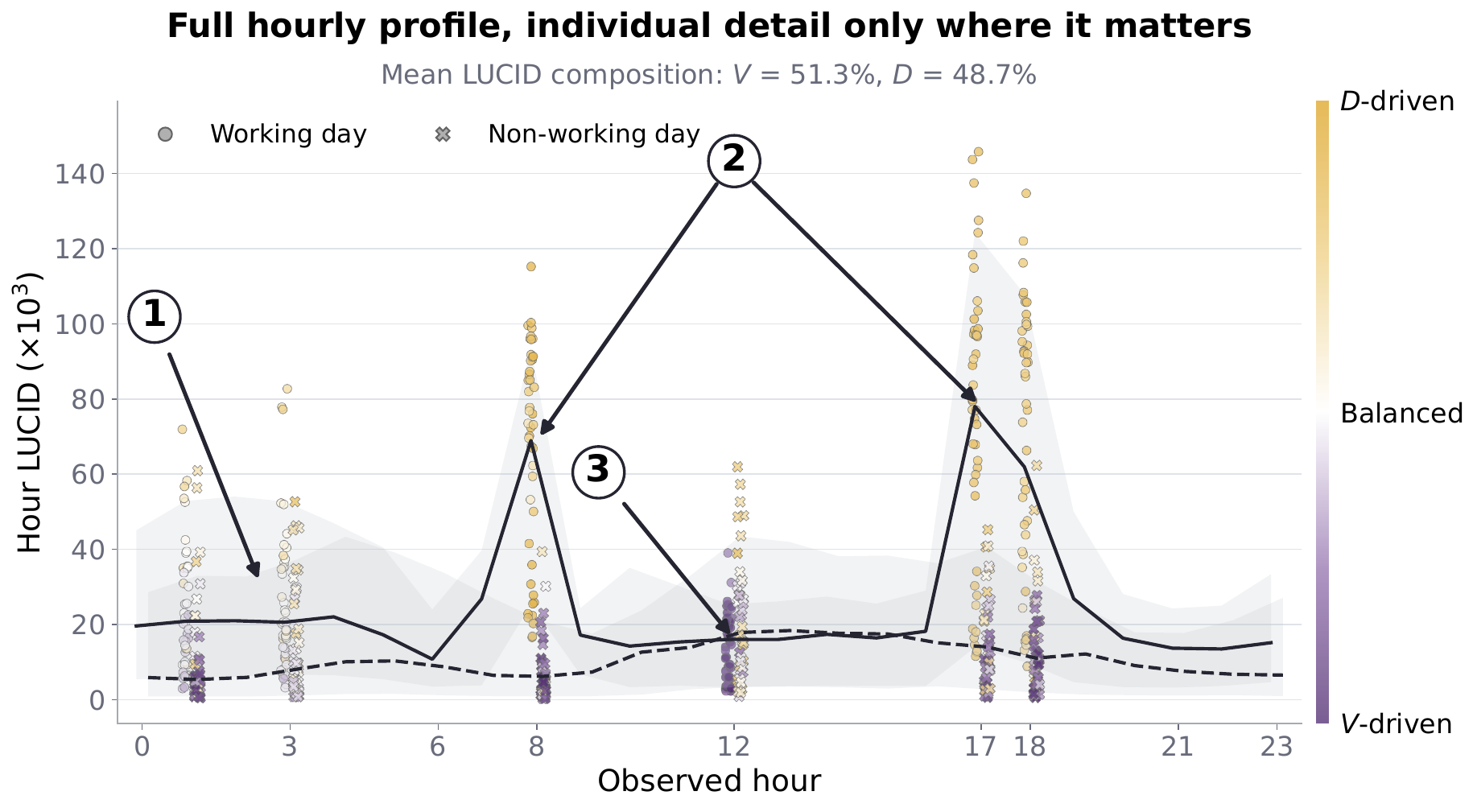}
    \caption{\textbf{Dispersion--displacement decomposition of hour \LCVI.} Observation-level hour \LCVI is decomposed into the \textcolor{mypurple}{dispersion} term $V$ and the \textcolor{myyellow}{displacement} term $D$, the percentages indicate their respective shares of the mean LUCID across observations. Mean LUCID composition: $V=51.3\%$, $D=48.7\%$ of the mean LUCID. Color indicates which component contributes more strongly, while the solid and dashed curves summarize working and non-working days, respectively.}
    \label{fig:app_experiments_bike_sharing_decomposition}
\end{figure}

\subsection{Conditional Replacement Distributions Across Methods}
\label{app:bike_cde_methods}

Finally, Figures~\ref{fig:app_experiments_bike_sharing_cde_hr}--\ref{fig:app_experiments_bike_sharing_cde_season} compare the conditional replacement distributions produced by the different sampling methods across season, working-day status, and hour periods. The goal here is not to identify a ground-truth conditional distribution, which is unavailable in the real data, but to illustrate how the different estimators behave under the same observed covariate contexts. DR-BART, ARF, cs-PFI, and GCMR generally preserve the support and qualitative structure of the observed features, whereas CMR can generate values outside the natural feature range because its residual-based construction is not intrinsically support-preserving; GCMR was designed in part to address this limitation for continuous features. Marginal permutation, by construction, does not adapt its replacement distribution to the conditioning context and therefore misses context-dependent distributional shifts. Differences among the conditional methods are most visible for features such as temperature and humidity, whose replacement distributions vary appreciably across season or hour period.

\textbf{Note on visualization.} For binary and categorical features, the violin shapes use smoothing only for display: the underlying samples from permutation, cs-PFI, DR-BART, and ARF remain on the valid categorical support. Each violin is normalized independently to the same maximum width, so violin widths are comparable only within a distribution and should not be interpreted as categorical probabilities across methods. Methods that generate small off-support continuous values, such as CMR and GCMR, are displayed without the discrete smoothing used for purely categorical samples. As a result, each local mode can appear visually narrower, so do not interpret the apparent width of individual bumps as lower categorical probability mass.

\begin{figure}[t]
    \centering
    \includegraphics[width=\linewidth]{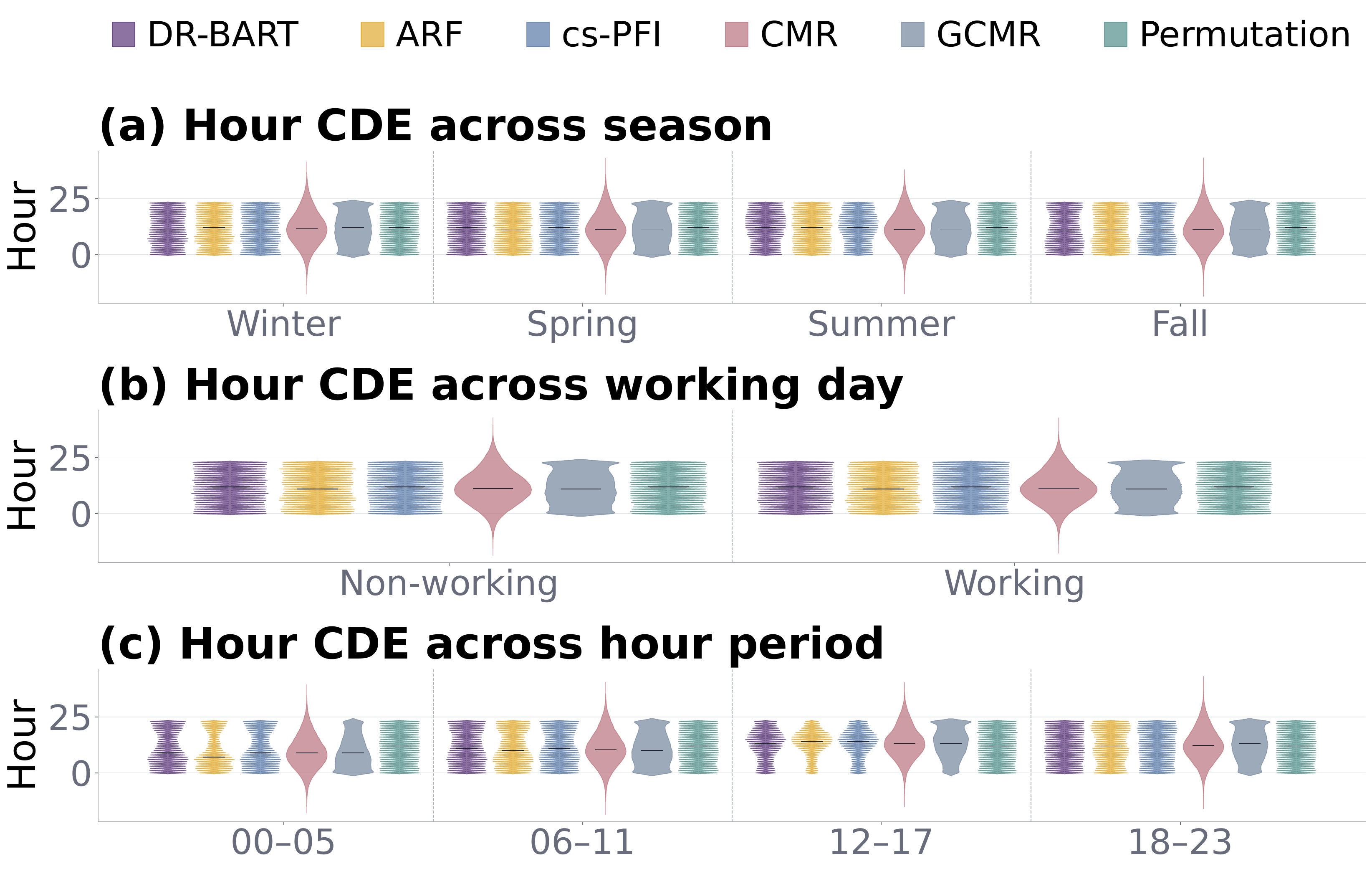}
    \caption{\textbf{Conditional replacement distributions for hour.} Replacement distributions produced by DR-BART, ARF, cs-PFI, CMR, GCMR, and marginal permutation across season, working-day status, and hour periods. Hour is weakly predictable from the remaining features, and its conditional replacement distribution is correspondingly broad and relatively stable across most contexts.}
    \label{fig:app_experiments_bike_sharing_cde_hr}
\end{figure}

\begin{figure}[t]
    \centering
    \includegraphics[width=\linewidth]{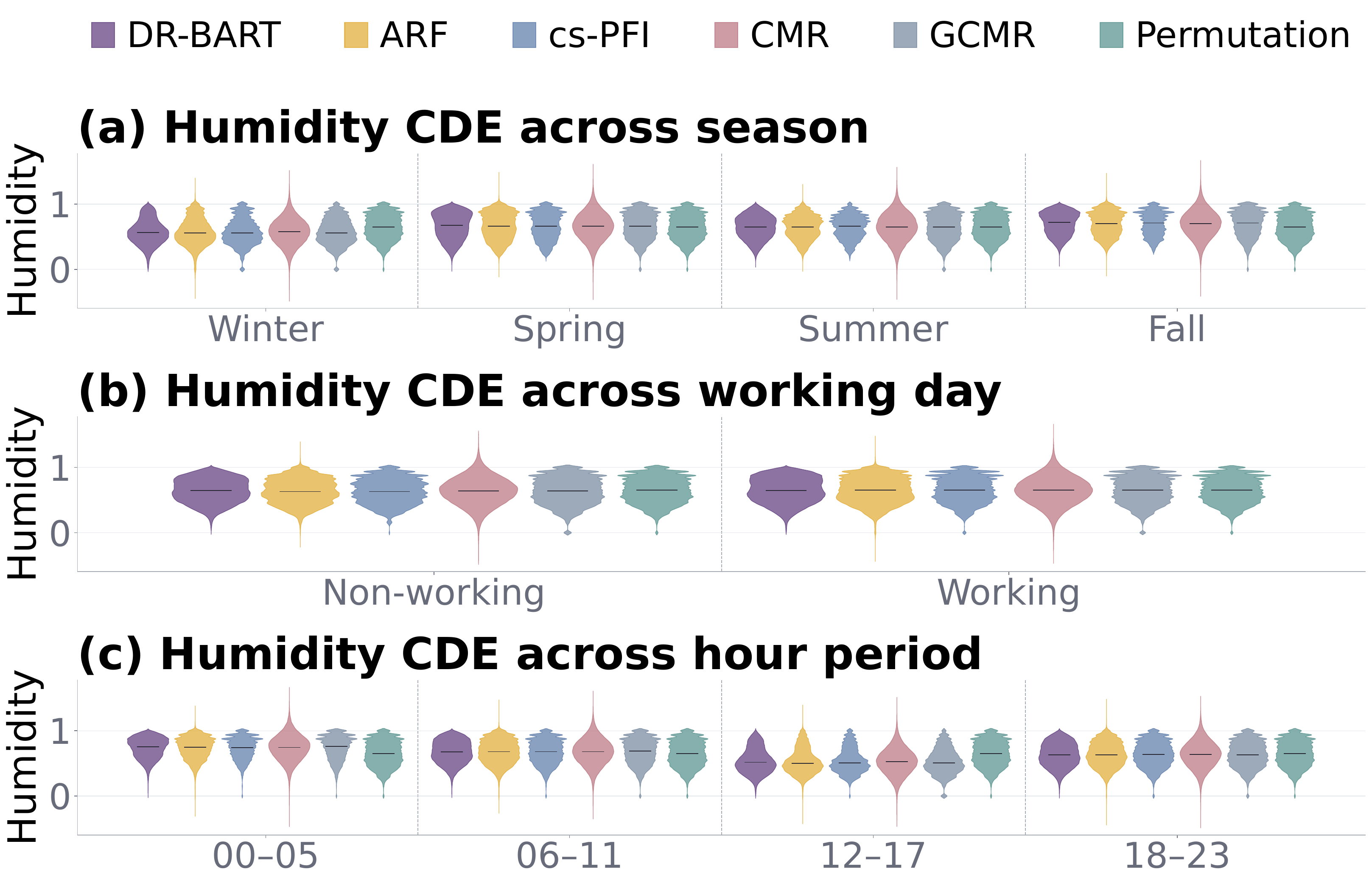}
    \caption{\textbf{Conditional replacement distributions for humidity.} Replacement distributions across season, working-day status, and hour periods. Several conditional methods capture systematic changes in humidity across covariate contexts, particularly across hour periods, whereas marginal permutation uses the same unconditional distribution throughout. CMR also exhibits broader support than the naturally bounded humidity feature.}
    \label{fig:app_experiments_bike_sharing_cde_hum}
\end{figure}

\begin{figure}[t]
    \centering
    \includegraphics[width=\linewidth]{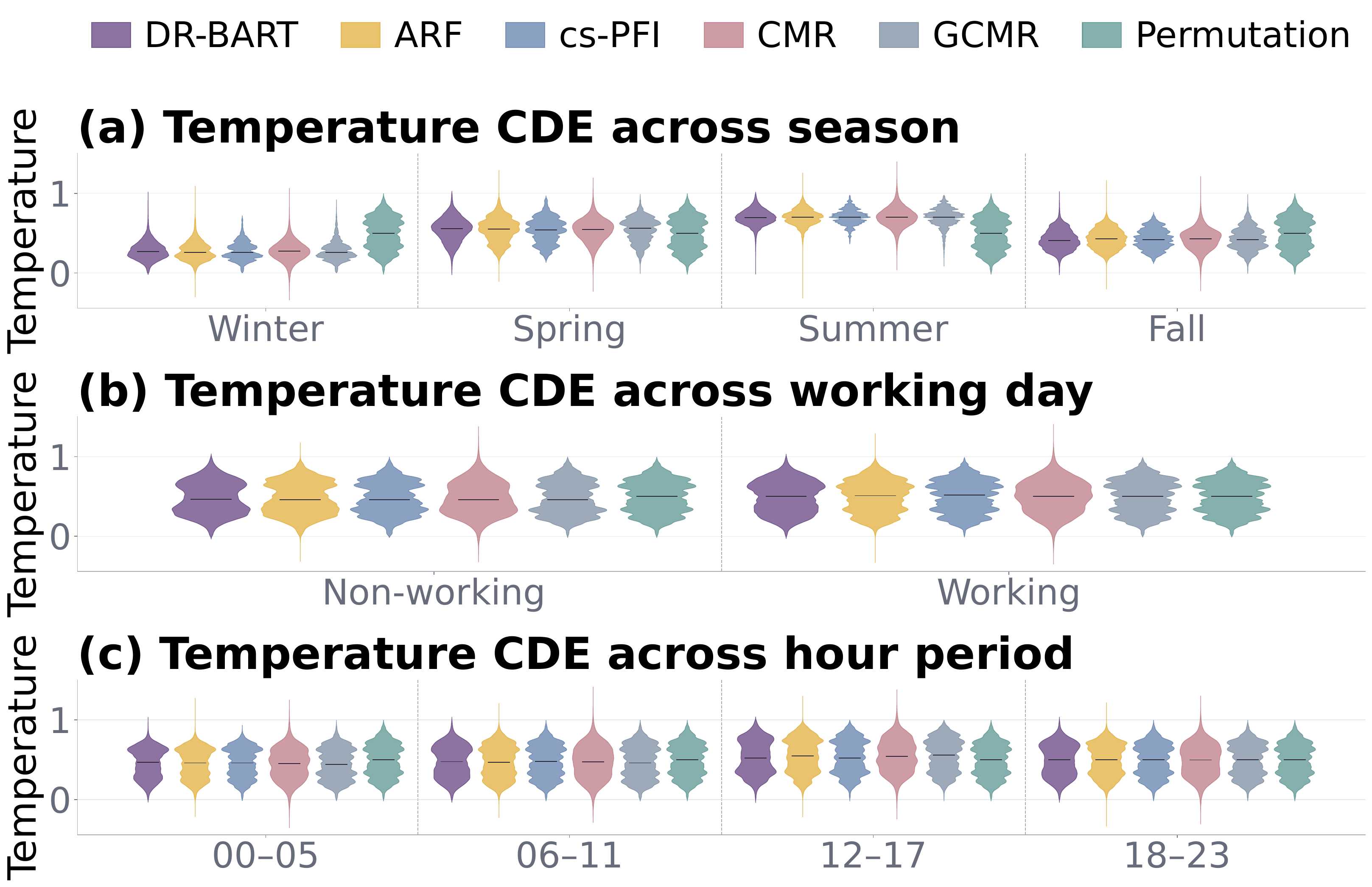}
    \caption{\textbf{Conditional replacement distributions for temperature.} Replacement distributions across season, working-day status, and hour periods. The strongest distributional shifts occur across seasons, where conditional samplers adapt to the changing temperature distribution while marginal permutation remains fixed across contexts.}
    \label{fig:app_experiments_bike_sharing_cde_temp}
\end{figure}

\begin{figure}[t]
    \centering
    \includegraphics[width=\linewidth]{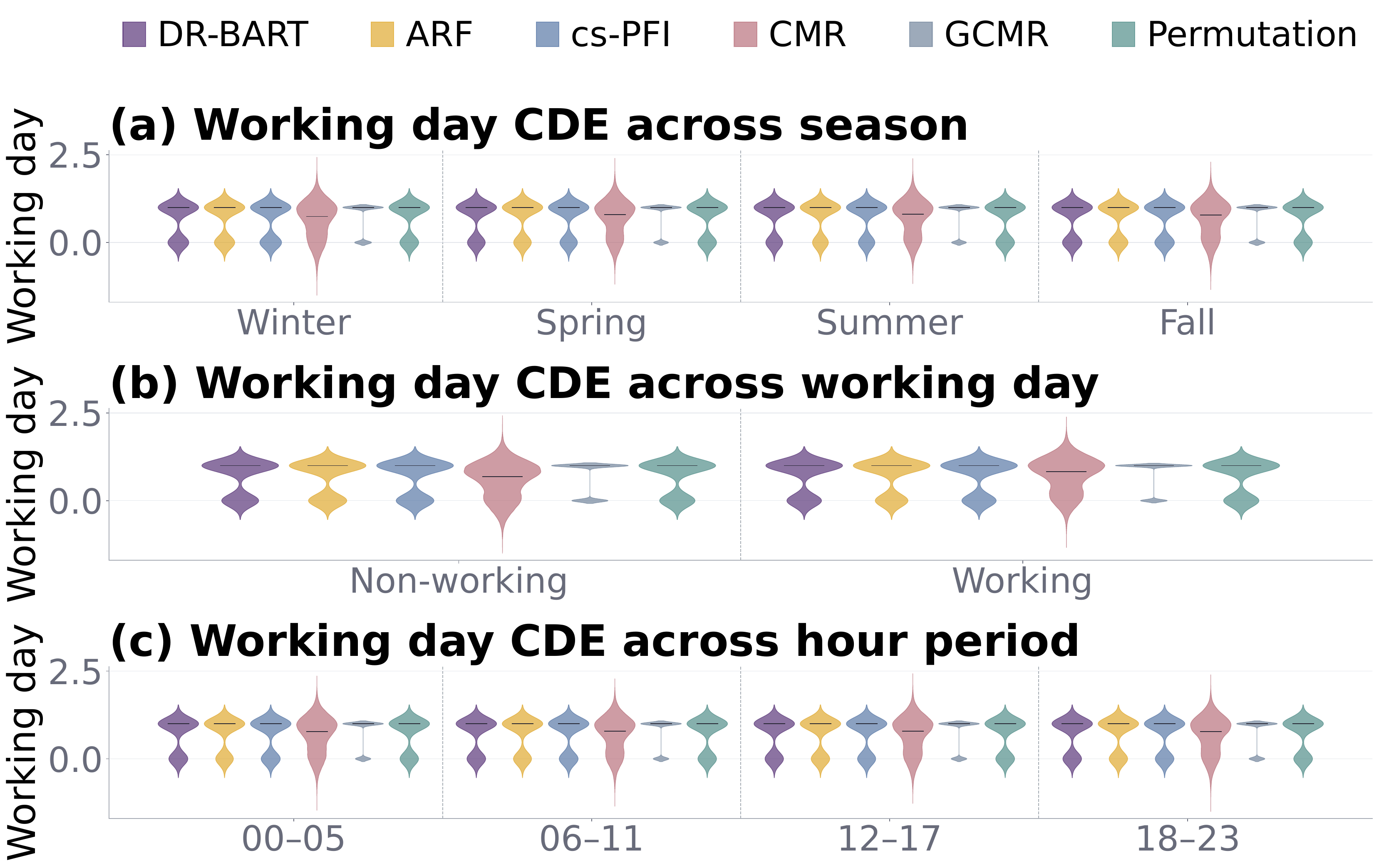}
    \caption{\textbf{Conditional replacement distributions for working day.} Replacement distributions across season, working-day status, and hour periods. Most conditional methods preserve the binary structure of the feature, while CMR can place mass outside the natural $\{0,1\}$ support. The conditional distribution varies relatively little across the displayed contexts.}
    \label{fig:app_experiments_bike_sharing_cde_workingday}
\end{figure}

\begin{figure}[t]
    \centering
    \includegraphics[width=\linewidth]{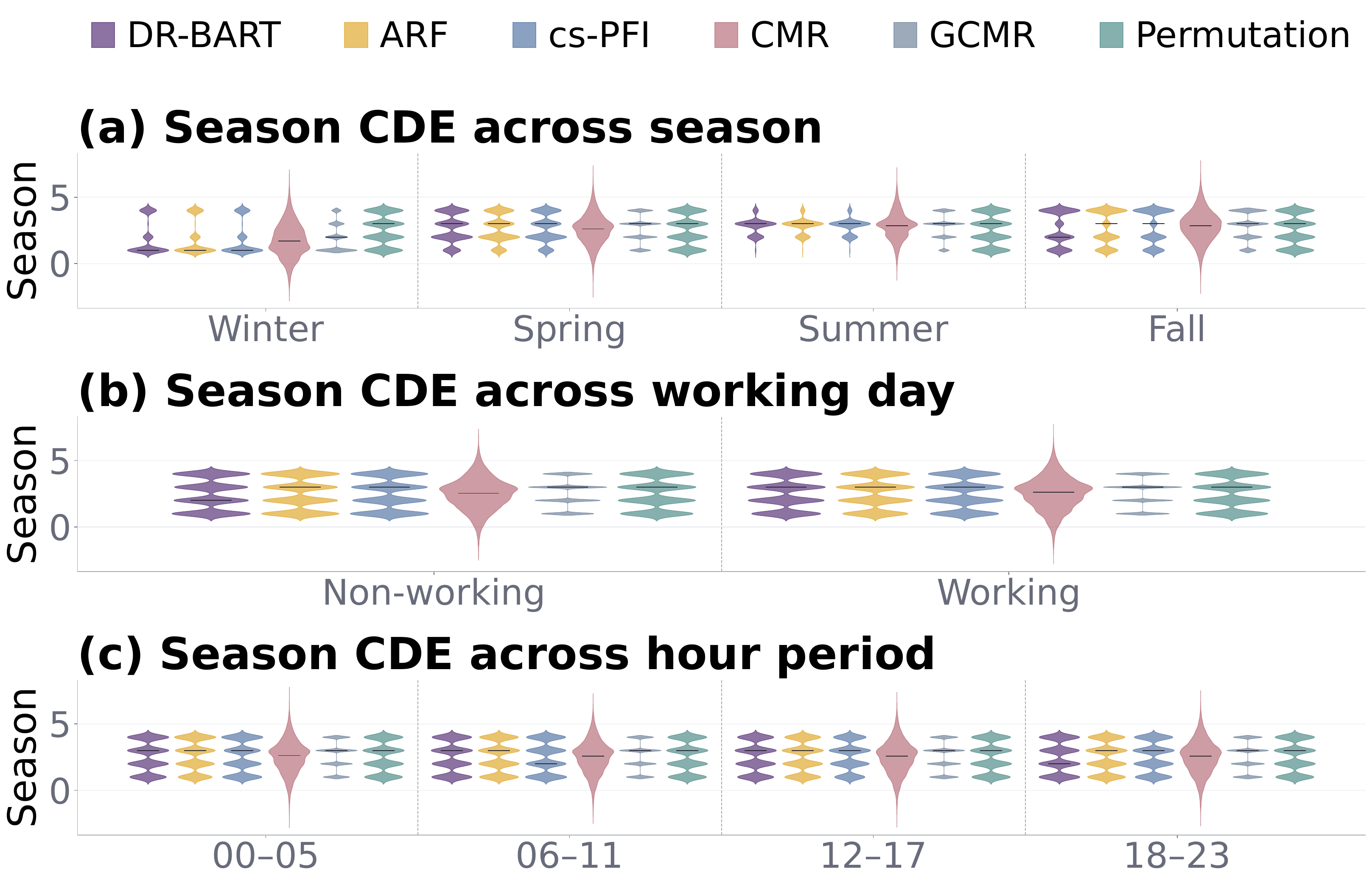}
    \caption{\textbf{Conditional replacement distributions for season.} Replacement distributions across observed season, working-day status, and hour periods. Conditional methods preserve differing degrees of context dependence in the categorical replacement distribution, whereas marginal permutation does not condition on the observed context. CMR again produces values outside the natural categorical support.}
    \label{fig:app_experiments_bike_sharing_cde_season}
\end{figure}
\clearpage
\subsection{Runtime and Memory Scaling on OpenXAI Datasets}
\label{app:openxai_scaling}

To further assess the computational feasibility of our DR-BART-based \LCVI estimator beyond the five-feature Bike Sharing example, we benchmark runtime and peak memory usage on four tabular datasets from the OpenXAI \citep{NEURIPS2022_65398a0e} benchmark. These datasets provide complementary stress tests in both sample size and feature dimension: German contains the largest number of features with $p=60$, while GMSC contains more than $10^5$ observations. Table~\ref{tab:openxai_dataset_summary} summarizes the datasets used in this experiment.

\begin{table}[ht]
\centering
\caption{OpenXAI benchmark datasets used in the scalability experiments.}
\label{tab:openxai_dataset_summary}
\setlength{\tabcolsep}{8pt}
\renewcommand{\arraystretch}{1.05}
\begin{tabular}{lrr}
\toprule
\textbf{Dataset} & \textbf{Samples ($n$)} & \textbf{Features ($p$)} \\
\midrule
Heart  & 3,658   & 15 \\
German & 1,000   & 60 \\
HELOC  & 9,871   & 23 \\
GMSC   & 102,209 & 10 \\
\bottomrule
\end{tabular}
\end{table}

Because \LCVI computation is performed separately for each feature, we report runtime in two complementary ways. First, Table~\ref{tab:openxai_runtime_per_feature} reports the average runtime of one complete feature-specific task, obtained within each fold by dividing the total runtime across all features by the number of features. This normalizes for different feature dimensions and facilitates comparison across datasets. Second, Table~\ref{tab:openxai_runtime_serial} reports the total runtime per fold when all feature-specific jobs are executed serially. Each task includes any method-specific fitting and the computation of that method's importance scores. For DR-BART, this includes conditional-sampler fitting and LUCID computation. All reported values are the mean and standard deviation across five folds. Feature-specific jobs are independent and can be executed concurrently, so the serial totals represent the computational cost in the absence of parallel resources, while actual wall-clock time can be substantially lower when multiple workers are available.

\begin{table}[ht]
\centering
\caption{Average runtime per feature-specific task in \textcolor{red}{seconds} (mean $\pm$ standard deviation across five folds). Each task includes conditional-sampler fitting and \LCVI computation.}
\label{tab:openxai_runtime_per_feature}
\setlength{\tabcolsep}{4pt}
\renewcommand{\arraystretch}{1.03}
\begin{tabular}{lrrrr}
\toprule
\textbf{Method}
& \textbf{Heart}
& \textbf{German}
& \textbf{HELOC}
& \textbf{GMSC} \\
\midrule
Permutation & $8.81\pm1.16$ & $6.47\pm0.64$ & $15.67\pm2.21$ & $87.18\pm21.91$ \\
CMR         & $5.89\pm0.03$ & $5.35\pm0.28$ & $14.99\pm3.37$ & $92.92\pm23.30$ \\
GCMR        & $7.37\pm1.85$ & $6.84\pm0.14$ & $16.33\pm2.77$ & $106.18\pm24.79$ \\
local LOCO        & $5.41\pm0.88$ & $4.97\pm0.15$ & $5.60\pm1.29$ & $4.70\pm0.18$ \\
CLIQUE      & $7.65\pm1.68$ & $5.50\pm0.34$ & $14.61\pm2.87$ & $117.08\pm3.76$ \\
ARF         & $43.84\pm5.88$ & $30.10\pm3.69$ & $287.20\pm25.11$ & $2451.36\pm816.10$ \\
cs-PFI      & $30.61\pm5.61$ & $9.70\pm0.12$ & $224.23\pm10.35$ & $10036.02\pm530.45$ \\
DR-BART     & $184.94\pm21.31$ & $36.96\pm0.29$ & $838.26\pm64.25$ & $9015.14\pm1419.16$ \\
\bottomrule
\end{tabular}
\end{table}

\begin{table}[ht]
\centering
\caption{Serial-equivalent runtime per fold in \textcolor{red}{minutes} (mean $\pm$ standard deviation across five folds), obtained by summing all feature-specific jobs. This represents the runtime required when feature computations are not parallelized.}
\label{tab:openxai_runtime_serial}
\setlength{\tabcolsep}{4pt}
\renewcommand{\arraystretch}{1.03}
\begin{tabular}{lrrrr}
\toprule
\textbf{Method}
& \textbf{Heart}
& \textbf{German}
& \textbf{HELOC}
& \textbf{GMSC} \\
\midrule
Permutation & $2.20\pm0.29$ & $6.47\pm0.64$ & $6.01\pm0.85$ & $14.53\pm3.65$ \\
CMR         & $1.47\pm0.01$ & $5.35\pm0.28$ & $5.75\pm1.29$ & $15.49\pm3.88$ \\
GCMR        & $1.84\pm0.46$ & $6.84\pm0.14$ & $6.26\pm1.06$ & $17.70\pm4.13$ \\
local LOCO        & $1.35\pm0.22$ & $4.97\pm0.15$ & $2.15\pm0.50$ & $0.78\pm0.03$ \\
CLIQUE      & $1.91\pm0.42$ & $5.50\pm0.34$ & $5.60\pm1.10$ & $19.51\pm0.63$ \\
ARF         & $10.96\pm1.47$ & $30.10\pm3.69$ & $110.09\pm9.63$ & $408.56\pm136.02$ \\
cs-PFI      & $7.65\pm1.40$ & $9.70\pm0.12$ & $85.95\pm3.97$ & $1672.67\pm88.41$ \\
DR-BART     & $46.23\pm5.33$ & $36.96\pm0.29$ & $321.33\pm24.63$ & $1502.52\pm236.53$ \\
\bottomrule
\end{tabular}
\end{table}

Table~\ref{tab:openxai_memory} reports peak memory usage. For each fold, we take the largest peak memory among all feature-specific jobs, measured at $0.1$-second intervals. Unlike runtime, peak memory is already a per-job quantity and therefore is not normalized by the number of features.

\begin{table}[ht]
\centering
\caption{Peak memory usage in GiB (mean $\pm$ standard deviation across five folds). For each fold, the maximum peak memory among feature-specific jobs is reported.}
\label{tab:openxai_memory}
\setlength{\tabcolsep}{4pt}
\renewcommand{\arraystretch}{1.03}
\begin{tabular}{lrrrr}
\toprule
\textbf{Method}
& \textbf{Heart}
& \textbf{German}
& \textbf{HELOC}
& \textbf{GMSC} \\
\midrule
Permutation & $0.338\pm0.001$ & $0.329\pm0.000$ & $0.352\pm0.000$ & $0.365\pm0.004$ \\
CMR         & $0.335\pm0.000$ & $0.327\pm0.002$ & $0.354\pm0.003$ & $0.371\pm0.002$ \\
GCMR        & $0.339\pm0.003$ & $0.332\pm0.001$ & $0.358\pm0.001$ & $0.375\pm0.005$ \\
local LOCO        & $0.320\pm0.003$ & $0.320\pm0.002$ & $0.325\pm0.003$ & $0.341\pm0.000$ \\
CLIQUE      & $0.336\pm0.003$ & $0.325\pm0.002$ & $0.351\pm0.004$ & $0.367\pm0.001$ \\
ARF         & $0.996\pm0.006$ & $2.072\pm0.030$ & $1.705\pm0.004$ & $0.985\pm0.010$ \\
cs-PFI      & $0.758\pm0.021$ & $0.562\pm0.003$ & $2.226\pm0.094$ & $40.548\pm0.620$ \\
DR-BART     & $1.245\pm0.001$ & $1.353\pm0.001$ & $0.808\pm0.013$ & $41.854\pm0.031$ \\
\bottomrule
\end{tabular}
\end{table}

As expected, DR-BART is substantially more computationally demanding than simpler replacement schemes because it fits a flexible conditional density model separately for each feature. Nevertheless, the complete \LCVI pipeline remains computationally feasible across all four benchmark datasets. In particular, we successfully run DR-BART on GMSC with $102{,}209$ observations. The average feature-specific task requires approximately $9{,}015$ seconds and the largest feature-specific job uses approximately $41.9$ GiB of memory. If all ten feature-specific jobs are executed serially, one fold requires approximately $1{,}503$ minutes, or about $25$ hours of aggregate computation.

\begin{remark}
Local explanations are often required at prediction time. While training DR-BART itself requires expensive Markov Chain Monte Carlo (MCMC) iterations, the DR-BART model is fixed at prediction time and runtime for explanations depends only on time to sample posterior draws, counterfactuals for each posterior, and fitting/comparing the fixed predictive models. In contrast, other local explanation methods like LIME can require refitting new surrogate models for each new test observation, increasing explanation-inference time.
\end{remark}

\end{document}